\documentclass{article}
\usepackage[utf8]{inputenc}
\usepackage{bbm}
\usepackage{xcolor}
\usepackage{comment}
\usepackage{algorithm}
\usepackage{amsmath}
\usepackage{amsthm}
\usepackage{amssymb}
\usepackage{algpseudocode}
\usepackage{hyperref}
\usepackage{xcolor}
\usepackage{graphicx}
\usepackage{subcaption}
\usepackage[font=scriptsize]{caption}
\usepackage{color}
\hypersetup{
    colorlinks,
    linkcolor={blue!50!black},
    citecolor={red!50!black},
    urlcolor={green!50!black}
}

\algnewcommand\algorithmicforeach{\textbf{for each}}
\algdef{S}[FOR]{ForEach}[1]{\algorithmicforeach\ #1\ \algorithmicdo}

\newcommand{\ldot}[2]{\left\langle #1, #2\right\rangle}
\newcommand{\bX}{\mathbf{X}}

\newcommand{\calX}{\mathcal{X}}

\newcommand{\calD}{\mathcal{D}}
\newcommand{\calB}{\mathcal{B}}
\newcommand{\calA}{\mathcal{A}}

\newcommand{\calZ}{\mathcal{Z}}
\newcommand{\calL}{\mathcal{L}}

\newcommand{\calS}{\mathcal{S}}
\newcommand{\calF}{\mathcal{F}}
\newcommand{\cF}{\mathcal{F}}
\newcommand{\calG}{\mathbb{G}}
\newcommand{\calP}{\mathbb{P}}
\newcommand{\htheta}{\hat{\theta}}

\newcommand{\normal}{\mathsf{N}}
\DeclareMathOperator{\vol}{vol}
\DeclareMathOperator{\supp}{supp}
\DeclareMathOperator{\Var}{Var}

\newcommand{\E}{\mathbf{E}}
\newcommand{\reals}{\mathbb{R}}

\newtheorem{definition}{Definition}
\newtheorem{assumption}{Assumption}
\newtheorem{corollary}{Corollary}
\newtheorem{theorem}{Theorem}
\newtheorem{proposition}{Proposition}
\newtheorem{example}{Example}
\newtheorem{lemma}{Lemma}
\newtheorem{remark}{Remark}

\newcommand{\bb}[1]{\left[#1\right]}
\newcommand{\bp}[1]{\left(#1\right)}
\newcommand{\bc}[1]{\left\{#1\right\}}

\usepackage{enumitem}
\setlist[enumerate]{itemsep=-0.5mm}

\usepackage{natbib}

\usepackage[preprint]{neurips_2019}

\usepackage[utf8]{inputenc} 
\usepackage[T1]{fontenc}    
\usepackage{hyperref}       
\usepackage{url}            
\usepackage{booktabs}       
\usepackage{amsfonts}       
\usepackage{nicefrac}       
\usepackage{microtype}      

\title{Non-Parametric Inference Adaptive to Intrinsic Dimension}

\author{%
  Khashayar Khosravi
  \\
  Stanford University \\
  \texttt{khosravi@stanford.edu}
  \And
  Greg Lewis \\
  Microsoft Research \\
  \texttt{glewis@microsoft.com}
  \And
  Vasilis Syrgkanis \\
  Microsoft Research \\
  \texttt{vasy@microsoft.com}
}

\begin{document}

\maketitle

\begin{abstract}
We consider non-parametric estimation and inference of conditional moment models in high dimensions. We show that even when the dimension $D$ of the conditioning variable is larger than the sample size $n$, estimation and inference is feasible as long as the distribution of the conditioning variable has small intrinsic dimension $d$, as measured by locally low doubling measures. Our estimation is based on a sub-sampled ensemble of the $k$-nearest neighbors ($k$-NN) $Z$-estimator. We show that if the intrinsic dimension of the covariate distribution is equal to $d$, then the finite sample estimation error of our estimator is of order $n^{-1/(d+2)}$ and our estimate is $n^{1/(d+2)}$-asymptotically normal, irrespective of $D$. The sub-sampling size required for achieving these results depends on the unknown intrinsic dimension $d$. We propose an adaptive data-driven approach for choosing this parameter and prove that it achieves the desired rates. We discuss extensions and applications to heterogeneous treatment effect estimation.
\end{abstract}

\section{Introduction}

Many non-parametric estimation problems in econometrics and causal inference can be formulated as finding a parameter vector $\theta(x)\in \reals^p$ that is a solution to a set of conditional moment equations:
\begin{equation}
\textstyle{\E[\psi(Z;\theta(x)) | X=x] = 0\,,}
\end{equation}
when given $n$ i.i.d. samples $(Z_1, \ldots, Z_n)$ from the distribution of $Z$, where $\psi: \calZ \times \reals^p \rightarrow \reals^p$ is a known vector valued moment function, $\calZ$ is an arbitrary data space, $X \in \calX \subset \reals^D$ is the feature vector that is included $Z$. Examples include non-parametric regression\footnote{$Z=(X,Y)$, where $Y\in \reals^p$ is the dependent variable, and $\psi(Z;\theta(x))=Y-\theta(x)$.}, quantile regression\footnote{$Z=(X,Y)$ and $\psi(Z;\theta(x))=1\{Y\leq \theta(x)\} - \alpha$, for some $\alpha\in [0,1]$ that denotes the target quantile.}, heterogeneous treatment effect estimation\footnote{$Z=(X, T, Y)$, where $T\in\reals^p$ is a vector of treatments, and $\psi(Z;\theta(x))= (Y - \ldot{\theta(x)}{T})\,T$.}, instrumental variable regression\footnote{$Z=(X,T,W,Y)$, where $T\in \reals$ is a treatment, $W\in \reals$ an instrument and $\psi(Z;\theta(x))= (Y - \theta(x)\,T)\,W$.}, local maximum likelihood estimation\footnote{Where the distribution of $Z$ admits a known density $f(z;\theta(x))$ and $\psi(Z;\theta(x)) = \nabla_{\theta} \log(f(Z; \theta(x))$.} and estimation of structural econometric models (see e.g., \cite{Reiss2007} and examples in \cite{Chernozhukov2016locally,Chernozhukov2018plugin}). The study of such conditional moment restriction problems has a long history in econometrics (see e.g., \cite{Newey1993,Ai2003,Chen2009,Chernozhukov2015}). However, the majority of the literature assumes that the conditioning variable $X$ is low dimensional, i.e. $D$ is a constant as the sample size $n$ grows (see e.g., \cite{athey2019generalized}). High dimensional variants have primarily been analyzed under parametric assumptions on $\theta(x)$, such as sparse linear forms (see e.g., \cite{chernozhukov2018double}).
There are some papers that address the fully non-parametric setup (see e.g., \cite{lafferty2008,dasgupta2008random,kpotufe2011k,Biau2012,scornet2015}) but those are focused on the estimation problem, and do not address inference (i.e., constructing asymptotically valid confidence intervals).


The goal of this work is to address estimation and inference in conditional moment models with a high-dimensional conditioning variable. As is obvious without any further structural assumptions on the problem, the exponential in dimension rates of approximately $n^{1/D}$ (see e.g.,  \cite{Stone1982}) cannot be avoided. Thereby estimation is in-feasible even if $D$ grows very slowly with $n$. Our work, follows a long line of work in machine learning \citep{dasgupta2008random,kpotufe2011k,kpotufe2013adaptivity}, which is founded on the observation that in many practical applications, even though the variable $X$ is high-dimensional (e.g. an image), one typically expects that the coordinates of $X$ are highly correlated. The latter intuition is formally captured by assuming that the distribution of $X$ has a small doubling measure around the target point $x$. 

We refer to the latter notion of dimension, as the intrinsic dimension of the problem. Such a notion has been studied in the statistical machine learning literature, so as to establish fast estimation rates in high-dimensional kernel regression settings \citep{dasgupta2008random,kpotufe2011k,kpotufe2013adaptivity,Xue2018,Chen2018,Kim2018,jiang2017rates}. However, these works solely address the problem of estimation and do not characterize the asymptotic distribution of the estimates, so as to enable inference, hypothesis testing and confidence interval construction. Moreover, they only address the regression setting and not the general conditional moment problem and consequently do not extend to quantile regression, instrumental variable regression or treatment effect estimation. 

From the econometrics side, pioneering works of \cite{wager2017estimation,athey2019generalized} address estimation and inference of conditional moment models with all the aforementioned desiderata that are required for the application of such methodologies to social sciences, albeit in the low dimensional regime. In particular, \cite{wager2017estimation} consider regression and heterogeneous treatment effect estimation with a scalar $\theta(x)$ and prove $n^{1/D}$-asymptotic normality of a sub-sampled random forest based estimator and \cite{athey2019generalized} extend it to the general conditional moment settings.

These results have been extended and improved in multiple directions, such as improved estimation rates through local linear smoothing \cite{friedberg2018local}, robustness to nuisance parameter estimation error \cite{oprescu2018orthogonal} and improved bias analysis via sub-sampled nearest neighbor estimation \cite{fan2018dnn}. However, they all require low dimensional setting and the rate provided by the theoretical analysis is roughly $n^{-1/D}$, i.e. to get a confidence interval of length $\epsilon$ or an estimation error of $\epsilon$, one would need to collect $O(\epsilon^{-D})$ samples which is prohibitive in most target applications of machine learning based econometrics. 

Hence, there is a strong need to provide theoretical results that justify the success of machine learning estimators for doing inference, via their adaptivity to some low dimensional hidden structure in the data. \emph{Our work makes a first step in this direction and provides estimation and asymptotic normality results for the general conditional moment problem, where the rate of estimation and the asymptotic variance depend only on the intrinsic dimension, independent of the explicit dimension of the conditioning variable.}

Our analysis proceeds in four parts. First, we extend the results by \cite{wager2017estimation,athey2019generalized} on the asymptotic normality of sub-sampled kernel estimators to the high-dimensional, low intrinsic dimension regime and to vector valued parameters $\theta(x)$. Concretely, when given a sample $S=(Z_1, \ldots, Z_n)$, our estimator is based on the approach proposed in \cite{athey2019generalized} of solving a locally weighted empirical version of the conditional moment restriction
\begin{equation}
\hat{\theta}(x) \text{ solves}: \sum_{i=1}^n K(x, X_i, S)\, \psi(Z_i; \theta) = 0\,,
\end{equation}
where $K(x, X_i, S)$ captures proximity of $X_i$ to the target point $x$. The approach dates back to early work in statistics on local maximum likelihood estimation \citep{fan1998local,newey1994kernel,stone1977consistent,tibshirani1987local}. As in \cite{athey2019generalized}, we consider weights $K(x, X_i, S)$ that take the form of an average over $B$ base weights: $K(x, X_i, S) = \frac{1}{B}\sum_{b=1}^B K(x, X_i, S_b)\,1\{i\in S_b\}$, where each $K(x, X_i, S_b)$ is calculated based on a randomly drawn sub-sample $S_b$ of size $s<n$ from the original sample. We will typically refer to the function $K$ as the \emph{kernel}. In \cite{wager2017estimation,athey2019generalized} $K(x, X_i, S_b)$ is calculated by building a tree on the sub-sample, while in \cite{fan2018dnn} it is calculated based on the $1$-NN rule on the sub-sample.

Our main results are general estimation rate and asymptotic normality theorems for the estimator $\hat{\theta}(x)$ (see Theorems~\ref{thm:mse_rate} and \ref{thm:normality}), which
are stated in terms of two high-level assumptions, specifically an upper bound $\epsilon(s)$ on the rate at which the kernel ``shrinks'' and a lower bound $\eta(s)$ on the ``incrementality'' of the kernel.
Notably, the explicit dimension of the conditioning variable $D$ does not enter the theorem, so it suffices in what follows to show that $\epsilon(s)$ and $\eta(s)$ depend only on $d$ rather than $D$.

The shrinkage rate $\epsilon(s)$ is defined as the $\ell_2$-distance between the target point $x$ and the furthest point on which the kernel places positive weight $X_i$, when trained on a data set of $s$ samples, i.e.,
\begin{equation}
\epsilon(s) = \E\bb{\sup \{\|X_i-x\|_2: i\in S_b, K(x, X_i, S_b)>0, |S_b|=s\}}\,.
\end{equation}
The shrinkage rate of the kernel controls the bias of the estimate (small $\epsilon(s)$ implies low bias).
The sub-sampling size $s$ is a lever to trade off bias and variance; larger $s$ achieves smaller bias, since $\epsilon(s)$ is smaller, but increases the variance, since for any fixed $x$ the weights $K(x, X_i, S_b)$ will tend to concentrate on the same data points, rather than averaging over observations.  
Both estimation and asymptotic normality results require the bias to be controlled through the shrinkage rate.  

Incrementality of a kernel describes how much information is revealed about the weight of a sample $i$ solely by knowledge of $X_i$, and is captured by the second moment of the conditional expected weight
\begin{equation} 
\label{eqn:incrementality}
\eta(s)=\E\bb{\E\bb{K(x, X_i, S_b) | X_i}^2}\,.
\end{equation}
The incrementality assumption is used in the asymptotic normality proof to argue that the weights have sufficiently high variance that all data points have some influence on the estimate.
From the technical side, we use the H\'{a}jek projection to analyze our $U$-statistic estimator. Incrementality ensures that there is sufficiently weak dependence in the weights across a sequence of sub-samples and hence the central limit theorem applies. 
As discussed, the sub-sampling size $s$ can be used to control the variance of the weights, and so incrementality and shrinkage are related.
We make this precise, proving that incrementality can be lower bounded as a function of kernel shrinkage, so that having a sufficiently low shrinkage rate enables both estimation and inference.
These general results could be of independent interest beyond the scope of this work.

For the second part of our analysis, we specialize to the case where the base kernel is the $k$-NN kernel, for some constant $k$.  
We prove that both shrinkage and incrementality depend only on the intrinsic dimension $d$, rather than the explicit dimension $D$.
In particular, we show that $\epsilon(s) = O(s^{-1/d})$ and $\eta(s) = \Theta(1/s)$. These lead to our main theorem \emph{that the sub-sampled $k$-NN estimate achieves an estimation rate of order $n^{1/(d+2)}$ and is also $n^{1/(d+2)}$-asymptotically normal.} 

In the third part, we provide a closed form characterization of the asymptotic variance of the sub-sampled $k$-NN estimate, as a function of the conditional variance of the moments, which is defined as $\sigma^2(x) = \Var\bp{\psi(Z;\theta) \mid X=x}$. For example, for the $1$-NN kernel, the asymptotic variance is given by
\begin{equation*}
\Var(\hat{\theta}(x)) = \frac{\sigma^2(x) s^2}{n(2s-1)}\,.
\end{equation*} 
This strengthens prior results of \cite{fan2018dnn} and \cite{wager2017estimation}, which only proved the existence of an asymptotic variance without providing an explicit form (and thereby relied on bootstrap approaches for the construction of confidence intervals). Our tight characterization enables an easy construction of plugin normal-based intervals that only require a preliminary estimate of $\sigma(x)$. Our Monte Carlo study shows that such intervals provide very good finite sample coverage in a high dimensional regression setup (see Figure \ref{fig:distribution})\footnote{See Appendix \ref{app:sim} for detailed explanation of our simulations.}.

Finally in the last part, we discuss an adaptive data-driven approach for picking the sub-sample size $s$ so as to achieve estimation or asymptotic normality with rates that only depend on the unknown intrinsic dimension. This allows us to achieve near-optimal rates while adapting to the unknown intrinsic dimension of data (see Propositions \ref{thm:est-adapt} and \ref{thm:normal-adapt}). Figure~\ref{fig:coverage1d} depicts the performance of our adaptive approach compared to two benchmarks, one constructed based on theory for intrinsic dimension $d$ which may be unknown, and the other one constructed na\"ively based on the known but sub-optimal extrinsic dimension $D$. As it can be observed from this figure, setting $s$ based on intrinsic dimension $d$ allows us to build more accurate and smaller confidence intervals, which is crucial for drawing inference in the high-dimensional finite sample regime. Our adaptive approach uses samples to pick $s$ very close to the value suggested by our theory and therefore leads to a compelling finite sample coverage\footnote{A preliminary implementation of our code is available via \href{http://github.com/khashayarkhv/np_inference_intrinsic}{http://github.com/khashayarkhv/np\_inference\_intrinsic}.}. 

Our results shed some light on the importance of using adaptive machine learning based estimators, such as nearest neighbor based estimates, when performing estimation and inference in high-dimensional econometric problems. Such estimators address the curse of dimensionality by adapting to a priori unknown latent structure in the data. Moreover, coupled with the powerful sub-sampling based averaging approach, such estimators can maintain their adaptivity, while also satisfying asymptotic normality and thereby enabling asymptotically valid inference; a property that is crucial for embracing such approaches in econometrics and causal inference.

\begin{figure}[h]
    \centering
        \includegraphics[height = 14cm]{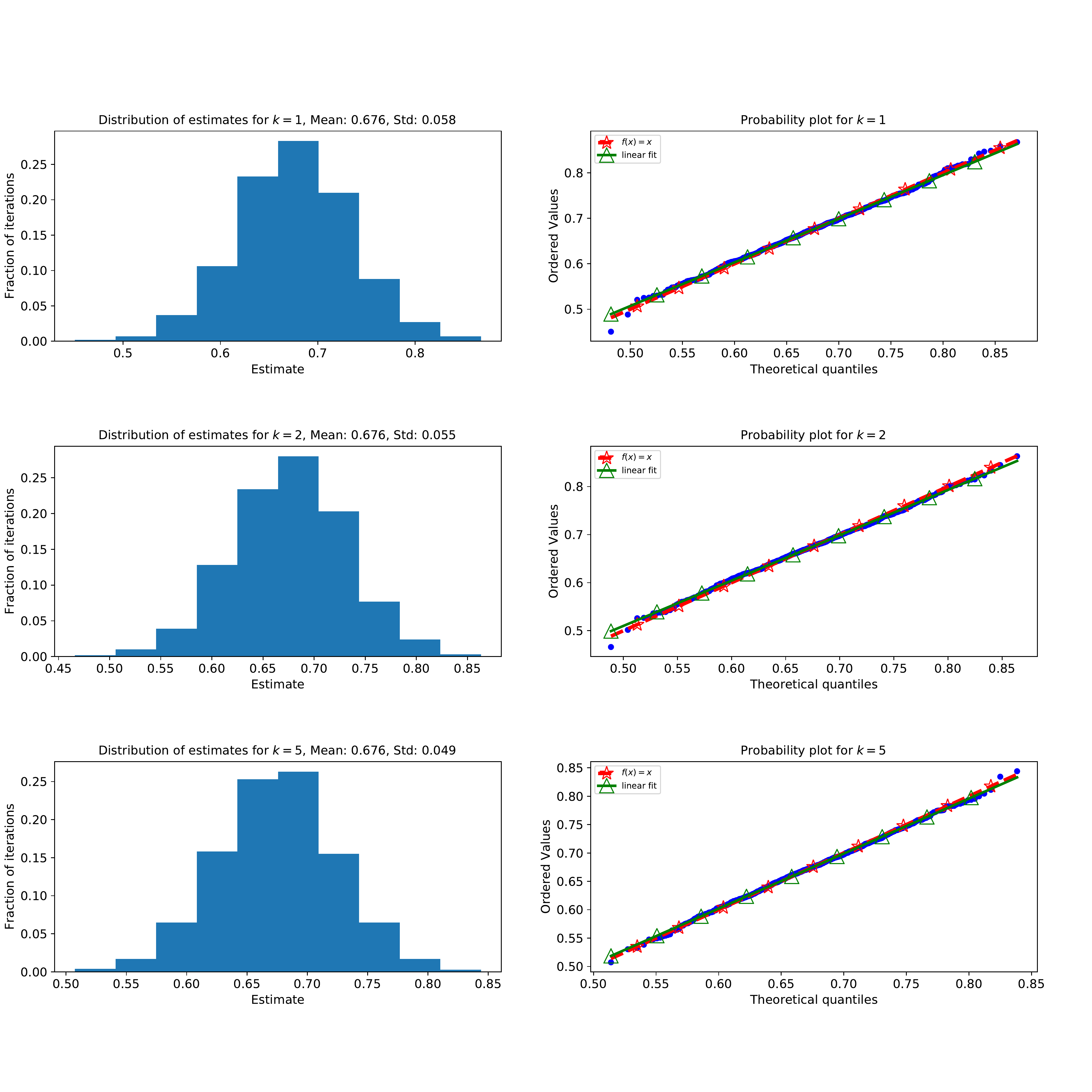}
        \caption{Left: distribution of estimates over $1000$ Monte Carlo runs for $k=1,2,5$. Right: the quantile-quantile plot when comparing to the theoretical asymptotic normal distribution of estimates stemming from our characterization, whose means are $0.676, 0.676,$ and $0.676$ for $k=1,2,5$, respectively. Standard deviations are $0.058, 0.055,$ and $0.049$ for $k=1,2,5$ respectively. $n=20000$, $D=20$, $d=2$, $\E[Y|X] = \frac{1}{1+\exp\{-3 X[0]\}}$, $\sigma=1$. Test point: $x[0]\approx 0.245$, $\E[Y|X=x]\approx 0.676$.}
        \label{fig:distribution}
\end{figure}

\begin{figure}
\centering
\includegraphics[height=14cm]{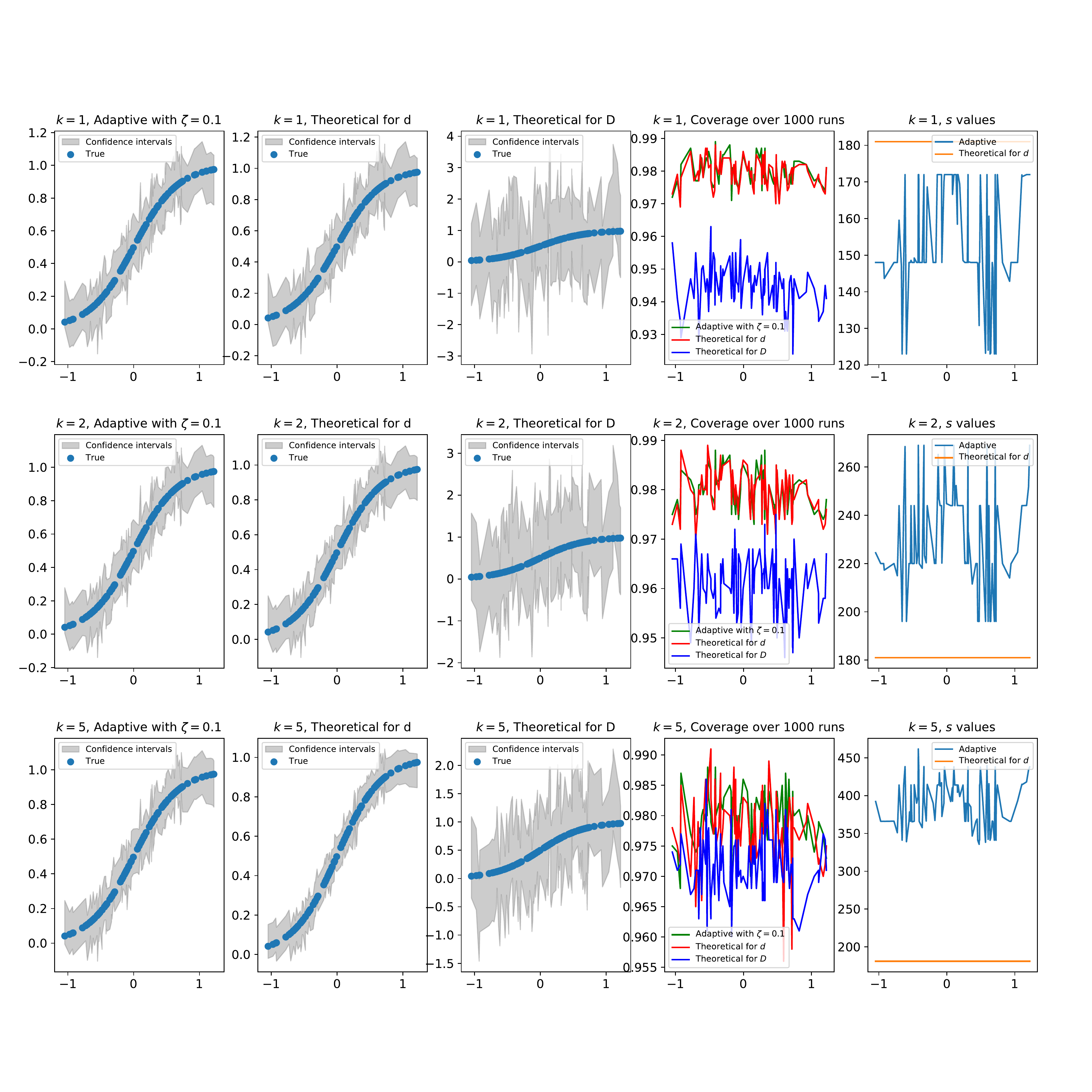}
\caption{Confidence interval and true values for $100$ randomly sampled test points on a single run for $k=1,2,5$ and when (1) left: $s=s_{\zeta}$ is chosen adaptively using Proposition \ref{thm:normal-adapt} with $\zeta=0.1$, (2) second from the left: $s=n^{1.05 d/(d+2)}$, and (3) middle: $s=n^{1.05 D/(D+2)}$. Second from the right: coverage over $1000$ runs for three different methods described. Right: average value of $s_{\zeta}$ chosen adaptively using Proposition \ref{thm:normal-adapt} for $\zeta = 0.1$ for different test points compared to the theoretical value $s=n^{1.05 d/(d+2)}$. Here $n=20000$, $D=20$, $d=2$, $\E[Y|X] = \frac{1}{1+\exp\{-3 X[0]\}}$, $\sigma=1$. Nominal coverage: $0.98$.}\label{fig:coverage1d}
\end{figure}

\paragraph{Structure of the paper.} The rest of the paper is organized as follows. In \S \ref{sec:prelim}, we provide preliminary definitions, in \S \ref{ssec:gen} and \S \ref{ssec:DNN} we explain our algorithms, in \S \ref{ssec:intrinsic-dimension} we explain doubling dimension (see Appendix \ref{ssec:examples} for examples). In \S \ref{sec:assumptions} we state our assumptions, in \S \ref{sec:subsamp-kernel} we provide general estimation and inference results for kernels that satisfy shrinkage and incrementality conditions, and in \S \ref{sec:main-thm} we apply such results to the $k$-NN kernel and prove estimation and inference rates for such kernels that only depend on intrinsic dimension. We defer a discussion on the extension to heterogeneous treatment effect estimation and also technical proofs to Appendices.

\section{Preliminaries}\label{sec:prelim}

Suppose we have a data set $M$ of $n$ observations $Z_1, Z_2, \ldots, Z_n$ drawn independently from some distribution $\calD$ over the observation domain $\calZ$. We focus on the case that $Z_i = (X_i, Y_i)$, where $X_i$ is the vector of covariates and $Y_i$ is the outcome. In Appendix \ref{sec:nuisance}, we briefly discuss how our results can be extended to the setting where nuisance parameters and treatments are included in the model.

Suppose that the covariate space $\calX \subset \reals^D$ is contained in a ball with unknown diameter $\Delta_\calX$. Denote the marginal distribution of $X$ by $\mu$ and the empirical distribution of $X$ on $n$ sample points by $\mu_n$. Let $B(x,r) = \bc{z \in \reals^D: \|x-z\|_2 < r}$ be the $\ell_2$-ball centered at $x$ with radius $r$ and denote the standard basis for $\reals^p$ by $\bc{e_1,e_2,\ldots,e_p}$.

Let $\psi: \calZ \times \reals^p \to \reals^p$ be a score function that maps observation $Z$ and parameter $\theta \in \reals^p$ to a $p$-dimensional score $\psi(Z; \theta)$. For $x \in \calX$ and $\theta \in \reals^p$ define the expected score as
$m(x; \theta) = \E[\psi(Z; \theta) \mid X=x]$. The goal is to estimate the quantity $\theta(x)$ via local moment condition, i.e.
\begin{equation*}
\theta(x) \text{~solves:~} m(x; \theta) = \E[\psi(Z; \theta)\mid X=x] = 0.
\end{equation*}

\subsection{Sub-Sampled Kernel Estimation}\label{ssec:gen}

\paragraph{Base Kernel Learner.} Our learner $\calL_k$ takes a data set $S$ containing $m$ observations as input and a realization of internal randomness $\omega$, and outputs a kernel weighting function $K_{\omega}:\calX \times \calX \times \calZ^m \to [0,1]$. In particular, given any target feature $x$ and the set $S$, the weight of each observation $Z_i$ in $S$ with feature vector $X_i$ is $K_{\omega}(x,X_i,S)$. Define the weighted score on a set $S$ with internal randomness $\omega$ as
$\Psi_S(x; \theta) = \sum_{i\in S} K_{\omega}(x,X_i,S) \psi(Z_i; \theta)$. When it is clear from context we will omit $\omega$ from our notation for succinctness and essentially treat $K$ as a random function. For the rest of the paper, we are going to use notations $\alpha_{S,\omega}(X_i) = K_{\omega}(x,X_i,S)$ interchangeably. 

\paragraph{Averaging over $B$ sub-samples of size $s$.} Suppose that we consider $B$ random and independent draws from all $\binom{n}{s}$ possible subsets of size $s$ and internal randomness variables $\omega$ and look at their average. Index these draws by $b=1,2,\ldots,B$ where $S_b$ contains samples in $b$th draw and $\omega_b$ is the corresponding draw of internal randomness. We can define the weighted score as 
\begin{equation}
\label{eqn:incomp-u-stat}
\Psi(x; \theta) = \frac{1}{B} \sum_{b=1}^B \Psi_{S_b,\omega_b}(x; \theta) = \frac{1}{B} \sum_{b=1}^B \sum_{i\in S_b} \alpha_{S_b,\omega_b}(X_i) \psi(Z_i; \theta) \,.
\end{equation}

\paragraph{Estimating $\theta(x)$.} We estimate $\theta(x)$ as a vanishing point of $\Psi(x; \theta)$. Letting $\htheta$ be this point, then
$\Psi(x; \htheta) = \frac{1}{B} \sum_{b=1}^B \sum_{i=1}^n \alpha_{S_b,\omega_b}(X_i) \psi(Z_i; \htheta) = 0.$ This procedure is explained in Algorithm \ref{alg:gen}.

\subsection{Sub-Sampled $k$-NN Estimation}
\label{ssec:DNN}

We specially focus on the case that the weights are distributed across the $k$-NN of $x$. In other words, given a data set $S$, the weights are given according to $K_{\omega}(x,X_i,S) = \,1 \bc{X_i \in H_k(x,S)}/k$, where $H_k(x,S)$ are $k$-NN of $x$ in the set $S$. The pseudo-code for this can be found in Algorithm \ref{alg:DNN}.

\paragraph{Complete $U$-statistic.} The expression in Equation \eqref{eqn:incomp-u-stat} is an incomplete $U$-statistic. Complete $U$-statistic is obtained if we allow each subset of size $s$ from $n$ samples to be included in the model exactly once. In other words, this is achieved if $B=\binom{n}{s}$, all subsets $S_1, S_2, \ldots, S_B$ are distinct, and we also take expectation over the internal randomness $\omega$. Denoting this by $\Psi_0(x;\theta)$, we have
\begin{equation}\label{eqn:comp-u-stat}
   \Psi_0(x;\theta) = \binom{n}{s}^{-1} \sum_{S \in [n]: |S|=s}  \E_{\omega}\bb{\sum_{i\in S} \alpha_{S,\omega}(X_i) \psi(Z_i; \theta)}\,.
\end{equation}
Note in the case of $k$-NN estimator we can also represent $\Psi_0$ in terms of order statistics, i.e., $\Psi_0$ is an $L$-statistics (see e.g., \cite{serfling2009approximation}). By sorting samples in $\bX = \bc{X_1,X_2,\ldots,X_n}$ based on their distance with $x$ as $\|X_{(1)}-x\| \leq \|X_{(2)}-x\| \leq \cdots \leq \|X_{(n)}-x\|$, we can write $\Psi_0(x; \theta) = \sum_{i=1}^n \alpha(X_{(i)})\, \psi(Z_{(i)};\theta)$ where the weights are given by
\[
\alpha(X_{(i)}) = 
\begin{cases}
    \frac{1}{k} \binom{n}{s}^{-1} \, \binom{n-i}{s-1} & \text{~if~} i \leq k \\ 
    \frac{1}{k} \binom{n}{s}^{-1} \, \sum_{j=0}^{k-1} \binom{i-1}{j} \binom{n-i}{s-1-j} &\text{~if~} i \geq k+1 \,.
\end{cases}
\]
\vspace{-1mm}
\begin{minipage}{0.48\textwidth}
\begin{algorithm}[H]
    \centering
    \caption{Sub-Sampled Kernel Estimation}\label{alg:gen}
    \begin{algorithmic}[1]
        \State {\bf Input.} Data $\bc{Z_i=(X_i, Y_i)}_{i=1}^n$, moment $\psi$, kernel $K$, sub-sampling size $s$, number of iterations $B$
        \State {\bf Initialize.} $\alpha(X_i) = 0, 1 \leq i \leq n$
        \For{$b \gets 1, B$}
        \State \textbf{Sub-sampling.} Draw set $S_b$ by sampling $s$ points from $Z_1,Z_2,\ldots,Z_n$ without replacement.
        \State \textbf{Weight Updates.} $\alpha(X_i) \gets \alpha(X_i) + K_{\omega_b}(x,X_i,S_b)$
        \EndFor
        \State \textbf{Weight Normalization.}  $\alpha(X_i) \gets \alpha(X_i)/B$
        \State \textbf{Estimation.} Denote $\hat{\theta}$ as a solution of $\Psi(x;\theta)=\sum_{i=1}^n \alpha(X_i) \psi(Z_i;\theta)=0$
    \end{algorithmic}
\end{algorithm}
\end{minipage}
\hfill
\begin{minipage}{0.48\textwidth}
\begin{algorithm}[H]
    \centering
    \caption{Sub-Sampled $k$-NN Estimation}\label{alg:DNN}
    \begin{algorithmic}[1]
    
        \State {\bf Input.} Data $\bc{Z_i=(X_i, Y_i)}_{i=1}^n$, moment $\psi$, sub-sampling size $s$, number of iterations $B$, number of neighbors $k$
        \State {\bf Initialize.} $\alpha(X_i) \gets 0, 1 \leq i \leq n$
        \For{$b \gets 1,B$}
        \State \textbf{Sub-sampling.} Draw set $S_b$ by sampling $s$ points from $Z_1,Z_2,\ldots,Z_n$ without replacement
        \State \textbf{Weight Updates.} $\alpha(X_i) \gets \alpha(X_i) + \,1 \bc{X_i \in H_k(x,S_b)}/k $
        \EndFor
        \State \textbf{Weight Normalization.}  $\alpha(X_i) \gets \alpha(X_i)/B$
        \State \textbf{Estimation.} Denote $\hat{\theta}$ as a solution of $\Psi(x;\theta)=\sum_{i=1}^n \alpha(X_i) \psi(Z_i;\theta)=0$
    \end{algorithmic}
\end{algorithm}
\end{minipage}

\subsection{Local intrinsic dimension}\label{ssec:intrinsic-dimension}
We are interested in settings that the distribution of $X$ has some low dimensional structure on a ball around the target point $x$. The following notions are adapted from \cite{kpotufe2011k}, which we present here for completeness.

\begin{definition}\label{def:doubling-dim}
The marginal $\mu$ is called {\bf doubling measure} if there exists a constant $C_{db}>0$ such that for any $x \in \calX$ and any $r>0$ we have $\mu(B(x,r)) \leq C_{db} \mu(B(x,r/2))$.
\end{definition}
An equivalent definition of this notion is that, the measure $\mu$ is doubling measure if there exist $C, d>0$ such that for any $x \in \calX, r>0$, and $\theta \in (0,1)$ we have $\mu(B(x,r)) \leq C \theta^{-d} \mu(B(x,\theta r))$.

One example is given by Lebesgue measure on the Euclidean space $\reals^d$, where for any $r>0, \theta \in(0,1)$ we have $\vol(B(x,\theta r)) = \vol(B(x,r)) \theta^d$. Building upon this, let $\calX \in \reals^D$ be a subset of $d$-dimensional hyperplane and suppose that for any ball $B(x,r)$ in $\calX$ we have $\vol(B(x,r) \cap \calX)=\Theta(r^d)$. If $\mu$ is almost uniform, then we also have $\mu(B(x,\theta r))/\mu(B(x,r)) = \Theta(\theta^d)$.

Unfortunately, this global notion of doubling measure is restrictive and most probability measures are globally complex. Rather, once restricted to local neighborhoods, the measure becomes lower dimensional and intrinsically less complex. The following definition captures this intuition better.

\begin{definition}\label{def:hom-measure}
Fix $x \in \calX$ and $r > 0$. The marginal $\mu$ is $(C,d)$-{\bf homogeneous on} $B(x,r)$ if for any $\theta \in (0,1)$ we have $\mu(B(x, r)) \leq C \theta^{-d} \mu(B(x, \theta r))$. 
\end{definition}

Intuitively, this definition requires the marginal $\mu$ to have a local support that is intrinsically $d$-dimensional. This definition covers low-dimensional manifolds, mixture distributions, $d$-sparse data, and also any combination of these examples. These examples are explained in Appendix \ref{ssec:examples}.

\section{Assumptions}\label{sec:assumptions}
For non-parametric estimators the bias is connected to the kernel shrinkage, as noted by \cite{athey2019generalized, wager2017estimation, oprescu2018orthogonal}.
\begin{definition}[Kernel Shrinkage in Expectation]\label{def:ker-shrink}
The kernel weighting function output by learner $\calL_k$ when it is given $s$ i.i.d. observations drawn from distribution $\calD$ satisfies 
\begin{equation*}
\E \bb{\sup \bc{\|x - X_i\|_2 : K(x,X_i,S)>0}} = \epsilon(s)\,.
\end{equation*}
\end{definition}

\begin{definition}[Kernel Shrinkage with High Probability]\label{def:ker-shrink-whp}
The kernel weighting function output by learner $\calL_k$ when it is given $s$ i.i.d. observations drawn from distribution $\calD$ w.p. $1-\delta$ over the draws of the $s$ samples satisfies
\begin{equation*}
\sup \bc{\|x - X_i\|_2: K(x,X_i,S)>0} \leq \epsilon(s,\delta)\,.
\end{equation*}
\end{definition}

As shown in \cite{wager2017estimation}, for trees that satisfy some regularity condition, $\epsilon(s) \leq s^{-c/D}$ for a constant $c$. We are interested in shrinkage rates that scale as $s^{-c/d}$, where $d$ is the local intrinsic dimension of $\mu$ on $B(x,r)$. Similar to \cite{oprescu2018orthogonal, athey2019generalized}, we rely on the following assumptions on the moment and score functions. 

\begin{assumption}\label{ass:mse}
\begin{enumerate}
\item[]
\item The moment $m(x; \theta)$ corresponds to the gradient w.r.t. $\theta$ of a $\lambda$-strongly convex loss $L(x; \theta)$. This also means that the Jacobian $M_0= \nabla_\theta m(x; \theta(x))$ has minimum eigenvalue at least $\lambda$.
\item For any fixed parameters $\theta$, $m(x; \theta)$ is a $L_m$-Lipschitz function in $x$ for some constant $L_m$. 
\item There exists a bound $\psi_{\max}$  such that for any observation $z$ and any $\theta$, $\|\psi(z; \theta)\|_{\infty}\leq \psi_{\max}$.
\item The bracketing number $N_{[]}({\cal F}, \epsilon, L_2)$ of the function class: ${ \cal F}=\{\psi(\cdot; \theta): \theta \in \Theta\}$, satisfies $\log(N_{[]}({\cal F}, \epsilon, L_2))= O(1 /\epsilon)$.
\end{enumerate}
\end{assumption}

\begin{assumption}\label{ass:normality}
  \begin{enumerate}
    \item[]
     \item For any coordinate $j$ of the moment vector $m$, the Hessian $H_j(x; \theta)=\nabla^2_{\theta\theta} m_j(x;\theta)$ has eigenvalues bounded above by a constant $L_H$ for all $\theta$.
     \item Maximum eigenvalue of $M_0$ is upper bounded by $L_J$.
     \item Second moment of $\psi(x;\theta)$ defined as $\Var\bp{\psi(Z;\theta) \mid X=x}$ is $L_{mm}$-Lipschitz in $x$, i.e.,
     \begin{equation*}
       \|\Var\bp{\psi(Z;\theta) \mid X=x} - \Var\bp{\psi(Z;\theta) \mid X=x'}\|_F  \leq L_{mm} \|x-x'\|_2\,.
     \end{equation*}
     \item Variogram is Lipschitz:
    $\sup_{x\in \calX} \|\Var(\psi(Z; \theta) - \psi(Z; \theta') \mid X=x)\|_F \leq L_{\psi} \|\theta - \theta'\|_2$.
  \end{enumerate}
\end{assumption}

The condition on variogram always holds for a $\psi$ that is Lipschitz in $\theta$. This larger class of functions $\psi$ allows estimation in more general settings such as $\alpha$-quantile regression that involves a $\psi$ which is non-Lipschitz in $\theta$.
Similar to \cite{athey2016recursive, athey2019generalized}, we require kernel $K$ to be \emph{honest} and \emph{symmetric}.

\begin{assumption}\label{ass:honest}
The kernel $K$, built using samples $\bc{Z_1,Z_2,\ldots,Z_s}$, is {\bf honest} if the weight of sample $i$ given by $K(x,X_i,\bc{Z_j}_{j=1}^s)$ is independent of $Y_j$ conditional on $X_j$ for any $j \in [s]$.
\end{assumption}

\begin{assumption}\label{ass:sym}
The kernel $K$, built using samples $\bc{Z_1, Z_2, \ldots, Z_s}$, is {\bf symmetric} if for any permutation $\pi:[s] \to [s]$, the distribution of $K(x,X_i,\bc{Z_j}_{j=1}^s)$ and $K(x, X_{\pi(i)}, \bc{Z_{\pi(j)}}_{j=1}^s)$ are equal. In other words, the kernel weighting distribution remains unchanged under permutations.
\end{assumption}

For a deterministic kernel $K$, the above condition implies that $K(x,X_i,\{Z_j \}_{j=1}^s) = K(x, X_i, \{Z_{\pi(j)}\}_{j=1}^s)$, for any $i \in [s]$. In the next section, we provide general estimation and inference results for a general kernel based on the its shrinkage and incrementality rates. 

\section{Guarantees for sub-sampled kernel estimators}\label{sec:subsamp-kernel}
Our first result establishes estimation rates, both in expectation and high probability, for kernels based on their shrinkage rates. The proof of this theorem is deferred to Appendix \ref{app:mse_rate}.  
\begin{theorem}[Finite Sample Estimation Rate]\label{thm:mse_rate}
Let Assumptions \ref{ass:mse} and \ref{ass:honest} hold. Suppose that Algorithm \ref{alg:gen} is executed with $B \geq n/s$. If the base kernel $K$ satisfies kernel shrinkage in expectation, with rate $\epsilon(s)$, then w.p. $1-\delta$
\begin{equation}
    \|\htheta - \theta(x)\|_2 \leq \frac{2}{\lambda}\left(L_m \epsilon(s) + O\left(\psi_{\max}\sqrt{\frac{p\, s}{n} \left(\log\log(n /s) + \log(p/\delta)\right)}\right)\right)\,.
\end{equation}
Moreover,
\begin{equation}
     \sqrt{\E\bb{\|\htheta - \theta(x)\|_2^2}} \leq \frac{2}{\lambda}\left(L_m \epsilon(s) + O\left(\psi_{\max}\sqrt{\frac{p\, s}{n} \log\log(p\, n /s)}\right)\right)\,.
\end{equation}
\end{theorem}
The next result establishes asymptotic normality of sub-sampled kernel estimators. In particular, it provides coordinate-wise asymptotic normality of our estimate $\htheta$ around its true underlying value $\theta(x)$. The proof of this theorem is deferred to Appendix \ref{app:normality}.
\begin{theorem}[Asymptotic Normality]\label{thm:normality}
Let Assumptions \ref{ass:mse}, \ref{ass:normality}, \ref{ass:honest}, and \ref{ass:sym} hold. Suppose that Algorithm \ref{alg:gen} is executed with $B \geq (n/s)^{5/4}$ and the base kernel $K$ satisfies kernel shrinkage, with rate $\epsilon(s, \delta)$ in probability and $\epsilon(s)$ in expectation. Let $\eta(s)$ be the incrementality of kernel $K$ defined in Equation \eqref{eqn:incrementality} and $s$ grow at a rate such that $s\rightarrow \infty$, $n \eta(s) \rightarrow \infty$, and $\epsilon(s,\eta(s)^2) \rightarrow 0$. Consider any fixed coefficient $\beta\in \reals^p$ with $\|\beta\|\leq 1$ and define the variance as
$$\sigma_{n,\beta}^2(x) = \frac{s^2}{n}\Var\bb{\E\bb{\sum_{i=1}^s K(x, X_i, \{Z_j\}_{j=1}^s) \ldot{\beta}{M_0^{-1} \psi(Z_i; \theta(x))}\mid Z_1}}.$$
Then it holds that $\sigma_{n,\beta}(x) = \Omega\left(s \sqrt{\eta(s)/n}\right)$. Moreover, suppose that
\begin{equation}
\label{eqn:rate-condition-normality}
\max \left(\epsilon(s), \epsilon(s)^{1/4} \bp{\frac{s}{n} \,\log\log(n/s) }^{1/2}, \bp{\frac{s}{n} \, \log\log(n/s)}^{5/8} \right) = o(\sigma_{n,\beta}(x))\,.
\end{equation}
Then,
\[
\frac{\ldot{\beta}{\hat{\theta} - \theta(x)}}{\sigma_{n,\beta}(x)} \rightarrow_d \normal(0, 1)\,.
\]
\end{theorem}

 Theorems \ref{thm:mse_rate} and \ref{thm:normality} generalize existing estimation and asymptotic normality results of \cite {athey2019generalized,wager2017estimation, fan2018dnn} to an arbitrary kernel that satisfies appropriate shrinkage and incrementality rates (see Remark \ref{rem:inc} in Appendix \ref{app:normality}). 
 The following lemma relates these two and provides a lower bound on the incrementality in terms of kernel shrinkage. The proof uses the Paley-Zygmund inequality and is left to Appendix \ref{app:incrementality}.

\begin{lemma}\label{lem:shrinkage_incrementality}
For any symmetric kernel $K$ (Assumption \ref{ass:sym}) and for any $\delta\in [0,1]$:
\begin{equation*}
    \eta_s = \E \bb{\E\bb{K(x, X_1, \{Z_j\}_{j=1}^s) \mid X_1}^2} \geq \frac{\left(1-\delta\right)^2\,(1/s)^2}{\inf_{\rho>0} \left(\mu(B(x, \epsilon(s, \rho)))+\rho\,s/\delta\right)}\,.
\end{equation*}
Thus if $\mu(B(x, \epsilon(s, 1/(2s^2))))=O(\log(s)/s)$, then picking $\rho=1/(2s^2)$ and $\delta = 1/2$ implies that $\E[\E[K(x, X_1, \{Z_j\}_{j=1}^s) |X_1]^2] = \Omega(1/s\log(s)).$
\end{lemma}
\begin{corollary}\label{cor:incrementality}
If $\epsilon(s, \delta)=O((\log(1/\delta)/s)^{1/d})$ and $\mu$ satisfies a two-sided version of the doubling measure property on $B(x,r)$, defined in Definition \ref{def:hom-measure}, i.e., $c \theta^d \mu(B(x, r)) \leq \mu(B(x,\theta r)) \leq C \theta^d \mu(B(x,r))$
for any $\theta \in (0,1)$. Then, $\E[\E[K(x, X_1, \{Z_j\}_{j=1}^s) |X_1]^2] = \Omega(1/(s\log(s)))$.
\end{corollary}

Even without this extra assumption, we can still characterize the incrementality rate of the $k$-NN estimator, as we observe in the next section.

\section{Main theorem: adaptivity of $k$-NN estimator}\label{sec:main-thm}

In this section, we provide estimation guarantees and asymptotic normality of the $k$-NN estimator by using Theorems \ref{thm:mse_rate} and \ref{thm:normality}. We first establish shrinkage and incrementality rates for this kernel.

\subsection{Estimation guarantees for the $k$-NN estimator} \label{ssec:kNN-est}
In this section we provide shrinkage results for the $k$-NN kernel. As observed in Theorem \ref{thm:mse_rate}, shrinkage rates are sufficient for bounding the estimation error. The shrinkage result that we present in the following would only depend on the local intrinsic dimension of $\mu$ on $B(x,r)$.

\begin{lemma}[High probability shrinkage for the $k$-NN kernel]\label{lem:pb-kernel-shrink-kNN}
    Suppose that the measure $\mu$ is $(C,d)$-homogeneous on $B(x,r)$. Then, for any $\delta$ satisfying $2 \exp \bp{-\mu(B(x,r))s/(8C)} \leq \delta \leq \frac{1}{2} \exp(-k/2),$ w.p. at least $1-\delta$ we have 
\begin{equation*}
\|x - X_{(k)}\|_2 \leq \epsilon_k(s,\delta) = O \bp{\frac{\log(1/\delta)}{s}}^{1/d}\,.
\end{equation*}
\end{lemma}
We can turn this into a shrinkage rate in expectation as follows. In fact, by the very convenient choice of $\delta = s^{-1/d}$ combined with the fact that $\calX$ has diameter $\Delta_{\calX}$, we can establish $O\bp{(\log(s)/s)^{1/d}}$ rate on expected kernel shrinkage. However, a more careful analysis would help us to remove the $\log(s)$ dependency in the bound and is stated in the following corollary:

\begin{corollary}[Expected shrinkage for the $k$-NN kernel]\label{cor:exp-kernel-shrink-kNN}
Suppose that the conditions of Lemma \ref{lem:pb-kernel-shrink-kNN} hold. Let $k$ be a constant and $\epsilon_k(s)$ be the expected shrinkage for the $k$-NN kernel. Then, for any $s$ larger than some constant we have
$\epsilon_k(s)=\E \bb{\|x - X_{(k)}\|_2} = O \bp{\frac{1}{s}}^{1/d}$.
\end{corollary}

We are now ready to state our estimation result for the $k$-NN kernel, which is honest and symmetric. Therefore, we can substitute the expected shrinkage rate established in Corollary \ref{cor:exp-kernel-shrink-kNN} in Theorem \ref{thm:mse_rate} to derive estimation rates for this kernel.
\begin{theorem}[Estimation Guarantees for the $k$-NN Kernel]\label{thm:knn-mse}
Suppose that $\mu$ is $(C,d)$-homogeneous on $B(x,r)$, Assumption \ref{ass:mse} holds and that Algorithm \ref{alg:DNN} is executed with $B\geq n/s$. Then, w.p. $1-\delta$:
\begin{equation}
    \|\htheta - \theta(x)\|_2 \leq \frac{2}{\lambda}\left(O \left(s^{-1/d} \right) + O\left(\psi_{\max}\sqrt{\frac{p\, s}{n} \left(\log\log(n /s) + \log(p/\delta)\right)}\right)\right)\,,
\end{equation}
and
\begin{equation}
\sqrt{\E\bb{\|\htheta - \theta(x)\|_2^2}} \leq \frac{2}{\lambda}\bp{O \bp{s^{-1/d}} + O \bp{\psi_{\max}\sqrt{\frac{s\, p\, \log \log(p\,n/s)}{n}}}} \,.
\end{equation}
By picking $s = \Theta \bp{n^{d/(d+2)}}$ and $B = \Omega \bp{n^{2/(d+2)}}$ we get $\sqrt{\E\bb{\|\htheta - \theta(x)\|_2^2}} = \tilde{O}\bp{n^{-1/(d+2)}}$.
\end{theorem}

\subsection{Asymptotic normality of the $k$-NN estimator}\label{ssec:kkN-normality}
In this section we prove asymptotic normality of $k$-NN estimator. We start by provide bounds on the incrementality of the $k$-NN kernel.

\begin{lemma}[$k$-NN Incrementality]\label{lem:kNN-inc}
Let $K$ be the $k$-NN kernel and let $\eta_k(s)$ denote the incrementality rate of this kernel. Then, the following holds:
\begin{equation*}
     \eta_k(s)=\E\bb{\E\bb{K(x, X_1, \{Z_j\}_{j=1}^s) \mid X_1}^2} = \frac{1}{(2s-1)\,k^2} \bp{\sum_{t=0}^{2k-2}\frac{a_t}{b_t}}\,,
\end{equation*}
where sequences $\bc{a_t}_{t=0}^{2k-2}$ and $\bc{b_t}_{t=0}^{2k-2}$ are defined as 
\begin{equation*}
    a_t = \sum_{i=\max \bc{0, t-(k-1)}}^{\min \bc{t, k-1}} \binom{s-1}{i}\,\binom{s-1}{t-i} \quad \quad \text{and} \quad \quad
    b_t = \sum_{i=0}^t \binom{s-1}{i}\,\binom{s-1}{t-i}
\end{equation*}
\end{lemma}

We can substitute $\eta_k(s)$ in Theorem \ref{thm:normality} to prove asymptotic normality of the $k$-NN estimator. Before that, we establish the asymptotic variance of this estimator $\sigma_{n,j}(x)$, up to the smaller order terms. 

\begin{theorem}[Asymptotic Variance of $k$-NN]\label{thm:knn_var_range}
Let $j \in [p]$ be one of coordinates. Suppose that $k$ is constant while $s\rightarrow \infty$. Then, for the $k$-NN kernel
\begin{equation}
\label{eqn:knn-var}
    \sigma_{n,j}^2(x) = \frac{s^2}{n}\frac{\sigma_{j}^2(x)}{k^2\,(2s-1)}\,\zeta_k  + o(s/n)\,,
\end{equation}
where $\sigma_j^2(x)=\Var\bb{\ldot{e_j}{M_0^{-1}\psi(Z; \theta(x))}\mid X=x}$ and $\zeta_k = k + \sum_{t=k}^{2k-2} 2^{-t} \sum_{i=t-k+1}^{k-1} \binom{t}{i}.$
\end{theorem}

Combining results of Theorem \ref{thm:normality}, Theorem \ref{thm:knn_var_range}, Corollary \ref{cor:exp-kernel-shrink-kNN}, and Lemma \ref{lem:kNN-inc} we have:

\begin{theorem}[Asymptotic Normality of $k$-NN Estimator]\label{thm:knn-normality}
Suppose that $\mu$ is $(C,d)$-homogeneous on $B(x,r)$. Let Assumptions \ref{ass:mse}, \ref{ass:normality} hold and suppose that Algorithm \ref{alg:DNN} is executed with $B\geq (n/s)^{5/4}$ iterations. Suppose that $s$ grows at a rate such that $s \rightarrow \infty$, $n/s \rightarrow \infty$, and also $s^{-1/d} (n/s)^{1/2} \rightarrow 0$. Let $j \in [p]$ be one of coordinates and $\sigma^2_{n,j}(x)$ be defined in Equation \eqref{eqn:knn-var}. Then, 
\begin{equation*}
 \frac{\htheta_j(x)-\theta_{j}(x)}{\sigma_{n,j}(x)} \rightarrow \normal(0,1)\,.
\end{equation*}
Finally, if $s = n^{\beta}$ and $B\geq n^{\frac{5}{4}(1-\beta)}$ with $\beta \in (d/(d+2),1)$. Then,
\begin{equation*}
\frac{\htheta_j(x)-\theta_{j}(x)}{\sigma_{n,j}(x)} \rightarrow \normal(0,1)\,.
\end{equation*}
\end{theorem}


\paragraph{Plug-in confidence intervals.} Observe that the Theorem \ref{thm:knn_var_range} implies that if we define $\tilde{\sigma}_{n,j}^2(x)=\frac{s^2}{n}\frac{\sigma_j^2(x)}{2s-1}\,\frac{\zeta_k}{k^2}$ as the leading term in the variance, then $\frac{\sigma_{n,j}^2(x)}{\tilde{\sigma}_{n,j}^2(x)}\rightarrow_p 1$. Thus, due to Slutsky's theorem
\begin{equation}
    \frac{\hat{\theta}_j-\theta_{j}}{\tilde{\sigma}_{n,j}^2(x)}=\frac{\hat{\theta}_j-\theta_{j}}{\sigma_{n,j}^2(x)}\frac{\sigma_{n,j}^2(x)}{\tilde{\sigma}_{n,j}^2(x)} \rightarrow_d \normal(0,1)\,.
\end{equation}
Hence, we have a closed form solution to the variance in our asymptotic normality theorem. If we have an estimate $\hat{\sigma}_j^2(x)$ of the variance of the conditional moment around $x$, then we can build plug-in confidence intervals based on the normal distribution with variance $\frac{s^2}{n}\frac{\hat{\sigma}_j^2(x)}{2s-1}\,\frac{\zeta_k}{k^2}$. Note that $\zeta_k$ can be calculated easily for desired values of $k$. For instance, we have $\zeta_1 = 1, \zeta_2 = \frac{5}{2},$ and $\zeta_3 = \frac{33}{8}$ and for $k=1,2,3$ the asymptotic variance becomes $\frac{s^2}{n}\frac{\hat{\sigma}_j^2(x)}{2s-1}, \frac{5}{8}\frac{s^2}{n}\frac{\hat{\sigma}_j^2(x)}{2s-1},$ and $\frac{11}{24}\frac{s^2}{n}\frac{\hat{\sigma}_j^2(x)}{2s-1}$ respectively. 

\subsection{Adaptive choice for $s$}\label{ssec:adapt}
According to Theorem \ref{thm:knn-mse}, picking $s = \Theta(n^{d/(d+2)})$ would trade-off between bias and variance terms. Also, according to Theorem \ref{thm:knn-normality}, picking $s = n^\beta$ with $d/(d+2) < \beta < 1$ would result in asymptotic normality of the estimator. However, both choices depend on the unknown intrinsic dimension of $\mu$ on the ball $B(x,r)$. Inspired by \cite{kpotufe2011k}, we explain a data-driven way for estimating $s$.

Suppose that $\delta >0$ is given. Let $C_{n,p,\delta} = 2 \log(2p n/\delta)$ and pick $\Delta \geq \Delta_{\calX}$. For any $k \leq s \leq n$, let $H(s)$ be the $U$-statistic estimator for $\epsilon(s)$ defined as $H(s) = \sum_{S \in [n]: |S|=s} \max_{X_i \in H_k(x,S)} \|x-X_i\|_2 / \binom{n}{s} $. Each term in the summation computes the distance of $x$ to its $k$-nearest neighbor on $S$ and $H(s)$ is the average of these numbers over all $\binom{n}{s}$ possible subsets $S$ (see Remark \ref{rem:H} in Appendix \ref{app:missings} regarding to efficient computation of $H(s)$). Define $G_\delta(s) = \Delta \sqrt{C_{n,p,\delta} ps/n}$. Iterate over $s=n,\cdots,k$. Let $s_2$ be the smallest $s$ for which we have $H(s) > 2G_\delta(s)$ and let $s_1 = s_2+1$. Note that $\epsilon_k(s)$ is decreasing in $s$ and $G_\delta(s)$ is increasing in $s$. Therefore, there exists a unique $1 \leq s^* \leq n$ such that $\epsilon_k(s^*) \leq G_\delta(s^*)$ and $\epsilon_k(s^*-1) > G_\delta(s^*-1)$. We have following results.

\begin{proposition}[Adaptive Estimation]\label{thm:est-adapt}
    Let Assumptions of Theorem \ref{thm:knn-mse} hold. Suppose that $s_1$ is the output of the above process. Let $s_* = 9s_1+1$ and suppose that Algorithm \ref{alg:DNN} is executed with $s=s_*$ and $B \geq n/s_*$. Then w.p. at least  $1-2\delta$ we have $\|\htheta - \theta(x)\|_2 = O(G_{\delta}(s^*)) =O \bp{\bp{\frac{n}{p\log(2pn/\delta)}}^{-1/(d+2)}}$. Further, for $\delta = 1/n$ we have $\sqrt{\E\bb{\|\htheta - \theta(x)\|_2^2}} = \tilde{O}\bp{n^{-1/(d+2)}}$.
\end{proposition}

\begin{proposition}[Adaptive Asymptotic Normality]\label{thm:normal-adapt}
Let Assumptions of Theorem \ref{thm:knn-normality} hold. Suppose that $s_1$ is the output of the above process when $\delta = 1/n$ and $s_* = 9s_1+1$. For any $\zeta \in (0, (\log(n)-\log(s_1)-\log\log^2(n))/\log(n)))$ define $s_{\zeta} = s_* n^{\zeta}$. Suppose that Algorithm \ref{alg:DNN} is executed with $s=s_{\zeta}$ and $B \geq (n/s_{\zeta})^{5/4}$, then for any coordinate $j \in [p]$, we have
 $\frac{\htheta_j(x)-\theta_{j}(x)}{\sigma_{n,j}(x)} \rightarrow \normal(0,1)\,.$
\end{proposition}

\bibliographystyle{plainnat}
\bibliography{refs}

\begin{thebibliography}{64}
\providecommand{\natexlab}[1]{#1}
\providecommand{\url}[1]{\texttt{#1}}
\expandafter\ifx\csname urlstyle\endcsname\relax
  \providecommand{\doi}[1]{doi: #1}\else
  \providecommand{\doi}{doi: \begingroup \urlstyle{rm}\Url}\fi

\bibitem[Ai and Chen(2003)]{Ai2003}
Chunrong Ai and Xiaohong Chen.
\newblock Efficient estimation of models with conditional moment restrictions
  containing unknown functions.
\newblock \emph{Econometrica}, 71\penalty0 (6):\penalty0 1795--1843, 2003.

\bibitem[Andoni et~al.(2017)Andoni, Laarhoven, Razenshteyn, and
  Waingarten]{andoni2017optimal}
Alexandr Andoni, Thijs Laarhoven, Ilya Razenshteyn, and Erik Waingarten.
\newblock Optimal hashing-based time-space trade-offs for approximate near
  neighbors.
\newblock In \emph{Proceedings of the Twenty-Eighth Annual ACM-SIAM Symposium
  on Discrete Algorithms}, pages 47--66. Society for Industrial and Applied
  Mathematics, 2017.

\bibitem[Andoni et~al.(2018)Andoni, Indyk, and
  Razenshteyn]{andoni2018approximate}
Alexandr Andoni, Piotr Indyk, and Ilya Razenshteyn.
\newblock Approximate nearest neighbor search in high dimensions.
\newblock \emph{arXiv preprint arXiv:1806.09823}, 2018.

\bibitem[Assouad(1983)]{assouad1983plongements}
Patrice Assouad.
\newblock Plongements lipschitziens dans $\mathbb{R}^ n$.
\newblock \emph{Bull. Soc. Math. France}, 111:\penalty0 429--448, 1983.

\bibitem[Athey and Imbens(2016)]{athey2016recursive}
Susan Athey and Guido Imbens.
\newblock Recursive partitioning for heterogeneous causal effects.
\newblock \emph{Proceedings of the National Academy of Sciences}, 113\penalty0
  (27):\penalty0 7353--7360, 2016.

\bibitem[Athey and Imbens(2017)]{athey2017state}
Susan Athey and Guido~W Imbens.
\newblock The state of applied econometrics: Causality and policy evaluation.
\newblock \emph{Journal of Economic Perspectives}, 31\penalty0 (2):\penalty0
  3--32, 2017.

\bibitem[Athey et~al.(2018)Athey, Imbens, and Wager]{athey2018approximate}
Susan Athey, Guido~W Imbens, and Stefan Wager.
\newblock Approximate residual balancing: debiased inference of average
  treatment effects in high dimensions.
\newblock \emph{Journal of the Royal Statistical Society: Series B (Statistical
  Methodology)}, 80\penalty0 (4):\penalty0 597--623, 2018.

\bibitem[Athey et~al.(2019)Athey, Tibshirani, and Wager]{athey2019generalized}
Susan Athey, Julie Tibshirani, and Stefan Wager.
\newblock Generalized random forests.
\newblock \emph{The Annals of Statistics}, 47\penalty0 (2):\penalty0
  1148--1178, 2019.

\bibitem[Belkin and Niyogi(2003)]{belkin2003laplacian}
Mikhail Belkin and Partha Niyogi.
\newblock Laplacian eigenmaps for dimensionality reduction and data
  representation.
\newblock \emph{Neural computation}, 15\penalty0 (6):\penalty0 1373--1396,
  2003.

\bibitem[Bellman(1961)]{bellman1961adaptive}
RE~Bellman.
\newblock Adaptive control processes princeton.
\newblock \emph{Press, Princeton, NJ}, 1961.

\bibitem[Belloni et~al.(2014{\natexlab{a}})Belloni, Chernozhukov, and
  Hansen]{belloni2014high}
Alexandre Belloni, Victor Chernozhukov, and Christian Hansen.
\newblock High-dimensional methods and inference on structural and treatment
  effects.
\newblock \emph{Journal of Economic Perspectives}, 28\penalty0 (2):\penalty0
  29--50, 2014{\natexlab{a}}.

\bibitem[Belloni et~al.(2014{\natexlab{b}})Belloni, Chernozhukov, and
  Hansen]{belloni2014inference}
Alexandre Belloni, Victor Chernozhukov, and Christian Hansen.
\newblock Inference on treatment effects after selection among high-dimensional
  controls.
\newblock \emph{The Review of Economic Studies}, 81\penalty0 (2):\penalty0
  608--650, 2014{\natexlab{b}}.

\bibitem[Belloni et~al.(2017)Belloni, Chernozhukov, Fern{\'a}ndez-Val, and
  Hansen]{belloni2017program}
Alexandre Belloni, Victor Chernozhukov, Ivan Fern{\'a}ndez-Val, and Christian
  Hansen.
\newblock Program evaluation and causal inference with high-dimensional data.
\newblock \emph{Econometrica}, 85\penalty0 (1):\penalty0 233--298, 2017.

\bibitem[Berrett et~al.(2019)Berrett, Samworth, Yuan,
  et~al.]{berrett2019efficient}
Thomas~B Berrett, Richard~J Samworth, Ming Yuan, et~al.
\newblock Efficient multivariate entropy estimation via $ k $-nearest neighbour
  distances.
\newblock \emph{The Annals of Statistics}, 47\penalty0 (1):\penalty0 288--318,
  2019.

\bibitem[Biau(2012)]{Biau2012}
G{\'e}rard Biau.
\newblock Analysis of a random forests model.
\newblock \emph{J. Mach. Learn. Res.}, 13\penalty0 (1):\penalty0 1063--1095,
  April 2012.
\newblock ISSN 1532-4435.

\bibitem[Biau and Devroye(2015)]{biau2015lectures}
G{\'e}rard Biau and Luc Devroye.
\newblock \emph{Lectures on the nearest neighbor method}.
\newblock Springer, 2015.

\bibitem[Billingsley(2008)]{billingsley2008probability}
Patrick Billingsley.
\newblock \emph{Probability and measure}.
\newblock John Wiley \& Sons, 2008.

\bibitem[Borovkov(2013)]{borovkov2013probability}
A.A. Borovkov.
\newblock \emph{Probability Theory}.
\newblock Springer London, 2013.
\newblock ISBN 9781447152002.

\bibitem[Breiman(2001)]{breiman2001random}
Leo Breiman.
\newblock Random forests.
\newblock \emph{Machine learning}, 45\penalty0 (1):\penalty0 5--32, 2001.

\bibitem[Carmo(1992)]{carmo1992riemannian}
Manfredo Perdig{\~a}o~do Carmo.
\newblock \emph{Riemannian geometry}.
\newblock Birkh{\"a}user, 1992.

\bibitem[Chen and Shah(2018)]{Chen2018}
George~H. Chen and Devavrat Shah.
\newblock Explaining the success of nearest neighbor methods in prediction.
\newblock \emph{Foundations and Trends® in Machine Learning}, 10\penalty0
  (5-6):\penalty0 337--588, 2018.
\newblock ISSN 1935-8237.

\bibitem[Chen and Pouzo(2009)]{Chen2009}
Xiaohong Chen and Demian Pouzo.
\newblock Efficient estimation of semiparametric conditional moment models with
  possibly nonsmooth residuals.
\newblock \emph{Journal of Econometrics}, 152\penalty0 (1):\penalty0 46 -- 60,
  2009.
\newblock ISSN 0304-4076.
\newblock Recent Adavances in Nonparametric and Semiparametric Econometrics: A
  Volume Honouring Peter M. Robinson.

\bibitem[{Chernozhukov} et~al.(2015){Chernozhukov}, {Newey}, and
  {Santos}]{Chernozhukov2015}
Victor {Chernozhukov}, Whitney~K. {Newey}, and Andres {Santos}.
\newblock {Constrained Conditional Moment Restriction Models}.
\newblock \emph{arXiv e-prints}, art. arXiv:1509.06311, September 2015.

\bibitem[{Chernozhukov} et~al.(2016){Chernozhukov}, {Escanciano}, {Ichimura},
  {Newey}, and {Robins}]{Chernozhukov2016locally}
Victor {Chernozhukov}, Juan~Carlos {Escanciano}, Hidehiko {Ichimura},
  Whitney~K. {Newey}, and James~M. {Robins}.
\newblock {Locally Robust Semiparametric Estimation}.
\newblock \emph{arXiv e-prints}, art. arXiv:1608.00033, July 2016.

\bibitem[Chernozhukov et~al.(2018{\natexlab{a}})Chernozhukov, Chetverikov,
  Demirer, Duflo, Hansen, Newey, and Robins]{chernozhukov2018double}
Victor Chernozhukov, Denis Chetverikov, Mert Demirer, Esther Duflo, Christian
  Hansen, Whitney Newey, and James Robins.
\newblock Double/debiased machine learning for treatment and structural
  parameters.
\newblock \emph{The Econometrics Journal}, 21\penalty0 (1):\penalty0 C1--C68,
  2018{\natexlab{a}}.

\bibitem[Chernozhukov et~al.(2018{\natexlab{b}})Chernozhukov, Nekipelov,
  Semenova, and Syrgkanis]{Chernozhukov2018plugin}
Victor Chernozhukov, Denis Nekipelov, Vira Semenova, and Vasilis Syrgkanis.
\newblock Plug-in regularized estimation of high-dimensional parameters in
  nonlinear semiparametric models.
\newblock \emph{arXiv preprint arXiv:1806.04823}, 2018{\natexlab{b}}.

\bibitem[Cutler(1993)]{cutler1993review}
Colleen~D Cutler.
\newblock A review of the theory and estimation of fractal dimension.
\newblock In \emph{Dimension estimation and models}, pages 1--107. World
  Scientific, 1993.

\bibitem[Dasgupta and Freund(2008)]{dasgupta2008random}
Sanjoy Dasgupta and Yoav Freund.
\newblock Random projection trees and low dimensional manifolds.
\newblock In \emph{Proceedings of the fortieth annual ACM symposium on Theory
  of computing}, pages 537--546. ACM, 2008.

\bibitem[Efron(1982)]{efron1982jackknife}
Bradley Efron.
\newblock \emph{The jackknife, the bootstrap, and other resampling plans},
  volume~38.
\newblock Siam, 1982.

\bibitem[Fan et~al.(1998)Fan, Farmen, and Gijbels]{fan1998local}
Jianqing Fan, Mark Farmen, and Irene Gijbels.
\newblock Local maximum likelihood estimation and inference.
\newblock \emph{Journal of the Royal Statistical Society: Series B (Statistical
  Methodology)}, 60\penalty0 (3):\penalty0 591--608, 1998.

\bibitem[Fan et~al.(2018)Fan, Lv, and Wang]{fan2018dnn}
Yingying Fan, Jinchi Lv, and Jingbo Wang.
\newblock Dnn: A two-scale distributional tale of heterogeneous treatment
  effect inference.
\newblock \emph{arXiv preprint arXiv:1808.08469}, 2018.

\bibitem[Friedberg et~al.(2018)Friedberg, Tibshirani, Athey, and
  Wager]{friedberg2018local}
Rina Friedberg, Julie Tibshirani, Susan Athey, and Stefan Wager.
\newblock Local linear forests.
\newblock \emph{arXiv preprint arXiv:1807.11408}, 2018.

\bibitem[Friedman et~al.(2001)Friedman, Hastie, and
  Tibshirani]{friedman2001elements}
Jerome Friedman, Trevor Hastie, and Robert Tibshirani.
\newblock \emph{The elements of statistical learning}, volume~1.
\newblock Springer series in statistics New York, NY, USA:, 2001.

\bibitem[Gy{\"o}rfi et~al.(2006)Gy{\"o}rfi, Kohler, Krzyzak, and
  Walk]{gyorfi2006distribution}
L{\'a}szl{\'o} Gy{\"o}rfi, Michael Kohler, Adam Krzyzak, and Harro Walk.
\newblock \emph{A distribution-free theory of nonparametric regression}.
\newblock Springer Science \& Business Media, 2006.

\bibitem[Hansen(1982)]{hansen1982large}
Lars~Peter Hansen.
\newblock Large sample properties of generalized method of moments estimators.
\newblock \emph{Econometrica: Journal of the Econometric Society}, pages
  1029--1054, 1982.

\bibitem[Hoeffding(1994)]{hoeffding1994probability}
Wassily Hoeffding.
\newblock Probability inequalities for sums of bounded random variables.
\newblock In \emph{The Collected Works of Wassily Hoeffding}, pages 409--426.
  Springer, 1994.

\bibitem[Jiang(2017)]{jiang2017rates}
Heinrich Jiang.
\newblock Rates of uniform consistency for k-nn regression.
\newblock \emph{arXiv preprint arXiv:1707.06261}, 2017.

\bibitem[{Kim} et~al.(2018){Kim}, {Shin}, {Rinaldo}, and {Wasserman}]{Kim2018}
Jisu {Kim}, Jaehyeok {Shin}, Alessandro {Rinaldo}, and Larry {Wasserman}.
\newblock {Uniform Convergence Rate of the Kernel Density Estimator Adaptive to
  Intrinsic Dimension}.
\newblock \emph{arXiv e-prints}, art. arXiv:1810.05935, October 2018.

\bibitem[Kpotufe(2011)]{kpotufe2011k}
Samory Kpotufe.
\newblock $k$-nn regression adapts to local intrinsic dimension.
\newblock In \emph{Advances in Neural Information Processing Systems}, pages
  729--737, 2011.

\bibitem[Kpotufe and Dasgupta(2012)]{kpotufe2012tree}
Samory Kpotufe and Sanjoy Dasgupta.
\newblock A tree-based regressor that adapts to intrinsic dimension.
\newblock \emph{Journal of Computer and System Sciences}, 78\penalty0
  (5):\penalty0 1496--1515, 2012.

\bibitem[Kpotufe and Garg(2013)]{kpotufe2013adaptivity}
Samory Kpotufe and Vikas Garg.
\newblock Adaptivity to local smoothness and dimension in kernel regression.
\newblock In \emph{Advances in neural information processing systems}, pages
  3075--3083, 2013.

\bibitem[Lafferty and Wasserman(2008)]{lafferty2008}
John Lafferty and Larry Wasserman.
\newblock Rodeo: Sparse, greedy nonparametric regression.
\newblock \emph{Ann. Statist.}, 36\penalty0 (1):\penalty0 28--63, 02 2008.

\bibitem[Lewbel(2007)]{lewbel2007local}
Arthur Lewbel.
\newblock A local generalized method of moments estimator.
\newblock \emph{Economics Letters}, 94\penalty0 (1):\penalty0 124--128, 2007.

\bibitem[Mack(1981)]{mack1981local}
Yue-Pok Mack.
\newblock Local properties of $k$-nn regression estimates.
\newblock \emph{SIAM Journal on Algebraic Discrete Methods}, 2\penalty0
  (3):\penalty0 311--323, 1981.

\bibitem[Mackey et~al.(2017)Mackey, Syrgkanis, and Zadik]{mackey2017orthogonal}
Lester Mackey, Vasilis Syrgkanis, and Ilias Zadik.
\newblock Orthogonal machine learning: Power and limitations.
\newblock \emph{arXiv preprint arXiv:1711.00342}, 2017.

\bibitem[Mullainathan and Spiess(2017)]{mullainathan2017machine}
Sendhil Mullainathan and Jann Spiess.
\newblock Machine learning: an applied econometric approach.
\newblock \emph{Journal of Economic Perspectives}, 31\penalty0 (2):\penalty0
  87--106, 2017.

\bibitem[Newey(1993)]{Newey1993}
Whitney~K. Newey.
\newblock 16 efficient estimation of models with conditional moment
  restrictions.
\newblock In \emph{Econometrics}, volume~11 of \emph{Handbook of Statistics},
  pages 419 -- 454. Elsevier, 1993.

\bibitem[Newey(1994)]{newey1994kernel}
Whitney~K Newey.
\newblock Kernel estimation of partial means and a general variance estimator.
\newblock \emph{Econometric Theory}, 10\penalty0 (2):\penalty0 1--21, 1994.

\bibitem[Oprescu et~al.(2018)Oprescu, Syrgkanis, and Wu]{oprescu2018orthogonal}
Miruna Oprescu, Vasilis Syrgkanis, and Zhiwei~Steven Wu.
\newblock Orthogonal random forest for heterogeneous treatment effect
  estimation.
\newblock \emph{arXiv preprint arXiv:1806.03467}, 2018.

\bibitem[Peel et~al.(2010)Peel, Anthoine, and Ralaivola]{peel2010empirical}
Thomas Peel, Sandrine Anthoine, and Liva Ralaivola.
\newblock Empirical bernstein inequalities for u-statistics.
\newblock In J.~D. Lafferty, C.~K.~I. Williams, J.~Shawe-Taylor, R.~S. Zemel,
  and A.~Culotta, editors, \emph{Advances in Neural Information Processing
  Systems 23}, pages 1903--1911. Curran Associates, Inc., 2010.

\bibitem[Reiss and Wolak(2007)]{Reiss2007}
Peter~C Reiss and Frank~A Wolak.
\newblock Structural econometric modeling: Rationales and examples from
  industrial organization.
\newblock \emph{Handbook of econometrics}, 6:\penalty0 4277--4415, 2007.

\bibitem[Robins and Ritov(1997)]{robins1997toward}
James~M Robins and Ya'acov Ritov.
\newblock Toward a curse of dimensionality appropriate (coda) asymptotic theory
  for semi-parametric models.
\newblock \emph{Statistics in medicine}, 16\penalty0 (3):\penalty0 285--319,
  1997.

\bibitem[Roweis and Saul(2000)]{roweis2000nonlinear}
Sam~T Roweis and Lawrence~K Saul.
\newblock Nonlinear dimensionality reduction by locally linear embedding.
\newblock \emph{science}, 290\penalty0 (5500):\penalty0 2323--2326, 2000.

\bibitem[Samworth et~al.(2012)]{samworth2012optimal}
Richard~J Samworth et~al.
\newblock Optimal weighted nearest neighbour classifiers.
\newblock \emph{The Annals of Statistics}, 40\penalty0 (5):\penalty0
  2733--2763, 2012.

\bibitem[Scornet et~al.(2015)Scornet, Biau, and Vert]{scornet2015}
Erwan Scornet, Gérard Biau, and Jean-Philippe Vert.
\newblock Consistency of random forests.
\newblock \emph{Ann. Statist.}, 43\penalty0 (4):\penalty0 1716--1741, 08 2015.

\bibitem[Serfling(2009)]{serfling2009approximation}
Robert~J Serfling.
\newblock \emph{Approximation theorems of mathematical statistics}, volume 162.
\newblock John Wiley \& Sons, 2009.

\bibitem[Staniswalis(1989)]{staniswalis1989kernel}
Joan~G Staniswalis.
\newblock The kernel estimate of a regression function in likelihood-based
  models.
\newblock \emph{Journal of the American Statistical Association}, 84\penalty0
  (405):\penalty0 276--283, 1989.

\bibitem[Stone(1977)]{stone1977consistent}
Charles~J Stone.
\newblock Consistent nonparametric regression.
\newblock \emph{The Annals of Statistics}, pages 595--620, 1977.

\bibitem[Stone(1982)]{Stone1982}
Charles~J. Stone.
\newblock Optimal global rates of convergence for nonparametric regression.
\newblock \emph{Ann. Statist.}, 10\penalty0 (4):\penalty0 1040--1053, 12 1982.

\bibitem[Tenenbaum et~al.(2000)Tenenbaum, De~Silva, and
  Langford]{tenenbaum2000global}
Joshua~B Tenenbaum, Vin De~Silva, and John~C Langford.
\newblock A global geometric framework for nonlinear dimensionality reduction.
\newblock \emph{Science}, 290\penalty0 (5500):\penalty0 2319--2323, 2000.

\bibitem[Tibshirani and Hastie(1987)]{tibshirani1987local}
Robert Tibshirani and Trevor Hastie.
\newblock Local likelihood estimation.
\newblock \emph{Journal of the American Statistical Association}, 82\penalty0
  (398):\penalty0 559--567, 1987.

\bibitem[Verma et~al.(2009)Verma, Kpotufe, and Dasgupta]{verma2009spatial}
Nakul Verma, Samory Kpotufe, and Sanjoy Dasgupta.
\newblock Which spatial partition trees are adaptive to intrinsic dimension?
\newblock In \emph{Proceedings of the twenty-fifth conference on uncertainty in
  artificial intelligence}, pages 565--574. AUAI Press, 2009.

\bibitem[Wager and Athey(2018)]{wager2017estimation}
Stefan Wager and Susan Athey.
\newblock Estimation and inference of heterogeneous treatment effects using
  random forests.
\newblock \emph{Journal of the American Statistical Association}, 113\penalty0
  (523):\penalty0 1228--1242, 2018.

\bibitem[Xue and Kpotufe(2018)]{Xue2018}
Lirong Xue and Samory Kpotufe.
\newblock Achieving the time of $1$-nn, but the accuracy of $k$-nn.
\newblock In Amos Storkey and Fernando Perez-Cruz, editors, \emph{Proceedings
  of the Twenty-First International Conference on Artificial Intelligence and
  Statistics}, volume~84 of \emph{Proceedings of Machine Learning Research},
  pages 1628--1636, 09--11 Apr 2018.

\end{thebibliography}

\newpage
\begin{appendix}
\section{Related work}\label{sec:related-work}
There exists a vast literature on average treatment effect estimation in high-dimensional settings. The key challenge in such settings is the problem of overfitting which is usually handled by adding regularization terms. However, this leads to a shrinked estimate for the average treatment effect and therefore not desirable. The literature has taken various approaches to solve this issue. For instance, \cite{belloni2014high, belloni2014inference} used a two-step method for estimating average treatment effect where in the first step feature-selection is accomplished via a lasso and then treatment effect is estimated using selected features. \cite{athey2018approximate} studied approximate residual balancing where a combination of weight balancing and regression adjustment is used for removing undesired bias and for achieving a double robust estimator. \cite{chernozhukov2018double, chernozhukov2018double} considered a more general semi-parametric framework and studied debiased/double machine learning methods via first order Neyman orthogonality condition. \cite{mackey2017orthogonal} extended this result to higher order moments. Please refer to \cite{athey2017state,mullainathan2017machine,belloni2017program} for a review on this literature.

However, in many applications, researchers are interested in estimating heterogeneous treatment effect on various sub-populations. One effective solution is to use one of the methods described in previous paragraph to estimate problem parameters and then project such estimations onto the sub-population of interest. However, these approaches usually perform poorly when there is a model mis-specification, i.e., when the true underlying model does not belong to the parametric search space. Consequently, researchers have studied non-parametric estimators such as $k$-NN estimators, kernel estimators, and random forests. While these non-parametric estimators are very robust to model mis-specification and work well under mild assumptions on the function of interest, they suffer from the curse of dimensionality (see e.g., \cite{bellman1961adaptive, robins1997toward, friedman2001elements}). Therefore, for applying these estimators in high-dimensional settings it is necessary to design and study non-parametric estimators that are able to overcome curse of dimensionality when possible.

The seminal work of \cite{wager2017estimation} utilized random forests originally introduced by \cite{breiman2001random} and adapted them nicely for estimating heterogeneous treatment effect. In particular, the authors demonstrated how the recursive partitioning idea, explained in \cite{athey2016recursive} for estimating heterogeneity in causal settings, can be further analyzed to establish asymptotic properties of such estimators. The main premise of random forests is that they are able to adaptively select nearest neighbors and that is very desirable in high-dimensional settings where discarding uninformative features is necessary for combating the curse of dimensionality. In a follow-up work, they extended these results and introduced Generalized Random Forests for more general setting of solving generalized method of moment (GMM) equations \cite{athey2019generalized}. There has been some interesting developments of such ideas to other settings. \cite{fan2018dnn} introduced Distributional Nearest Neighbor (DNN) where they used $1$-NN estimators together with sub-sampling and explained that by precisely combining two of these estimators for different sub-sampling sizes, the first order bias term can be efficiently removed. \cite{friedberg2018local} paired this idea with a local linear regression adjustment and introduced Local Linear Forests in order to improve forest estimations for smooth functions. \cite{oprescu2018orthogonal} incorporated the double machine learning methods of \cite{chernozhukov2018double} into GMM framework of \cite{athey2019generalized} and studied Orthogonal Random Forests in partially linear regression models with high-dimensional controls. Although forest kernels studied in \cite{wager2017estimation} and \cite{athey2019generalized} seem to work well in high-dimensional applications, to the best of our knowledge, there still does not exists a theoretical result supporting it. In fact, all existing theoretical results suffer from the curse of dimensionality as they depend on the dimension of problem $D$.

The literature on machine learning and non-parametric statistics has recently studied how these worst-case performances can be avoided when the intrinsic dimension of problem is smaller than $D$. Please refer to \cite{cutler1993review} for different notions of intrinsic dimension in metric spaces. \cite{dasgupta2008random} studied random projection trees and showed that the structure of these trees do not depend on the actual dimension $D$, but rather on the intrinsic dimension $d$. They used the notion of Assouad Dimension, introduced by \cite{assouad1983plongements}, and proved that using random directions for splitting, the number of levels required for halving the diameter of a leaf scales as $O(d \log d)$. The follow-up work \cite{verma2009spatial} generalized these results for some other notions of dimension. \cite{kpotufe2012tree} extended this idea to the regression setting and proved integrated risk bounds for random projection trees that were only dependent on intrinsic dimension. \cite{kpotufe2011k, kpotufe2013adaptivity} studied this in the context of $k$-NN and kernel estimations and established uniform point-wise risk bounds only depending on the local intrinsic dimension.

Our work is deeply rooted in the literature on intrinsic dimension explained above, literature on $k$-NN estimators (see e.g, \cite{mack1981local, samworth2012optimal, gyorfi2006distribution, biau2015lectures, berrett2019efficient, fan2018dnn}), and generalized method of moments (see e.g., \citep{tibshirani1987local, staniswalis1989kernel, fan1998local, hansen1982large, stone1977consistent, lewbel2007local, mackey2017orthogonal}). We adapt the framework of \cite{athey2019generalized} and \cite{oprescu2018orthogonal} and solve a generalized moment problem using a DNN estimator, originally introduced and studied by \cite{fan2018dnn}. We establish consistency and inference properties of this estimator and prove that these properties only depend on the local intrinsic dimension of problem. In particular, we prove that the finite sample estimation error of order $n^{-1/(d+2)}$ together with $n^{1/(d+2)}$-asymptotically normality result of DNN estimator for solving the generalized moment problem regardless of how big the actual dimension $D$ is. 

Our result differs from existing literature on intrinsic dimension (e.g., \cite{kpotufe2011k, kpotufe2013adaptivity}) since in addition to estimation guarantees for the regression setting, we also allow valid inference in solving conditional moment equations. Our asymptotic normality result is different from existing results for $k$-NN (see e.g., \cite{mack1981local}), generalized method of moments (see e.g., \cite{lewbel2007local}). This paper complements the work of \cite{fan2018dnn} and extends it to the generalized method of moment setting. Furthermore, we relax the common assumption on the existence of density for covariates and prove that DNN estimators are adaptive to intrinsic dimension.

We also provide the exact expression for the asymptotic variance of DNN estimator built using a $k$-NN kernel, which enables plug-in construction of confidence intervals, rather than the bootstrap method of \citep{efron1982jackknife} which was used by \citep{wager2017estimation, athey2019generalized, fan2018dnn}. While establishing consistency and asymptotic normality of our estimator, we also provide more general bounds on kernel shrinkage rate and also incrementality which can be useful for establishing asymptotic properties in other applications. One such application is given in high-dimensional settings where the exact nearest neighbor search is computationally expensive and Approximate Nearest Neighbor (ANN) search is often replaced in order to reduce this cost. Our flexible result allows us to use the state-of-the-art ANN algorithms (see e.g., \cite{andoni2017optimal,andoni2018approximate}) while maintaining consistency and asymptotic normality.

\section{Examples of spaces with small intrinsic dimension}\label{ssec:examples}
In this section we provide examples of metric spaces that have small local intrinsic dimension. Our first example covers the setting where the distribution of data lies on a low-dimensional manifold (see e.g., \cite{roweis2000nonlinear,tenenbaum2000global,belkin2003laplacian}). For instance, this happens for image inputs. Even though images are often high-dimensional (e.g., $4096$ in the case of $64$ by $64$ images), all these images belong intrinsically to a $3$-dimensional manifold.
\begin{example}[Low dimensional manifold (adapted from \cite{kpotufe2011k})]\label{ex:manifold}
Consider a $d$-dimensional submanifold $\calX \subset \reals^D$ and let $\mu$ have lower and upper bounded density on $\calX$. The local intrinsic dimension of $\mu$ on $B(x,r)$ is $d$, provided that $r$ is chosen small enough and some conditions on curvature hold. In fact, Bishop-Gromov theorem (see e.g., \cite{carmo1992riemannian}) implies that under such conditions, the volume of ball $B(x,r) \cap \calX$ is $\Theta(r^{d})$. This together with the lower and upper bound on the density implies that $\mu(B(x,r) \cap \calX)/\mu(B(x, \theta r) \cap \calX) = \Theta(\theta^d)$, i.e. $\mu$ is $(C,d)$-homogeneous on $B(x,r)$ for some $C>0$.
\end{example}

Another example which happens in many applications, is sparse data. For example, in the bag of words representation of text documents, we usually have a vocabulary consisting of $D$ words. Although $D$ is usually large, each text document contains only a small number of these words. In this application, we expect our data (and measure) to have smaller intrinsic dimension. Before stating this example, let us discuss a more general example about mixture distributions.
\begin{example}[Mixture distributions (adapted from \cite{kpotufe2011k})]\label{ex:mixture}
Consider any mixture distribution $\mu=\sum_i \pi_i \mu_i$, with each $\mu_i$ defined on $\calX$ with potentially different supports. Consider a point $x$ and note that if $x \not\in \supp(\mu_i)$, then there exists a ball $B(x,r_i)$ such that $\mu_i(B(x,r_i))=0$. This is true since the support of any probability measure is always closed, meaning that its complement is an open set. Now suppose that $r$ is chosen small enough such that for any $i$ satisfying $x \in \supp(\mu_i)$, $\mu_i$ is $(C_i,d_i)$-homogeneous on $B(x,r)$, while for any $i$ satisfying $x \not\in \supp(\mu_i)$ we have $\mu_i(B(x,r))=0$. Then,
 \begin{align*}
    \mu(B(x,r))
    =~&\sum_i \pi_i \mu_i(B(x,r)) = \sum_{i: \mu_i(B(x,r))=0} \pi_i \mu_i(B(x,r)) + \sum_{i: \mu_i(B(x,r))>0} \pi_i \mu_i(B(x,r)) \\
    \leq~& C \theta^{-d} \sum_{i: \mu_i(B(x,r))>0} \pi_i \mu_i(B(x,\theta r))
    =~ C \theta^{-d}  \sum_{i} \pi_i \mu_i(B(x,\theta r) = C \theta^{-d} \mu(B(x,\theta r))\,,
\end{align*}
where $C = \max_{i: \mu_i(B(x,r))>0} C_i$ and $d = \max_{i: \mu_i(B(x,r))>0} d_i$ and we used the fact that if $\mu_i(B(x,r))=0$ then $\mu_i(B(x,\theta r)) = 0$. Therefore, $\mu$ is $(C,d)$-homogeneous on $B(x,r)$.
\end{example}

This result applies to the case of $d$-sparse data and is explained in the following example.

\begin{example}[$d$-sparse data]\label{ex:sparse}
Suppose that $\calX \subset \reals^D$ is defined as
$$\calX = \bc{(x_1,x_2,\ldots,x_D) \in \reals^D: \sum_{i=1}^D \,1\bc{x_i \neq 0} \leq d}.$$
Let $\mu$ be a probability measure on $\calX$. In this case, we can write $\calX$ as the union of $k=\binom{D}{d}, d$-dimensonal hyperplanes in $\reals^D$. In fact,
\begin{equation*}
    \calX = \cup_{1 \leq i_1 < i_2 < \cdots i_d \leq D} \bc{(x_1,x_2,\cdots,x_D) \in \reals^D: x_j = 0,~ j \not \in \bc{i_1,i_2,\ldots,i_d}}\,.
\end{equation*}
Letting $\mu_{i_1,i_2,\ldots,i_d}$ be the probability measure restricted to the hyperplane defined by $x_j = 0, j\not\in \bc{i_1,i_2,\ldots,i_d}$, we can express $\mu = \sum_{1 \leq i_1 < i_2 < \cdots i_d \leq D} \pi_{i_1,i_2,\ldots,i_d} \mu_{i_1,i_2,\ldots,i_d}$. Therefore, the result of Example \ref{ex:mixture} implies that for any $x \in \calX$, for $r$ that is small enough $\mu$ is $(C,d)$-homogeneous on $B(x,r)$.
\end{example}

Our final example is about the product measure. This allows us to prove that any concatenation of spaces with small intrinsic dimension has a small intrinsic dimension as well.

\begin{example}[Concatenation under the product measure]\label{ex:concat}
Suppose that $\mu_i$ is a probability measure on $\calX_i \subset \reals^{D_1}$. Define $\calX = \bc{(z_1,z_2) \mid z_1 \in \calX_1, z_2 \in \calX_2}$ and let $\mu = \mu_1 \times \mu_2$ be the product measure on $\calX$, i.e., $\mu(E_1 \times E_2) = \mu_1(E_1) \times \mu_2(E_2)$ for $E_i$ that is $\mu_i$-measurable, $i=1,2$. Suppose that $\mu_i$ is $(C_i,d_i)$-homogeneous on $B(x_i,r_i)$ and let $x=(x_1,x_2)$. Then, $\mu$ is $(C,d)$-homogeneous on $B(x,r)$, where $d=d_1+d_2, r =\min \bc{r_1,r_2}$ and $C=(C_1\,C_2\,r^{-(d_1+d_2)}\,2^{(d_1+d_2)/2})/(r_1^{-d_1}\,r_2^{-d_2})$. To establish this, let $r=\min \bc{r_1,r_2}$ and note that for any $\theta \in (0,1)$ we have
\begin{align*}
    \mu\bp{B(x,r)} 
    \leq~& \mu \bp{B(x_1,r) \times B(x_2,r)} = \mu_1 \bp{B(x_1,r)} \times \mu_2 \bp{B(x_2,r)} \\
    &\leq \mu_1 \bp{B(x_1,r_1)} \times \mu_2 \bp{B(x_2,r_2)} \\
    \leq~& \bb{C_1 \bp{\frac{r\theta}{r_1\sqrt{2}}}^{-d_1} \mu_1 \bp{B \bp{x_1, \frac{r\theta}{\sqrt{2}}}}} \times \bb{C_2 \bp{\frac{r\theta}{r_2\sqrt{2}}}^{-d_2} \mu_2 \bp{B\bp{x_2, \frac{r\theta}{\sqrt{2}}}}} \\
    =~& \frac{C_1\,C_2\,r^{-(d_1+d_2)}}{r_1^{-d_1}\,r_2^{-d_2}\,\sqrt{2}^{-(d_1+d_2)}} \theta^{-d_1-d_2} \mu \bp{B(x_1, r\theta/\sqrt{2}) \times B(x_2, r\theta/\sqrt{2})} \\
    \leq~& \frac{C_1\,C_2\,r^{-(d_1+d_2)\,2^{(d_1+d_2)/2}}}{r_1^{-d_1}\,r_2^{-d_2}} \theta^{-(d_1+d_2)} \mu \bp{B(x,r \theta)}\,,
\end{align*}
where we used two simple inequalities that $\|(z_1,z_2)-(x_1,x_2)\|_2 \leq r$ implies $\|z_i-x_i\|_2 \leq r,~i=1,2,$ and further $\|z_i-x_i\|_2 \leq r/\sqrt{2},~i=1,2,$ implies $\|(z_1,z_2)-(x_1,x_2)\|_2 \leq r$. 
\end{example}

\section{Simulation Setting}\label{app:sim}
Here we explain the settings for simulations shown in Figures \ref{fig:distribution} and \ref{fig:coverage1d}.

\subsection{Single test point}
The data for single test point simulation, shown in Figure \ref{fig:distribution}, has been generated as follows. Here $p=1$, $D=20$ and $d=2$. All the points are generated using $X_i = A X_i^{\text{low}}$, where $A \in \reals^{D \times d}$ and entries of $A$ are independently sampled from $U[-1,1]$. Components of each $X_i^{\text{low}}$ are also generated independently from $U[-1,1]$. We generate a fix test point $x_{\text{test}} = A x_{\text{test}}^{\text{low}}$ and keep the matrix $A$ throughout all Monte-Carlo iterations fixed. In each Monte-Carlo iteration, we generate $n=20000$ training points as mentioned before. The values of $Y_i$ are generated according to $Y_i = f(X_i) + \varepsilon_{i}$, where $f(X) = \frac{1}{1+\exp(-3X[0])}$, and $\varepsilon_i \sim \normal(0,\sigma_e^2)$ with $\sigma_e = 1$. We are interested in estimate and inference for $f(x_{\text{test}})$ which is equivalent to solving for $\E[\psi(Z;\theta(x)) \mid X=x] = 0$ with $\psi(Z;\theta(x)) = Y - \theta(x)$ at $x = x_{\text{test}}$. We run DNN (Algorithm \ref{alg:DNN}) for $k=1,2$ and $5$ with parameter $s = s_{\zeta}$ chosen using Proposition \ref{thm:normal-adapt} with $\zeta=0.1$ over $1000$ Monte-Carlo iterations and report the histogram and quantile-quantile plot of estimates compared to theoretical asymptotic normal distribution of estimates stemming from our characterization. In our simulations, we considered the complete $U$-statistic case, i.e., $B = \binom{n}{s}$.

\subsection{Multiple test point}
The data for the multiple test point simulation, shown in Figure \ref{fig:coverage1d}, has been generated very similarly to the single test point setting. The only difference is that instead of generating a single test point we generate $100$ test points. These test points together with matrix $A$ are kept fixed throughout all $1000$ Monte-Carlo iterations. We compare the performance of DNN (Algorithm \ref{alg:DNN}) with parameter $s = s_{\zeta}$ chosen using Proposition \ref{thm:normal-adapt} with $\zeta = 0.1$ with two benchmarks that set $s_d=n^{1.05 d/(d+2)}$ and $s_D=n^{1.05 D/(D+2)}$. This process has been repeated for $k=1,2$ and $5$ and the coverage over a single run for all test points, the empirical coverage over $1000$ runs, and chosen $s_{\zeta}$ versus $s_d$ are depicted. 

\section{Nuisance parameters and heterogeneous treatment effects} \label{sec:nuisance}
Using the techniques of \cite{oprescu2018orthogonal}, our work also easily extends to the case where the moments depend on, potentially infinite dimensional, nuisance components $h_0$, that also need to be estimated, i.e.,
\begin{equation}
\textstyle{\theta(x) \text{~solves:~} m(x; \theta, h_0) = \E[\psi(Z; \theta, h_0)\mid x] = 0.}
\end{equation}
If the moment $m$ is orthogonal with respect to $h$ and assuming that $h_0$ can be estimated on a separate sample with a conditional MSE rate of
\begin{equation}
\textstyle{\E[(\hat{h}(z) - h_0(z))^2 | X=x] = o_p(\epsilon(s) + \sqrt{s/n})\,,}
\end{equation}
then using the techniques of \cite{oprescu2018orthogonal}, we can argue that both our finite sample estimation rate and our asymptotic normality rate, remain unchanged, as the estimation error only impacts lower order terms. This extension allows us to capture settings like heterogeneous treatment effects, where the treatment model also needs to be estimated when using the orthogonal moment as
\begin{equation}
\textstyle{\psi(z; \theta, h_0) = (y - q_0(x, w) - \theta (t - p_0(x, w)))\, (t-p_0(x,w))\,,}
\end{equation}
where $y$ is the outcome of interest, $t$ is a treatment, $x, w$ are confounding variables, $q_0(x, w) = \E[Y | X=x, W=w]$ and $p_0(x,w)=E[T | X=x, W=w]$. The latter two nuisance functions can be estimated via separate non-parametric regressions. In particular, if we assume that these functions are sparse linear in $w$, i.e.:
\begin{align}
q_0(x,w) =& \ldot{\beta(x)}{w}\,,&p_0(x,w) = \ldot{\gamma(x)}{w}\,.
\end{align}
Then we can achieve a conditional mean-squared-error rate of the required order by using the kernel lasso estimator of \cite{oprescu2018orthogonal}, where the kernel is the sub-sampled $k$-NN kernel, assuming the sparsity does not grow fast with $n$. 

\section{Proof of Theorem \ref{thm:mse_rate}}\label{app:mse_rate}
\begin{lemma}\label{lem:mse_lem1}
For any $\theta\in \Theta$:
\begin{equation}
    \|\theta - \theta(x)\|_2 \leq \frac{2}{\lambda} \| m(x; \theta)\|_2\,.
\end{equation}
\end{lemma}
\begin{proof}
By strong convexity of the loss $L(x; \theta)$ and the fact that $m(x;\theta(x))=0$, we have:
\begin{equation*}
    L(x; \theta) - L(x; \theta(x)) \geq \ldot{m(x;\theta(x))}{\theta - \theta(x)} + \frac{\lambda}{2} \cdot \|\theta- \theta(x)\|_2^2 = \frac{\lambda}{2} \cdot \|\theta- \theta(x)\|_2^2 \,.
\end{equation*}
By convexity of the loss $L(x; \theta)$ we have:
\begin{equation*}
    L(x; \theta(x)) - L(x; \theta) \geq \ldot{m(x;\theta)}{\theta(x) - \theta}\,.
\end{equation*}
Combining the latter two inequalities we get:
\begin{equation*}
    \frac{\lambda}{2} \cdot \|\theta- \theta(x)\|_2^2 \leq \ldot{m(x;\theta)}{\theta - \theta(x)} \leq \|m(x;\theta)\|_2 \cdot \|\theta - \theta(x)\|_2\,.
\end{equation*}
Note that if $\|\theta-\theta(x)\|_2 = 0$, then the result is obvious. Otherwise, dividing over by $\|\theta - \theta(x)\|_2$ completes the proof of the lemma.
\end{proof}

\begin{lemma}\label{lem:mse_lem2}
Let $\Lambda(x; \theta) = m(x;\theta) - \Psi(x;\theta)$. Then the estimate $\hat{\theta}$ satisfies:
\begin{equation}
    \|m(x;\hat{\theta})\|_2 \leq \sup_{\theta\in\Theta} \|\Lambda(x; \theta)\|_2\,.
\end{equation}
\end{lemma}
\begin{proof}
Observe that $\hat{\theta}$, by definition, satisfies $\Psi(x;\hat{\theta})=0$. Thus:
\begin{equation*}
    \|m(x;\hat{\theta})\|_2 = \|m(x;\hat{\theta}) - \Psi(x; \hat{\theta})\|_2 = \|\Lambda(x;\hat{\theta})\|_2 \leq \sup_{\theta\in \Theta} \|\Lambda(x; \theta)\|_2\,.
\end{equation*}
\end{proof}

\begin{lemma}\label{lem:lambda_bound_mse}\label{lem:mse_lem3}
Suppose that the kernel is built with sub-sampling at rate $s$, in an honest manner (Assumption \ref{ass:honest}) and with at least $B\geq n/s$ sub-samples. If the base kernel satisfies kernel shrinkage in expectation, with rate $\epsilon(s)$, then w.p. $1-\delta$:
\begin{equation}
    \sup_{\theta\in\Theta} \|\Lambda(x; \theta)\|_2 \leq L_m \epsilon(s) + O\left(\psi_{\max}\sqrt{\frac{p\, s}{n} \left(\log\log(n /s) + \log(p/\delta)\right)}\right)\,.
\end{equation}
\end{lemma}
\begin{proof}
Define
\begin{equation*}
\mu_0(x;\theta) = \E \bb{\Psi_0(x;\theta)}\,,
\end{equation*}
where we remind that $\Psi_0$ denotes the complete $U$-statistic:
\begin{equation*}
    \Psi_0(x; \theta) = \binom{n}{s}^{-1} \sum_{S_b \subset [n]: |S_b|=s} \E_{\omega_b}\bb{\sum_{i\in S_b} \alpha_{S_b,\omega_b}(X_i) \psi(Z_i; \theta)} \,.
\end{equation*}
Here the expectation is taken with respect to the random draws of $n$ samples.
Then, the following result which is due to \cite{oprescu2018orthogonal} holds.
\begin{lemma}[Adapted from \cite{oprescu2018orthogonal}]\label{lem:mu-eq}
For any $\theta$ and target $x$
\begin{equation*}
\mu_0(x;\theta) = \binom{n}{s}^{-1} \sum_{S_b \subset [n]: |S_b|=s} \E\bb{\sum_{i\in S_b} \alpha_{S_b,\omega_b}(X_i) m(X_i; \theta)}\,.
\end{equation*}
\end{lemma}
In other words, Lemma \ref{lem:mu-eq} states that, in the expression for $\mu_0$ we can simply replace $\psi(Z_i; \theta)$ with its expectation which is $m(X_i; \theta)$. 
We can then express $\Lambda(x;\theta)$ as sum of kernel error, sampling error, and sub-sampling error, by adding and subtracting appropriate terms, as follows:
\begin{align*}
\Lambda(x; \theta) =~& m(x; \theta) - \Psi(x;\theta)\\
=~& \underbrace{m(x; \theta) - \mu_0(x;\theta)}_{\Gamma(x,\theta)=\text{Kernel error}} + \underbrace{\mu_0(x; \theta) - \Psi_0(x; \theta)}_{\Delta(x,\theta)=\text{Sampling error}} + \underbrace{\Psi_0(x;\theta) - \Psi(x;\theta)}_{\Upsilon(x,\theta) = \text{Sub-sampling error}}
\end{align*}
The parameters should be chosen to trade-off these error terms nicely. We will now bound each of these three terms separately and then combine them to get the final bound.

\paragraph{Bounding the Kernel error.} By Lipschitzness of $m$ with respect to $x$ and triangle inequality, we have: 
\begin{align*}
    \|\Gamma(x;\theta)\|_2 \leq~& \binom{n}{s}^{-1} \sum_{S_b \subset [n]: |S_b|=s} \E\bb{\sum_{i\in S_b} \alpha_{S_b,\omega_b}(X_i) \|m(x;\theta) - m(X_i; \theta)\|}\\
    \leq~& L_m\binom{n}{s}^{-1} \sum_{S_b \subset [n]: |S_b|=s} \E\bb{\sum_{i\in S_b} \alpha_{S_b,\omega_b}(X_i) \|x-X_i\|}\\
    \leq~& L_m\binom{n}{s}^{-1} \sum_{S_b \subset [n]: |S_b|=s} \E\bb{\sup\{\|x-X_i\|: \alpha_{S_b,\omega_b}(X_i)>0\}}\\
    \leq~& L_m\, \epsilon(s) \,,
\end{align*}
where the second to last inequality follows from the fact that $\sum_{i} |\alpha_{S_b}(X_i) |=1$.

\paragraph{Bounding the Sampling error.} For bounding the sampling error we rely on Lemma \ref{lem:stoch_eq_general} and in particular Corollary~\ref{cor:bracketing}. Observe that for each $j\in \{1,\ldots, p\}$, $\Psi_{0j}(x;\theta)$ is a complete $U$-statistic for each $\theta$. Thus the sampling error defines a $U$-process over the class of symmetric functions $\text{conv}(\cF_j)=\{f_j(\cdot;\theta): \theta \in \Theta\}$, with $f_j(Z_1,\ldots, Z_{s};\theta) = \E_{\omega}\bb{\sum_{i=1}^{s} \alpha_{Z_{1:s},\omega}(X_i) \psi_j(Z_i; \theta)}$. Observe that since $f_j\in \text{conv}(\cF_j)$ is a convex combination of functions in $\cF_j=\{\psi_j(\cdot; \theta): \theta \in \Theta\}$, the bracketing number of functions in $\text{conv}(\cF_j)$ is upper bounded by the bracketing number of $\cF_j$, which by our assumption, satisfies $\log(N_{[]}({\cal F}_j, \epsilon, L_2))= O(1 /\epsilon)$. Moreover, by our assumptions on the upper bound $\psi_{\max}$ of $\psi_j(z;\theta)$, we have that $\sup_{f_j\in \text{conv}(\cF_j)} \|f_j\|_{2}, \sup_{f_j\in \text{conv}(\cF_j)} \|f_j\|_{\infty} \leq \psi_{\max}$. Thus all conditions of Corollary~\ref{cor:bracketing} are satisfied, with $\eta=G=\psi_{\max}$ and we get that w.p. $1-\delta/2p$:
\begin{equation}
    \sup_{\theta\in\Theta} |\Delta_j(x,\theta)| = O\left(\psi_{\max}\sqrt{\frac{s}{n} \left(\log\log(n/s) + \log(2p/\delta)\right)}\right)\,.
\end{equation}
By a union bound over $j$, we get that w.p. $1-\delta/2$:
\begin{equation}
\sup_{\theta\in\Theta} \|\Delta_j(x,\theta)\|_2 \leq \sqrt{p} \max_{j\in [p]} \sup_{\theta\in\Theta} |\Delta_j(x,\theta)| = O\left(\psi_{\max}\sqrt{\frac{p\, s}{n} \left(\log\log(n /s) + \log(p/\delta)\right)}\right).
\end{equation}

\paragraph{Bounding the Sub-sampling error.} Sub-sampling error decays as $B$ is increased. Note that for a fixed set of samples $\bc{Z_1,Z_2,\ldots,Z_n}$, for a set $S_b$ randomly chosen among all $\binom{n}{s}$ subsets of size $s$ from the $n$ samples, we have:
\begin{equation*}
    \E_{S_b,\omega_b}\bb{\sum_{i\in S_b} \alpha_{S_b,\omega_b}(X_i) \psi(Z_i;\theta)} = \Psi_0(x;\theta)\,.
\end{equation*}
Therefore, $\Psi(x;\theta)$ can be thought as the sum of $B$ i.i.d. random variables each with expectation equal to $\Psi_0(x;\theta)$, where expectation is taken over $B$ draws of sub-samples, each with size $s$. Thus one can invoke standard results on empirical processes for function classes as a function of the bracketing entropy. For simplicity, we can simply invoke Corollary~\ref{cor:bracketing} in the appendix for the case of a trivial $U$-process, with $s=1$ and $n=B$ to get that w.p. $1-\delta/2$:
\begin{equation*}
    \sup_{\theta\in \Theta} |\Upsilon(x;\theta)| = O\left(\psi_{\max}\sqrt{\frac{\log\log(B) + \log(2/\delta)}{B}}\right) 
\end{equation*}
Thus for $B\geq n/s$, the sub-sampling error is of lower order than the sampling error and can be asymptotically ignored.
Putting together the upper bounds on sampling, sub-sampling and kernel error finishes the proof of the Lemma.
\end{proof}

The probabilistic statement of the proof follows by combining the inequalities in the above three lemmas. The in expectation statement follows by simply integrating the exponential tail bound of the probabilistic statement.

\section{Proof of Theorem \ref{thm:normality}}\label{app:normality}

We will show asymptotic normality of $\hat{\alpha}=\ldot{\beta}{\hat{\theta}}$ for some arbitrary direction $\beta \in \reals^p$, with $\|\beta\|_2\leq R$. Consider the complete multi-dimensional $U$-statistic:
\begin{equation}
    \Psi_0(x; \theta) = \binom{n}{s}^{-1} \sum_{S_b \subset [n]: |S_b|=s} \E_{\omega_b}\bb{\sum_{i\in S_b}\alpha_{S_b,\omega_b}(X_i) \psi(Z_i; \theta)} \,.
\end{equation}
Let 
\begin{equation}
    \Delta(x;\theta) = \Psi_0(x;\theta)-\mu_0(x; \theta)
\end{equation}
where $\mu_0(x;\theta) = \E\bb{\Psi_0(x;\theta)}$ (as in the proof of Theorem~\ref{thm:mse_rate}) and
\begin{equation}
    \tilde{\theta} = \theta(x) - M_0^{-1} \Delta(x;\theta(x))
\end{equation}
Finally, let
\begin{equation}
\tilde{\alpha} \triangleq \ldot{\beta}{\tilde{\theta}} = \ldot{\beta}{\theta(x)} - \ldot{\beta}{M_0^{-1}\,\Delta(x;\theta(x))} 
\end{equation}
For shorthand notation let $\alpha_0=\ldot{\beta}{\theta(x)}$, $\psi_\beta(Z; \theta) = \ldot{\beta}{M_0^{-1} (\psi(Z;\theta)-m(X;\theta))}$ and 
\begin{align*}
\Psi_{0,\beta}(x;\theta) =~& \ldot{\beta}{M_0^{-1}\, \Delta(x;\theta(x))}\\
=~& \binom{n}{s}^{-1} \sum_{S_b \subset [n]: |S_b|=s} \E_{\omega_b}\bb{\sum_{i\in S_b}\alpha_{S_b,\omega_b}(X_i) \psi_{\beta}(Z_i; \theta)}
\end{align*}
be a single dimensional complete $U$-statistic. Thus we can re-write:
\begin{align*}
    \tilde{\alpha} = \alpha_0 - \Psi_{0,\beta}(x;\theta(x))
\end{align*}
We then have the following lemma which its proof is provided in Appendix \ref{app:hajek}:
\begin{lemma}\label{lem:normality_of_U}
Under the conditions of Theorem \ref{thm:normality}:
\begin{equation*}
    \frac{\Psi_{0,\beta}(x;\theta(x))}{\sigma_n(x)} \rightarrow \normal(0,1)\,,
\end{equation*}
for $\sigma_n^2(x) = \frac{s^2}{n}\Var\bb{\E\bb{\sum_{i=1}^sK(x, X_i, \{X_j\}_{j=1}^s) \psi_{\beta}(Z_i; \theta)\mid X_1}} = \Omega(\frac{s^2}{n} \eta(s))$.
\end{lemma}
Invoking Lemma~\ref{lem:normality_of_U} and using our assumptions on the kernel, we conclude that:
\begin{equation}
    \frac{\tilde{\alpha} - \alpha_0(x)}{\sigma_n(x)} \rightarrow \normal(0,1).
\end{equation}
For some sequence $\sigma_n^2$ which decays at least as slow as $s^2 \eta(s)/n$.
Hence, since 
\[
\frac{\hat{\alpha}-\alpha_0}{\sigma_n(x)} = \frac{\tilde{\alpha}-\theta(x)}{\sigma_n(x)} + \frac{\hat{\alpha}-\tilde{\alpha}}{\sigma_n(x)},
\]
if we show that $\frac{\hat{\alpha}-\tilde{\alpha}}{\sigma_n(x)}\rightarrow_p 0$, then by Slutsky's theorem we also have that: 
\begin{equation}
    \frac{\hat{\alpha}-\alpha_0}{\sigma_n(x)} \rightarrow \normal(0,1),
\end{equation}
as desired. Thus, it suffices to show that:
\begin{equation}
    \frac{\|\hat{\alpha} - \tilde{\alpha}\|_2}{\sigma_n(x)} \rightarrow_p 0.
\end{equation}
Observe that since $\|\beta\|_2 \leq R$, we have $\|\hat{\alpha} - \tilde{\alpha}\|_2\leq R \|\hat{\theta} - \tilde{\theta}\|_2$. Thus it suffices to show that:
\begin{equation*}
 \frac{\|\hat{\theta} - \tilde{\theta}\|}{\sigma_n(x)} \rightarrow_p 0.
\end{equation*}
\begin{lemma}\label{lem:remainder_normality}
Under the conditions of Theorem~\ref{thm:normality}, for $\sigma_n^2(x)= \Omega\bp{\frac{s^2}{n} \eta(s)}$:
\begin{equation}\label{eqn:htheta-acc}
 \frac{\|\hat{\theta} - \tilde{\theta}\|}{\sigma_n(x)} \rightarrow_p 0.
\end{equation}
\end{lemma}
\begin{proof}
Performing a second-order Taylor expansion of $m_j(x; \theta)$ around $\theta(x)$ and observing that $m_j(x;\theta(x))=0$, we have that for some $\bar{\theta}_j\in\Theta$:
\begin{align*}
    m_j(x;\hat{\theta}) =~& \ldot{\nabla_{\theta} m_j(x;\theta(x))}{\hat{\theta} - \theta(x)} +  \underbrace{(\hat{\theta}-\theta(x))^\top H_j(x;\bar{\theta}_j) (\hat{\theta}-\theta(x))^\top}_{\rho_j}\,.
\end{align*}
Letting $\rho=(\rho_1,\ldots, \rho_p)$, writing the latter set of equalities for each $j$ in matrix form, multiplying both sides by $M_0^{-1}$ and re-arranging, we get that:
\begin{equation*}
 \hat{\theta}  = \theta(x) + M_0^{-1} m (x;\hat{\theta})  - M_0^{-1}\rho \,.
\end{equation*}
Thus by the definition of $\tilde{\theta}$ we have:
\begin{equation*}
    \hat{\theta}-\tilde{\theta} = M_0^{-1} \cdot (m(x;\hat{\theta}) + \Delta(x;\theta(x))) - M_0^{-1}\rho \,.
\end{equation*}
By the bounds on the eigenvalues of $H_j(x;\theta)$ and $M_0^{-1}$, we have that:
\begin{equation}
    \|M_0^{-1}\rho\|_2 \leq \frac{L_H}{\lambda} \|\hat{\theta}-\theta(x)\|_2^2 \,.
\end{equation}
Thus we have:
\begin{equation*}
    \|\hat{\theta}-\tilde{\theta}\|_2 = \frac{1}{\lambda} \|m(x;\hat{\theta}) + \Delta(x;\theta(x))\|_2 +\frac{L_H}{\lambda} \|\hat{\theta}-\theta(x)\|_2^2 \,.
\end{equation*}

By our estimation error Theorem~\ref{thm:mse_rate}, we have that the expected value of the second term on the right hand side is of order $O\bp{\epsilon(s)^2, \frac{s}{n} \log\log(n/s)}$. Thus by the assumptions of the theorem, both are $o(\sigma_n)$. Hence, the second term is $o_p(\sigma_n)$.

We now argue about the convergence rate of the first term on the right hand side. Similar to the proof of Theorem \ref{thm:mse_rate}, since $\Psi(x;\hat{\theta})=0$ we have:
\begin{align*}
   m(x;\hat{\theta}) = m(x; \hat{\theta}) - \Psi(x; \hat{\theta})
   = m(x; \hat{\theta}) - \Psi_0(x; \hat{\theta}) + \underbrace{\Psi_0(x; \hat{\theta}) - \Psi(x;\hat{\theta})}_{\text{Sub-sampling error}} \,.
\end{align*}
We can further add and subtract $\mu_0$ from $m(x;\hat{\theta})$. 
\begin{align*}
   m(x;\hat{\theta}) =~& m(x; \hat{\theta}) - \mu_0(x;\hat{\theta}) +\mu_0(x;\hat{\theta})- \Psi_0(x; \hat{\theta}) + \Psi_0(x; \hat{\theta}) - \Psi(x;\hat{\theta})\\
   =~& m(x; \hat{\theta}) - \mu_0(x;\hat{\theta}) -\Delta(x;\hat{\theta}) + \Psi_0(x; \hat{\theta}) - \Psi(x;\hat{\theta})\,.
\end{align*}
Combining we have:
\begin{align*}
   m(x;\hat{\theta}) + \Delta(x;\theta(x)) =~& \underbrace{m(x; \hat{\theta}) - \mu_0(x;\hat{\theta})}_{C=\text{Kernel error}} + \underbrace{\Delta(x;\theta(x)) - \Delta(x;\hat{\theta})}_{F=\text{Stochastic equicontinuity term}} + \underbrace{\Psi_0(x; \hat{\theta}) - \Psi(x;\hat{\theta})}_{E=\text{Sub-sampling error}}\,.
\end{align*}
Now similar to proof of Theorem \ref{thm:mse_rate} we bound different terms separately and combine the results.
\paragraph{Kernel Error.} Term $C$ is a kernel error and hence is upper bounded by $\epsilon(s)$ in expectation. Since, by assumption $s$ is chosen such that $\epsilon(s)=o(\sigma_n(x))$, we ge that $\|C\|_2/\sigma_n(x)\rightarrow_p 0$. 

\paragraph{Sub-sampling Error.}
Term $E$ is a sub-sampling error, which can be made arbitrarily small if the number of drawn sub-samples is large enough and hence $\|E\|_2/\sigma_n(x)\rightarrow_p 0$. In fact, similar to the part about bounding sub-sampling error in Lemma \ref{lem:lambda_bound_mse} we have that that:
\begin{equation*}
    \E_{S_b}\bb{\sum_{i\in S_b} \alpha_{S_b}(X_i) \psi(Z_i;\theta)} = \Psi_0(x;\theta)\,,
\end{equation*}
Therefore, $\Psi(x;\theta)$ can be thought as the sum of $B$ independent random variables each with expectation equal to $\Psi_0(x;\theta)$. Now we can invoke Corollary \ref{cor:bracketing} in the appendix for the trivial U-process, with $s=1, n=B$ to get that w.p. $1-\delta_1$:
\begin{equation*}
    \sup_{\theta \in \Theta} \| \Psi_0(x;\theta)-\Psi(x;\theta)\| \leq O \bp{\Psi_{\max} \sqrt{\frac{\log\log(B)+\log(1/\delta_1)}{B}}}\,.
\end{equation*}
Hence, for $B\geq (n/s)^{5/4}$, due to our assumption that $\bp{s/n \log\log(n/s)}^{5/8} = o(\sigma_n(x))$ we get $\|E\|_2/\sigma_n(x)\rightarrow_p 0$.

\paragraph{Sampling Error.} 
Thus it suffices that show that $\|F\|_2/\sigma_n(x)\rightarrow_p 0$, in order to conclude that $\frac{\|m(x;\hat{\theta}) + \Psi_0(x;\theta(x))\|_2}{\sigma_n(x)}\rightarrow_p 0$.
Term $F$ can be re-written as:
\begin{equation}
    F = \Psi_0(x;\theta(x)) - \Psi_0(x; \hat{\theta}) - \E\bb{\Psi_0(x;\theta(x)) - \Psi_0(x; \hat{\theta})}\,.
\end{equation}
Observe that each coordinate $j$ of $F$, is a stochastic equicontinuity term for $U$-processes over the class of symmetric functions $\text{conv}(\cF_j)=\{f_j(\cdot;\theta): \theta \in \Theta\}$, with $f_j(Z_1,\ldots, Z_{s};\theta) = \E_{\omega}\bb{\sum_{i=1}^{s} \alpha_{Z_{1:s},\omega}(X_i) (\psi_j(Z_i; \theta(x))-\psi_j(Z_i;\theta))}$. Observe that since $f_j\in \text{conv}(\cF_j)$ is a convex combination of functions in $\cF_j=\{\psi_j(\cdot; \theta(x))-\psi_j(\cdot;\theta): \theta \in \Theta\}$, the bracketing number of functions in $\text{conv}(\cF_j)$ is upper bounded by the bracketing number of $\cF_j$, which in turn is upper bounded by the bracketing number of the function class $\{\psi_j(\cdot;\theta): \theta \in \Theta\}$, which by our assumption, satisfies $\log(N_{[]}({\cal F}_j, \epsilon, L_2))= O(1 /\epsilon)$. 
Moreover, under the variogram assumption and the lipschitz moment assumption we have that if $\|\theta-\theta(x)\|\leq r\leq 1$, then:
\begin{align*}
     \|f_j(\cdot;\theta)\|_{P,2}^2  =~& \E\bb{\bp{\sum_{i=1}^{s} \alpha_{Z_{1:s}}(X_i) (\psi_j(Z_i; \theta(x))-\psi_j(Z_i;\theta)}^2}\\
     \leq~& \E\bb{\sum_{i=1}^{s} \alpha_{Z_{1:s}}(X_i) \bp{\psi_j(Z_i; \theta(x))-\psi_j(Z_i;\theta)}^2} \tag{Jensen's inequality}\\
     =~& \E\bb{\sum_{i=1}^{s} \alpha_{Z_{1:s}}(X_i) \E\bb{\psi_j(Z_i; \theta(x))-\psi_j(Z_i;\theta)}^2|X_i]} \tag{honesty of kernel}\\
     =~&\E\bb{\sum_{i=1}^{s} \alpha_{Z_{1:s}}(X_i) \left(\Var(\psi(Z;\theta(x)) - \psi(Z;\theta)|X_i)  + \left(m(X_i;\theta(x)) - m(X_i;\theta)\right)^2\right)}\\
    \leq~& L_\psi \|\theta-\theta(x)\| + L_J^2 \|\theta - \theta(x)\|^2 \leq L_\psi r + L_J^2 r^2 = O(r)\,.
\end{align*}
Moreover, $\|f_j\|_{\infty}\leq 2\psi_{\max}$. Thus we can apply Corollary~\ref{cor:bracketing}, with $\eta= \sqrt{L_\psi r + L_J^2 r^2}=O(\sqrt{r})$ and $G=2\psi_{\max}$ to get that if $\|\hat{\theta}-\theta(x)\|\leq r$, then w.p. $1-\delta/p$:
\begin{align*}
    |F_j| \leq~& \sup_{\theta: \|\theta-\theta(x)\|\leq r} \left|\Psi_0(x;\theta(x)) - \Psi_0(x; \hat{\theta}) - \E\bb{\Psi_0(x;\theta(x)) - \Psi_0(x; \hat{\theta})}\right|\\
    =~& O\left( \left(r^{1/4} + \sqrt{r} \sqrt{\log(p/\delta) + \log\log(n/(s\,r))}\right) \sqrt{\frac{s}{n}}\right)\\
    =~& O\left( \left(r^{1/4}\sqrt{\log(p/\delta) + \log\log(n/s)}\right) \sqrt{\frac{s}{n}}\right)\triangleq \kappa(r,s,n,\delta)\,.
\end{align*}
Using a union bound this implies that w.p. $1-\delta$ we have
\begin{equation*}
    \max_j |F_j| \leq \kappa(r,s,n,\delta)\,.
\end{equation*}
By our MSE theorem and also Markov's inequality, w.p. $1-\delta'$: $\|\hat{\theta}-\theta(x)\|\leq \nu(s)/\delta'$, where:
\[
\nu(s) = \frac{1}{\lambda}\left(L_m \epsilon(s) + O\left(\psi_{\max}\sqrt{\frac{p\, s}{n} \log\log(p\, s /n)}\right)\right)
\]
Thus using a union bound w.p. $1-\delta-\delta'$, we have:
\begin{align*}
    \max_j |F_j| = O\left(\kappa(\nu(s)/\delta', s, n, \delta))\right)\,
\end{align*}
To improve readability from here we ignore all the constants in our analysis, while we keep all terms (even $\log$ or $\log\log$ terms) that depend on $s$ and $n$. Note that we can even ignore $\delta$ and $\delta'$, because they can go to zero at very slow rate such that terms $\log(1/\delta)$ or even $\delta'^{1/4}$ appearing in the analysis grow slower than $\log\log$ terms. Now, by the definition of $\nu(s)$ and $\kappa(r, s, n, \delta')$, as well as invoking the inequality $(a+b)^{1/4} \leq a^{1/4} + b^{1/4}$ for $a,b >0$ we have:
\begin{equation}
    \max_j |F_j| \leq O(\kappa(\nu(s)/\delta', s, n, \delta)) \leq O\left(\epsilon(s)^{1/4} \left(\frac{s}{n}\log\log(n/s)\right)^{1/2} + \left(\frac{s}{n}\log\log(n/s)\right)^{5/8}\right)\,,
\end{equation}
Hence, using our Assumption on the rates in the statement of Theorem \ref{thm:normality} we get that both of the terms above are $o(\sigma_n(x))$. Therefore, $\|F\|_2/\sigma_n(x) \rightarrow_p 0$. Thus, combining all of the above, we get that:
\begin{equation*}
    \frac{\|\tilde{\theta}-\hat{\theta}\|}{\sigma_n(x)} = o_p(1)
\end{equation*}
as desired.
\end{proof}

\begin{remark} 
\label{rem:inc}
Our notion of incrementality is slightly different from that of \cite{wager2017estimation}, as there the incrementality is defined as $\Var \bb{\E\bb{K(x, X_1, \{Z_j\}_{j=1}^s) \mid X_1}}$. However, using the tower law of expectation
\begin{align*}
    \E &\bb{\E[K(x, X_1, \{Z_j\}_{j=1}^s) \mid X_1]^2}  - \Var \bb{\E[K(x, X_1, \{Z_j\}_{j=1}^s)  \mid X_1]} \\
    =~& \E\bb{\E\bb{K(x, X_1, \{Z_j\}_{j=1}^s) \mid X_1}}^2 
    = \E\bb{K(x, X_1, \{Z_j\}_{j=1}^s)}^2\,.
\end{align*}
For a symmetric kernel the term $\E\bb{K(x, X_1, \{Z_j\}_{j=1}^s)}^2$ is equal to $1/s^2$ and is asymptotically negligible compared to $\Var \bb{\E\bb{K(x, X_1, \{Z_j\}_{j=1}^s) \mid X_1}}$, which usually decays at a slower rate.
\end{remark}

\section{Lower Bound on Incrementality as Function of Kernel Shrinkage}\label{app:incrementality}

We give a generic lower bound on the quantity $\E[\E[K(x, X_1, \{Z_j\}_{j=1}^s) |X_1]^2]$ that depends only on the Kernel shrinkage. The bound essentially implies that if we know that the probability that the distribution of $x$'s assigns to a ball of radius $\epsilon(s,1/2s)$ around the target $x$ is of order $1/s$, i.e. we should expect at most a constant number of samples to fall in the kernel shrinkage ball, then the main condition on incrementality of the kernel, required for asymptotic normality, holds. In some sense, this property states that the kernel shrinkage behavior is tight in the following sense. Suppose that the kernel was assigning positive weight to at most a constant number of $k$ samples. Then kernel shrinkage property states that with high probability we expect to see at least $k$ samples in a ball of radius $\epsilon(s,\delta)$ around $x$. The above assumption says that we should also not expect to see too many samples in that radius, i.e. we should also expect to see at most a constant number $K>k$ of samples in that radius. Typically, the latter should hold, if the characterization of $\epsilon(s,\delta)$ is tight, in the sense that if we expected to see too many samples in the radius, then most probably we could have improved our analysis on Kernel shrinkage and given a better bound that shrinks faster.

\subsection{Proof of Lemma \ref{lem:shrinkage_incrementality}}
By the Paley-Zygmund inequality, for any random variable $Z\geq 0$ and for any $\delta\in [0,1]$:
\begin{align*}
    \E[Z^2] \geq (1-\delta)^2 \frac{\E[Z]^2}{\Pr[Z\geq \delta \E[Z]]}
\end{align*}
Let $W_1=K(x, X_1, \{Z_j\}_{j=1}^s)$. Then, applying the latter to the random variable $Z=\E[W_1|X_1]$ and observing that by symmetry $\E[Z]=\E[W_1]=1/s$, yields:
\begin{align*}
    \E\bb{\E[W_1|X_1]^2} \geq \frac{\left(1-\delta\right)^2 \E[W_1]^2}{\Pr[\E[W_1|X_1]> \delta \E[W_1]]} =\frac{\left(1-\delta\right)^2 (1/s)^2}{\Pr[\E[W_1|X_1]> \delta/s]} 
\end{align*}
Moreover, observe that by the definition of $\epsilon(s,\rho)$ for some $\rho>0$:
\begin{align*}
    \Pr[W_1>0 \wedge \|X_1-x\|\geq \epsilon(s,\rho)] \leq \rho
\end{align*}
This means that at most a mass $\rho\,s/\delta$ of the support of $X_1$ in the region $\|X_1 - x\|\geq \epsilon(s, \rho)$ can have $\Pr[W_1 > 0 | X_1]\geq \delta/s$. Otherwise the overall probability that $W_1 > 0$ in the region of $\|X_1 - x\| \geq \epsilon(s,\rho)$ would be more than $\rho$. Thus we have that except for a region of mass $\rho\,s/\delta$, for each $X_1$ in the region $\|X_1 - x\|\geq \epsilon(s,\rho)$: $\E[W_1 | X_1] \leq \delta/s$. Combining the above we get:
\begin{equation*}
    \Pr[ \E[W_1 | X_1] \leq \delta/s ] \geq \Pr[\|X_1 - x\| \geq \epsilon(s,\rho)] - \rho\,s/\delta
\end{equation*}
Thus:
\begin{equation*}
    \Pr[ \E[W_1 | X_1] > \delta/s ] \leq \Pr[\|X_1 - x\| \leq \epsilon(s,\rho)] + \rho\,s/\delta= \mu(B(x,\epsilon(s,\delta))) + \rho\,s/\delta
\end{equation*}
Since $\rho$ was arbitrarily chosen, the latter upper bound holds for any $\rho$, which yields the result.

\subsection{Proof of Corollary \ref{cor:incrementality}}
Thus applying Lemma \ref{lem:shrinkage_incrementality} with $\delta=1/2$ yields:
\begin{equation*}
    \E[\E[K(x, X_1, \{Z_j\}_{j=1}^s) |X_1]^2] \geq \frac{\left(1/2s\right)^2}{\inf_{\rho>0} \left(\mu(B(x, \epsilon(s, \rho)))+2\rho\,s\right)}
\end{equation*}
Observe that:
\begin{equation*}
    \mu(B(x,\epsilon(s, \rho))) \leq C \epsilon(s, \rho)^d \mu(B(x,r)) = O\left(\frac{\log(1/\rho)}{s}\right)
\end{equation*}
Hence:
\begin{equation*}
    \inf_{\rho>0} \left(\mu(B(x, \epsilon(s, \rho)))+2\rho\,s\right) = O\left(\inf_{\rho>0} \left(\frac{\log(1/\rho)}{s}+2\rho\,s\right)\right) = O\left(\frac{\log(s)}{s}\right)
\end{equation*}
where the last follows by choosing $\rho=1/s^2$. Combining all the above yields:
\begin{equation*}
    \E[\E[K(x, X_1, \{Z_j\}_{j=1}^s) |X_1]^2] =\Omega\left(\frac{1}{s\log(s)}\right)
\end{equation*}

\section{Proofs of Section \ref{sec:main-thm}}\label{app:missings}
\subsection{Proof of Lemma \ref{lem:pb-kernel-shrink-kNN}}
For proving this result, we rely on Bernstein's inequality which is stated below:
\begin{proposition}[Bernstein's Inequality]\label{prop:bernstein}
Suppose that random variables $Z_1,Z_2,\ldots,Z_n$ are i.i.d., belong to $[-c,c]$ and $\E[Z_i] = \mu$. Let $\bar{Z}_n = \frac{1}{n} \sum_{i=1}^n Z_i$ and $\sigma^2 = \Var(Z_i)$. Then, for any $\theta>0$,
\begin{equation*}
\Pr \bp{|\bar{Z}_n - \mu| > \theta} \leq 2 \exp \bp{\frac{-n \theta^2}{2 \sigma^2 + 2 c \theta/3}}.
\end{equation*}
This also implies that w.p. at least $1-\delta$ the following holds:
\begin{equation}\label{eqn:bernstein-ineq}
|\bar{Z}_n - \mu| \leq \sqrt{\frac{2 \sigma^2 \log(2/\delta)}{n}} + \frac{2c \log(2/\delta)}{3n}.
\end{equation}
\end{proposition}
Let $A$ be any $\mu$-measurable set. An immediate application of Bernstein's inequality to random variables $Z_i = \,1 \bc{X_i \in A}$, implies that w.p. $1-\delta$ over the choice of covariates $\bp{X_i}_{i=1}^s$, we have:
\begin{equation*}
|\mu_s(A)-\mu(A)| \leq \sqrt{\frac{2 \mu(A) \log(2/\delta)}{s}} + \frac{2 \log(2/\delta)}{3s}.
\end{equation*}
In above, we used the fact that $\Var(Z_i) = \mu(A)(1-\mu(A)) \leq \mu(A)$. This result has the following corollary.

\begin{corollary}\label{cor:emp_close_pop}
Define $U = 2\log(2/\delta)/s$ and let $A$ be an arbitrary $\mu$-measurable set. Then, w.p. $1-\delta$ over the choice of training samples, $\mu(A) \geq 4U$ implies $\mu_s(A) \geq U$.
\end{corollary}
\begin{proof}
Define $U = 2\log(2/\delta)/s$. Then, Bernstein's inequality in Proposition \ref{prop:bernstein} implies that w.p. $1-\delta$ we have
\begin{equation*}
    |\mu_s(A)-\mu(A)| \leq \sqrt{U\mu(A)} + \frac{U}{3}\,.
\end{equation*}
 Assume that $\mu(A) \geq 4U$, we want to prove that $\mu_s(A) \geq U$. Suppose, the contrary, i.e., $\mu_s(A) < U$. Then, by dividing the above equation by $\mu(A)$ we get
\begin{equation*}
    \Big|\frac{\mu_s(A)}{\mu(A)} - 1 \Big| \leq \sqrt{\frac{U}{\mu(A)}} + \frac{1}{3} \frac{U}{\mu(A)}\\,.
\end{equation*}
Note that since $\mu_s(A)<U<\mu(A)$, by letting $z = U/\mu(A) \leq 1/4$ the above implies that
\begin{equation*}
    1-z \leq \sqrt{z} + \frac{z}{3} \Rightarrow \frac{4}{3} z + \sqrt{z} - 1 \geq 0\,,
\end{equation*}
which as $z>0$ only holds for 
\begin{equation*}
\sqrt{z} \geq \frac{-3+\sqrt{57}}{8} \Rightarrow z \geq 0.3234\,.
\end{equation*}
This contradicts with $z \leq 1/4$, implying the result.
\end{proof}
Now we are ready to finish the proof of Lemma \ref{lem:pb-kernel-shrink-kNN}. First, note that using the definition of $(C,d)$-homogeneous measure. Note that for any $\theta \in (0,1)$ we have $\mu(B(x,\theta r)) \geq (1/C) \theta^d \mu(B(x,r))$. Replace $\theta r = \epsilon$ in above. It implies that for any $\epsilon \in (0,r)$
    \begin{equation}\label{eqn:theta-to-eps-kNN}
        \mu(B(x,\epsilon)) \geq \frac{1}{C\,r^d} \epsilon^d \mu(B(x,r))\,.
    \end{equation}
    Pick $\epsilon_k(s,\delta)$ according to
    \begin{equation*}
        \epsilon_k(s,\delta) = r \bp{\frac{8C\,\log(2/\delta)}{\mu(B(x,r)) s}}^{1/d}\,.
    \end{equation*}
    Note that for having $\epsilon_k(s,\delta) \in (0,r)$ we need
    \begin{equation*}
        \log(2/\delta) \leq \frac{1}{8\,C} \mu(B(x,r)) s \Rightarrow \delta \geq 2 \exp \bp{-\frac{1}{8\,C}\mu(B(x,r)s} \,.
    \end{equation*}
    Therefore, replacing this choice of $\epsilon_k(s,\delta)$ in Equation \eqref{eqn:theta-to-eps-kNN} implies that $\mu(B(x,\epsilon_k(s,\delta)))\geq \frac{8\log(2/\delta)}{s}$.
    Now we can use the result of Corollary \ref{cor:emp_close_pop} for the choice $A=B(x,\epsilon_k(s,\delta))$. It implies that w.p. $1-\delta$ over the choice of $s$ training samples, we have 
    \begin{equation*}
        \mu_s(B(x,\epsilon_k(s,\delta))) \geq \frac{2\log(2/\delta)}{s}\,.
    \end{equation*}
    Note that whenever $\delta \leq \exp(-k/2)/2$ we have
    \begin{equation*}
        \frac{2\log(2/\delta)}{s} \geq \frac{k}{s} \,.
    \end{equation*}
    Therefore, w.p. $1-\delta$ we have
    \begin{equation*}
        \|x-X_{(k)}\| \leq \epsilon_k(s,\delta) = O \bp{\frac{\log(1/\delta)}{s}}^{1/d} \,.
    \end{equation*}

\subsection{Proof of Corollary \ref{cor:exp-kernel-shrink-kNN}}
    Lemma~\ref{lem:pb-kernel-shrink-kNN} shows that for any  $t = \epsilon_k(s,\delta) = r \bp{\frac{8C\log(2/\delta)}{\mu(B(x,r)) s}}^{1/d}$, such that $t\leq r$ and $t\geq r \bp{\frac{4k\,C}{ \mu(B(x,r)) s}}^{1/d}$, we have that:
    \begin{equation*}
        \Pr[\|x - X_{(k)}\|_2\geq \epsilon_k(s,\delta)] \leq \delta\,.
    \end{equation*}
    Let $\rho = \frac{1}{r}\left(\frac{\mu(B(x,r))}{8C}\right)^{1/d}$, which is a constant. Solving for $\delta$ in terms of $t$ we get:
    \begin{equation*}
        \Pr[\|x - X_{(k)}\|_2\geq t] \leq 2 \exp \bp{- \rho^d\, s\,t^d}\,,
    \end{equation*}
    for any $t\in \left[\frac{(s/k)^{-1/d}}{\rho} , r\right]$. Thus, noting that $X_i$'s and target $x$ both belong to $\calX$ that has diameter $\Delta_{\calX}$, we can upper bound the expected value of $[\|x - X_{(k)}\|_2$ as:
    \begin{align*}
        \E \bb{\|x - X_{(k)}\|_2} =~& \int_{0}^{\Delta_{\calX}} \Pr\bb{\|x - X_{(k)}\|_2 \geq t} dt\\
        \leq~& \frac{(s/k)^{-1/d}}{\rho} + \int_{\rho\, (s/k)^{-1/d}}^{r} \Pr\bb{\|x - X_{(k)}\|_2 \geq t} dt + \Pr\bb{\|x - X_{(k)}\|_2 \geq r} (\Delta_{\calX} - r)\\
        \leq~& \frac{(s/k)^{-1/d}}{\rho} + \int_{\rho\, (s/k)^{-1/d}}^{r} 2 \exp\left\{- \rho^d\,s\, t^d\right\} dt + 2 \exp\left\{- \rho^d\, r^d\, s\right\} (\Delta_{\calX} - r)\,.
    \end{align*}
    Note that for $s$ larger than some constant, we have $\exp\left\{- \rho^d\, r^d\, s\right\}\leq s^{-1/d}$. Thus the first and last terms in the latter summation are of order $\bp{\frac{1}{s}}^{1/d}$. We now show that the same holds for the middle term, which would complete the proof. By setting $u=\rho^d\,s\, t^d$ and doing a change of variables in the integral we get: 
    \begin{align*}
        \int_{\rho\, (s/k)^{1/d}}^{r} 2 \exp\left\{- \rho^d\,s\, t^d\right\} dt &
        \leq \int_{0}^{\infty} 2 \exp\left\{- \rho^d\,s\, t^d\right\} dt \\
        =~& \frac{1}{d\, \rho\, s^{1/d}} \int_{0}^{\infty} u^{1/d -1} \exp\left\{-u\right\} du = \frac{s^{-1/d}}{\rho} \frac{1}{d} \Gamma(1/d)\,.
    \end{align*}
    where $\Gamma$ is the Gamma function. Since by the properties of the Gamma function $z \Gamma(z)=\Gamma(z+1)$, the latter evaluates to: $ \frac{s^{-1/d}}{\rho} \Gamma((d+1)/d)$. Since $(d+1)/d\in [1,2]$, we have that $\Gamma((d+1)/d)\leq 2$. Thus the middle term is upper bounded by $\frac{2s^{-1/d}}{\rho}$, which is also of order $\bp{\frac{1}{s}}^{1/d}$.
    
\subsection{Proof of Lemma \ref{lem:kNN-inc}}
Before proving this lemma we state and prove and auxiliary lemma which comes in handy in our proof.
\begin{lemma}\label{lem:linearity-knn}
Let $P_1$ denote the mass that the density of the distribution of $X_i$ puts on the ball around $x$ with radius $\|x-X_1\|_2$, which is a random variable as it depends on $X_1$. Then, for any $s \geq k$ the following holds:
    \begin{align*}
        \E\bb{\sum_{i=0}^{k-1}\binom{s-1}{i} (1 - P_1)^{s-1-i} P_1^{i}} = \E \bb{\E \bb{S_1 \mid X_1}} = \frac{k}{s}\,.
    \end{align*}
\end{lemma}
\begin{proof}
The proof is an easy consequence of symmetry.  Let $S_1 = \,1\{\text{sample 1 is among $k$ nearest neighbors}\}$, then we can write
\begin{equation*}
    \E \bb{ \E [S_1 \mid X_1]}= \E\bb{\sum_{i=0}^{k-1}\binom{s-1}{i} (1 - P_1)^{s-1-i} P_1^{i}}\,,
\end{equation*}
which simply computes the probability that there are at most $k-1$ other points in the ball with radius $\|x-X_1\|$. Now, by using the tower law
\begin{equation*}
    \E \bb{ \E [S_1 \mid X_1]} = \E[S_1] = \frac{k}{s}\,,
\end{equation*}
which holds because of the symmetry. In other words, the probability that sample $1$ is among the $k$-NN is equal to $k/s$. Hence, the conclusion follows.
\end{proof}

We can finish the proof of Lemma \ref{lem:kNN-inc}. Define $S_1 = \,1\{\text{sample 1 is among $k$ nearest neighbors}\}$, then we can write
\begin{equation*}
    \E\bb{\E\bb{K(x, X_1, \{Z_j\}_{j=1}^s) \mid X_1}^2} = \frac{1}{k^2} \E\bb{\E\bb{S_1 \mid X_1}^2}\,.
\end{equation*}
Recall that if $P_1$ denotes the mass that the density of the distribution of $X_i$ puts on the ball around $x$ with radius $\|x-X_1\|_2$, which is a random variable depending on $X_1$. Therefore,
    \begin{align*}
        \E\bb{S_1 \mid X_1} = \sum_{i=0}^{k-1}\binom{s-1}{i} (1 - P_1)^{s-1-i} P_1^{i}\,.
    \end{align*}
 Now we can write
    \begin{align*}
        \E\bb{\E \bb{S_1 \mid X_1}^2} =~& \E\bb{\bp{\sum_{i=0}^{k-1}\binom{s-1}{i} (1 - P_1)^{s-1-i} P_1^{i}}^2}\\
        =~& \E\bb{\sum_{i=0}^{k-1}\sum_{j=0}^{k-1} \binom{s-1}{i}\,\binom{s-1}{j} (1 - P_1)^{2s-2-i-j} P_1^{i+j}}\\
        =~& \E\bb{\sum_{t=0}^{2k-2} (1 - P_1)^{2s-2-t} P_1^{t} \sum_{i=0}^{k-1}\sum_{j=0}^{k-1} \binom{s-1}{i}\,\binom{s-1}{j}\,1\bc{i+j=t}} \\
        =~& \E\bb{\sum_{t=0}^{2k-2} (1 - P_1)^{2s-2-t} P_1^{t} \sum_{i=\max \{0, t-(k-1)\}}^{\min\{t, k-1\}}\binom{s-1}{i}\,\binom{s-1}{t-i}} \\
        =~& \E\bb{\sum_{t=0}^{2k-2} a_t\, (1 - P_1)^{2s-2-t} P_1^{t}}
    \end{align*}
    Now using Lemma \ref{lem:linearity-knn} (where $s$ is replaced by $2s-1$) we know that for any value of $0 \leq r \leq 2s-2$ we have
    \begin{equation} \label{eqn:rec}
        \E \bb{\sum_{t=0}^{r} b_t\, (1 - P_1)^{2s-2-r} P_1^{r}}
        = \E \bb{\sum_{t=0}^{r} \binom{2s-2}{t}\, (1 - P_1)^{2s-2-t} P_1^{t}} = \frac{r+1}{2s-1}\,.
    \end{equation}
    This implies that for any value of $r$ we have $\E \bb{b_r (1 - P_1)^{2s-2-r} P_1^{r}} = 1/(2s-1)$.
    The reason is simple. Note that the above is obvious for $r=0$ using Equation \eqref{eqn:rec}. For other values of $r \geq 1$, we can write Equation \eqref{eqn:rec} for values $r$ and $r-1$. Taking their difference implies the result. Note that this further implies that
    $\E \bb{(1 - P_1)^{2s-2-r} P_1^{r}} = 1/(b_r\,(2s-1))$, as $b_r$ is a constant. Therefore, by plugging this back into the expression of $\E[\E[S_1 \mid X_1]^2]$ we have
    \begin{equation*}
        \E[\E[S_1|X_1]^2] = \E \bb{ \sum_{t=0}^{2k-2} a_t\, (1 - P_1)^{2s-2-t} P_1^{t}} = \frac{1}{2s-1} \bp{ \sum_{t=0}^{2k-2} \frac{a_t}{b_t}}\,,
    \end{equation*}
    which implies the desired result. 
\begin{remark}\label{rem:a-b-seq}
Note that $b_t = \binom{2s-2}{t}$ since we can view $b_t$ as follows: how many different subsets of size $t$ can we create from a set of $2s-2$ elements if we pick a number $i=\{0,\ldots,t\}$ and then choose $i$ elements from the first half of these elements and $t-i$ elements from the second half. Observe that this process creates all possible sets of size $t$ from among the $2s-2$ elements, which is equal to $\binom{2s-2}{t}$.

Furthermore, $a_t = b_t $ for $0 \leq t \leq k-1$ and for any $k \leq t \leq 2k-2$, after some little algebra, we have
    \begin{equation*}
       \frac{2k-1-t}{t+1} \leq \frac{a_t}{b_t} \leq 1\,.
    \end{equation*}
    This implies that the summation appeared in Lemma \ref{lem:lower_var} satisfies
    \begin{equation*}
        k + \sum_{t=k}^{2k-2} \frac{2k-1-t}{t+1} \leq \sum_{t=0}^{2k-2} \frac{a_t}{b_t} \leq 2k-1\,.
    \end{equation*}
\end{remark}

\subsection{Proof of Theorem \ref{thm:knn_var_range}}
Note that according to Lemma \ref{lem:hajek}, the asymptotic variance $\sigma_{n,j}^2(x)=\frac{s^2}{n}\Var\bb{\Phi_1(Z_1)}$, where $\Phi_1(Z_1) = \frac{1}{k} \E[\sum_{i \in H_k(x,s)} \ldot{e_j}{M_0^{-1}\psi(Z_i;\theta(x))} \mid Z_1]$. Therefore, once we establish an expression for $\Var\bb{\Phi_1(Z_1)}$ we can finish the proof of this theorem. The following lemma provides such result.

\begin{lemma}\label{lem:lower_var}
Suppose that the kernel $K$ is the $k$-NN kernel and let $\sigma_j^2(x) = \Var\bp{\ldot{e_j} {M_0^{-1}\psi(z;\theta(x))} \mid X=x}$. Moreover, suppose that $\epsilon_{k}(s, 1/s^2)\rightarrow 0$ for any constant $k$. Then:
\begin{equation*}
   \Var\bb{\Phi_1(Z_1)} = \sigma_j(x)^2\,\E\bb{\E\bb{K(x, X_1, \{Z_j\}_{j=1}^s) \mid X_1}^2} + o(1/s) = \frac{\sigma_j^2(x)}{(2s-1) k^2}\, \bp{\sum_{t=0}^{2k-2} \frac{a_t}{b_t}} + o(1/s)\,
\end{equation*}
where the second equality above holds due to Lemma \ref{lem:kNN-inc} and sequences $a_t$ and $b_t$, for $0 \leq t \leq 2k-2$, are defined in Lemma \ref{lem:kNN-inc}.
\end{lemma}
\begin{proof}
In this proof for simplicity we let $Y_i = \ldot{e_j}{M_0^{-1} \psi(Z_i;\theta(x))}$ and $\mu(X_i) = \E[Y_i] = \ldot{e_j}{M_0^{-1}m(X_i;\theta(x))}$. Let $Z_{(i)}$ denote the random variable of the $i$-th closest sample to $x$. For the case of $k$-NN we have that:
    \begin{align*}
        k\,\Phi_1(Z_1) = \E\bb{\sum_{i=1}^k Y_{(i)} \mid Z_1}\,.
    \end{align*}
    Let $S_1 = \mathbf{1}\{\text{sample 1 is among $k$ nearest neighbors}\}$. Then we have:
    \begin{align*}
        k\,\Phi_1(Z_1) = \E\bb{S_1 \sum_{i=1}^k Y_{(i)} \mid Z_1} + \E\bb{(1-S_1)\sum_{i=1}^k Y_{(i)} \mid Z_1}\,.
    \end{align*}
    Let $\tilde{Y}_{(i)}$ denote the label of the $i$-th closest point to $x$, excluding sample $1$. Then:
    \begin{align*}
        k\, \Phi_1(Z_1) =~& \E\bb{S_1 \sum_{i=1}^k Y_{(i)} \mid Z_1} + \E\bb{(1-S_1)\sum_{i=1}^k \tilde{Y}_{(i)} \mid Z_1}\\
        =~&\E\bb{S_1 \sum_{i=1}^k \bp{Y_{(i)} - \tilde{Y}_{(i)}} \mid Z_1} + \E\bb{\sum_{i=1}^k \tilde{Y}_{(i)}\mid Z_1}\,.
    \end{align*}
    Observe that $\tilde{Y}_{(i)}$ are all independent of $Z_1$. Hence:
    \begin{align*}
        k\,\Phi_1(Z_1) =~& \E\bb{S_1 \sum_{i=1}^k \bp{Y_{(i)} - \tilde{Y}_{(i)}} \mid Z_1} + \E\bb{\sum_{i=1}^k \tilde{Y}_{(i)}}\,.
    \end{align*}
    Therefore the variance of $\Phi(Z_1)$ is equal to the variance of the first term on the right hand side. 
    Hence:
    \begin{align*}
        k^2\,\Var\bb{\Phi_1(Z_1)} =~& \E\bb{\E\bb{S_1 \sum_{i=1}^k \bp{Y_{(i)} - \tilde{Y}_{(i)}} \mid Z_1}^2} - \E\bb{S_1 \sum_{i=1}^k \bp{Y_{(i)} - \tilde{Y}_{(i)}}}^2\\
        =~& \E\bb{\E\bb{S_1 \sum_{i=1}^k \bp{Y_{(i)} - \tilde{Y}_{(i)}} \mid Z_1}^2} + o(1/s)\,.
    \end{align*}
    Where we used the fact that:
    \begin{equation}
        \left|\E\bb{S_1 \sum_{i=1}^k \bp{Y_{(i)} - \tilde{Y}_{(i)}}}\right| \leq \E\bb{S_1} 2k\psi_{\max} = \frac{2k^2\psi_{\max}}{s}\,.
    \end{equation}
    Moreover, observe that under the event that $S_1=1$, we know that the difference between the closest $k$ values and the closest $k$ values excluding $1$ is equal to the difference between the $Y_1$ and $Y_{(k+1)}$. Hence:
    \begin{align*}
        \E\bb{S_1 \sum_{i=1}^k \bp{Y_{(i)} - \tilde{Y}_{(i)}}\mid Z_1} = \E\bb{S_1 \bp{Y_{1} - Y_{(k+1)}} \mid Z_1} = \E\bb{S_1 \bp{Y_{1} - \mu(X_{(k+1)})} \mid Z_1}\,.
    \end{align*}
    where the last equation holds from the fact that for any $j\neq 1$, conditional on $X_{j}$, the random variable $Y_{j}$ is independent of $Z_1$ and is equal to $\mu(X_j)$ in expectation. Under the event $S_1=1$, we know that the $(k+1)$-th closest point is different from sample $1$. 
    We now argue that up to lower order terms, we can replace $\mu(X_{(k+1)})$ with $\mu(X_{1})$ in the last equality:
    \begin{align*}
        \E\bb{S_1 \bp{Y_{1} - \mu(X_{(k+1)})} \mid Z_1} =~& \underbrace{\E\bb{S_1 \bp{Y_{1} - \mu(X_{1})} \mid Z_1}}_{A} + \underbrace{\E\bb{S_1 \bp{\mu(X_{1}) - \mu(X_{(k+1)})} \mid Z_1}}_{\rho}\,.
    \end{align*}
    Observe that:
    \begin{align*}
        \E\bb{\E\bb{S_1 \bp{Y_{1} - \mu(X_{(k+1)})} \mid Z_1}^2} = \E[A^2] + \E[\rho^2] + 2\E[A\rho]\,.
    \end{align*}
    Moreover, by Jensen's inequality, Lipschitzness of the first moments and kernel shrinkage:
    \begin{align*}
    \left|\E[\rho^2]\right| =~& \E\bb{\E\bb{S_1 \bp{\mu(X_{1}) - \mu(X_{(k+1)})} \mid Z_1}^2}
    \leq \E\bb{S_1 \bp{\mu(X_{1}) - \mu(X_{(k+1)})}^2}\\
    \leq~&
    4L_m^2\epsilon_{k+1}(s,\delta)^2 \E[\E[S_1|X_1]] + 4 \delta \psi_{\max}^2
    \leq
    4L_m^2 \epsilon_{k+1}(s,\delta)^2 \frac{k}{s} + 4 \delta \psi_{\max}^2\,.
    \end{align*}
    Hence, for $\delta=1/s^2$, the latter is $o(1/s)$. Similarly:
    \begin{align*}
    |\E[A\rho]|\leq \E[|A|\, |\rho|] \leq~& \psi_{\max} \E\bb{\E\bb{S_1 \left|\mu(X_{1}) - \mu(X_{(k+1)}\right|\mid Z_1}} = \psi_{\max} \E\bb{S_1 \left|\mu(X_{1}) - \mu(X_{(k+1)}\right|}\\
    \leq~& \psi_{\max} \E[S_1]\epsilon_{k+1}(s,\delta) + 2\delta \psi_{\max} = \psi_{\max} \epsilon_{k+1}(s,\delta) \frac{k}{s} + 2\delta \psi_{\max} \,.
    \end{align*} 
    which for $\delta=1/s^2$ is also of order $o(1/s)$.
    Combining all the above we thus have:
    \begin{align*}
        k^2\,\Var\bb{\Phi_1(Z_1)} =~& \E\bb{\E\bb{S_1 \bp{Y_{1} - \mu(X_{1})} \mid Z_1}^2} + o(1/s)\\
        =~& \E\bb{\E\bb{S_1\mid X_1}^2 \bp{Y_{1} - \mu(X_{1})}^2} + o(1/s)\,.
    \end{align*}
    We now work with the first term on the right hand side. By the tower law of expectations:
    \begin{align*}
        \E\bb{\E[S_1 \mid X_1]^2 \bp{Y_1 - \mu(X_1)}^2} =~& \E\bb{\E[S_1 \mid X_1]^2 \E\bb{Y_1 - \mu(X_1)^2\mid X_1}} = \E\bb{\E[S_1 \mid X_1]^2 \sigma_j^2(X_1)}\\
        =~& \E\bb{\E[S_1 \mid X_1]^2 \sigma_j^2(x)} + \E\bb{\E[S_1 \mid X_1]^2 \bp{\sigma_j^2(X_1) - \sigma_j^2(x)}}\,.
    \end{align*}
    By Lipschitzness of the second moments, we know that the second part is upper bounded as:
    \begin{align*}
        \left|\E\bb{\E[S_1 \mid X_1]^2 \bp{\sigma_j^2(X_1) - \sigma_j^2(x)}}\right| \leq~& \left|\E\bb{\E[S_1 \mid X_1] \bp{\sigma_j^2(X_1) - \sigma_j^2(x)}}\right|\\
        \leq~& \left|\E\bb{S_1 \bp{\sigma_j^2(X_1) - \sigma_j^2(x)}}\right|\\
        =~& \left|\E\bb{S_1 \bp{\sigma_j^2(X_{(k)}) - \sigma_j^2(x)}}\right|\\
        \leq~& L_{mm}\E\bb{S_1} \epsilon_k(s,\delta) + \delta \psi_{\max}^2\\
        =~& \frac{L_{mm}\, \epsilon_k(s, \delta)\, k}{s} + \delta \psi_{\max}^2\,.
    \end{align*}
    For $\delta=1/s^2$ it is of $o(1/s)$. Thus:
    \begin{align*}
        k^2\, \Var\bb{\Phi_1(Z_1)} = \E\bb{\E[S_1 \mid X_1]^2} \sigma_j^2(x) + o(1/s)\,.
    \end{align*}
    Note that Lemma \ref{lem:kNN-inc} provides an expression for $\E \bb{\E[S_1 \mid X_1]^2}$ which finishes the proof.
    \end{proof}
    
For finishing proof of Theorem \ref{thm:knn_var_range} we need to prove that $\sum_{t=0}^{2k-2} \frac{a_t}{b_t}$ is equal to $\zeta_k$ plus lower order terms. This is proved in the following lemma.

\begin{lemma}
\label{lem:a-b_to_zeta}
Suppose that $s \rightarrow \infty$ and $k$ is fixed. Then
\begin{equation*}
    \sum_{t=0}^{2k-2} \frac{a_t}{b_t} = \zeta_k + O(1/s)\,.
\end{equation*}
\begin{proof}
Note that for any $0 \leq t \leq k-1$ we have $a_t = b_t$ according to Remark \ref{rem:a-b-seq}. For any $k \leq t \leq 2k-2$ we have
\begin{align*}
    \frac{a_t}{b_t} 
    &= \sum_{i=t-k+1}^{k-1} \frac{\binom{s-1}{i} \binom{s-1}{t-i}}{\binom{2s-2}{t}}
    = \sum_{i=t-k+1}^{k-1} \frac{\frac{(s-1)(s-2)\ldots(s-i)}{i\,!} \frac{(s-1)(s-2)\ldots(s-t+i)}{(t-i)\,!}}{\frac{(2s-2)(2s-3)\ldots(2s-1-t)}{t\,!}} \\
    &= \sum_{i=t-k+1}^{k-1} \binom{t}{i} \frac{(s-1)(s-2)\ldots(s-i)\, \, (s-1)(s-2)\ldots(s-t+i)}{(2s-2)(2s-3)\ldots(2s-1-t)} \\
    &= \sum_{i=t-k+1}^{k-1} \binom{t}{i} \frac{s-1}{2s-2} \frac{s-2}{2s-3} \cdots \frac{s-i}{2s-1-i}\,\frac{s-1}{2s-i}\frac{s-2}{2s-i-1} \cdots \frac{s-t+i}{2s-1-t} \\
    &= \sum_{i=t-k+1}^{k-1} 2^{-t}  \binom{t}{i} \bp{1-\frac{1}{2s-3}} \cdots \bp{1-\frac{i-1}{2s-1-i}} \, \bp{1+\frac{i-2}{2s-i}} \cdots \bp{1+\frac{i-(i-t+1)}{2s-1-t}} \\
    &= 2^{-t} \sum_{i=t-k+1}^{k-1}  \binom{t}{i} (1+O(1/s)) \\
    &= 2^{-t} \sum_{i=t-k+1}^{k-1} \binom{t}{i} + O(1/s)\,,
\end{align*}
where we used the fact that $t$ and $i$ are both bounded above by $2k-2$ which is a constant. Hence, 
\begin{equation*}
    \sum_{t=0}^{2k-2} \frac{a_t}{b_t} = k + \sum_{t=k}^{2k-2} 2^{-t} \sum_{i=t-k+1}^{k-1} \binom{t}{i} + O(1/s) = \zeta_k + O(1/s)\,,
\end{equation*}
as desired.
\end{proof}
\end{lemma}
\subsection{Proof of Theorem \ref{thm:knn-normality}}
The goal is to apply Theorem \ref{thm:normality}. Note that $k$-NN kernel is both honest and symmetric. According to Lemma \ref{lem:pb-kernel-shrink-kNN}, we have that $\epsilon_k(s,\delta) = O\bp{ (\log(1/\delta)/s)^{1/d}}$ for $\exp(-Cs) \leq \delta \leq D$, where $C$ and $D$ are constants. Corollary \ref{cor:exp-kernel-shrink-kNN} also implies that $\epsilon_k(s) = O((1/s)^{1/d})$. Furthermore, according to Lemma \ref{lem:kNN-inc}, the incrementality $\eta_k(s)$ is $\Theta(1/s)$. Therefore, as $s$ goes to $\infty$ we have $\epsilon_k(s,\eta_k(s)) = O\bp{(\log(s)/s)^{1/d}} \rightarrow 0$. Moreover, as $\eta_k(s) = \Theta(1/s)$, we also get that $n \eta_k(s) = O(n/s) \rightarrow \infty$. We only need to ensure that Equation \eqref{eqn:rate-condition-normality} is satisfied. Note that $\sigma_{n,j}(x) = \Theta(\sqrt{s/n})$. Therefore, by dividing terms in Equation \eqref{eqn:rate-condition-normality} it suffices that
\begin{equation*}
    \max \bp{s^{-1/d} \bp{\frac{n}{s}}^{1/2}, s^{-1/4d} \bp{\log\log(n/s) }^{1/2}, \bp{\frac{n}{s}}^{-1/8} \, \bp{\log\log(n/s)}^{5/8}} = o(1)\,.
\end{equation*}
Note that due to our Assumption $n/s \rightarrow \infty$, the last term obviously goes to zero. Also, because of the assumption made in the statement of theorem, the first term also goes to zero. We claim that if the first term goes to zero, the same also holds for the second term. Note that we can write
\begin{equation*}
   s^{-1/4d} \bp{\log\log(n/s) }^{1/2} =  \bp{s^{-1/d} \bp{\frac{n}{s}}^{1/2}}^{1/4} \cdot \bb{\bp{\frac{n}{s}}^{-1/8} \, \bp{\log\log(n/s)}^{1/2}}\,,
\end{equation*}
and since $n/s \rightarrow \infty$, our claim follows. Therefore, all the conditions of Theorem \ref{thm:normality} are satisfied and the result follows.

The second part of result is implied by the first part since if $s = n^{\beta}$ and $\beta \in (d/(d+2),1)$ then $s^{-1/d} \sqrt{\frac{n}{s}} \rightarrow 0$.

\subsection{Proof of Proposition \ref{thm:est-adapt}}
For proving this lemma, we need two following auxiliary results. Before that we state the Hoeffding's inequality for $U$-statistics \cite{hoeffding1994probability}.
\begin{proposition}\label{prop:hoeff-U}
Suppose that $\bX = \bp{X_1,X_2,\ldots,X_n}$ are i.i.d. and $q$ is a function that has range $[0,1]$.  Define $U_s = \binom{n}{s}^{-1} \sum_{i_1 < i_2 < \ldots < i_s} q(X_{i_1},X_{i_2},\ldots, X_{i_s})$. Then, for any $\epsilon > 0$
\begin{equation*}
    \Pr \bb{|U_s - \E[U_s]| \geq \epsilon} \leq 2 \exp \bp{-\lfloor{n/s \rfloor} \epsilon^2}\,.
\end{equation*}
Furthermore, for any $\delta>0$, w.p. $1-\delta$ we have
\begin{equation*}
    |U_s - \E[U_s]| \leq \sqrt{\frac{1}{\lfloor{n/s \rfloor}} \log(2/\delta)}\,.
\end{equation*}
\end{proposition}
\begin{lemma}\label{lem:eps-G}
Consider the function $H(s)$ defined in Section \ref{ssec:adapt} and $G_\delta(s) = \Delta \sqrt{2ps/n \log(2np/\delta)}$. Then, w.p. $1-\delta$, for all values of $k \leq s \leq n$ we have
\begin{equation*}
    |H(s) - \epsilon_k(s)| \leq G_\delta(s)\,.
\end{equation*}
\begin{proof}
Note that $H(s)$ is the complete $U$-statistic estimator for $\epsilon_k(s)$. For each subset $S$ of size $s$ from $[n]$ we have
\begin{equation*}
    \E[\max_{X_i \in H_k(x,S)} \|x-X_i\|_2] = \epsilon_k(s)\,.
\end{equation*}
Further, $\|x-x'\|_2 \leq \Delta_{\calX} \leq \Delta$ holds for any $x' \in \calX$. Therefore, using Hoeffding's inequality for $U$-statistics stated in Proposition \ref{prop:hoeff-U}, for any fixed $s$, w.p. $1-\delta$ we have
\begin{equation*}
    |H(s) - \epsilon_k(s)| \leq \Delta \sqrt{\frac{1}{\lfloor{n/s \rfloor}} \log(2/\delta)}\,.
\end{equation*}
Note that $\lfloor{z \rfloor} \geq z/2$ for $z \geq 1$ and therefore the above translates to 
\begin{equation*}
    |H(s) - \epsilon_k(s)| \leq \Delta \sqrt{\frac{2s}{n} \log(2/\delta)}\,.
\end{equation*}
Taking a union bound over $s=k,k+1,\ldots,n$, replacing $\delta = \delta/n$, and using $p \geq 1$, implies the result.
\end{proof}
\end{lemma}
\begin{lemma}\label{lem:star-rel}
Consider the selection process mentioned in Section \ref{ssec:adapt} and let $s_1$ be the output of this process. Then, w.p. $1-\delta$ we have
\begin{equation*}
    \frac{s^*-1}{9} \leq s_1 \leq s^*\,.
\end{equation*}
\end{lemma}
\begin{proof}
Note that using Lemma \ref{lem:eps-G}, w.p. $1-\delta$, for all values of $s$ we have $|H(s)-\epsilon_k(s)| \leq G_\delta(s)$. Now consider three different cases:
\begin{itemize}
    \item $s_1 \geq s_2 \geq s^*:$ Note that based on the choice of $s_1,s_2$, we have $H(s_2) > 2G_\delta(s_2)$. However, $H(s_2) \leq \epsilon_k(s_2) + G_\delta(s_2)$. Hence, $\epsilon_k(s_2) > G_\delta(s_2)$ which contradicts with the assumption that $s_2 \geq s^*$. Note that this is true since $\epsilon_k(s) - G_\delta(s)$ is non-positive for $s \geq s^*$.
    \item $s_1 = s^*, s_2 = s^*-1:$ Obviously $s_1 \leq s^*$.
    \item $s_2 \leq s_1\leq s^*-1:$ Note that we have
    \begin{equation*}
        \epsilon_k(s_1) - G_\delta(s_1) \leq H(s_1) \leq 2G_\delta(s_1)\,.
    \end{equation*}
    Hence, $G_\delta(s^*-1) < \epsilon_k(s^*-1) \leq \epsilon_k(s_1) \leq 3G_\delta(s_1)$. This means that $G_\delta(s^*-1)/G_\delta(s_1) \leq 3$ which implies $\sqrt{(s^*-1)/s_1} \leq 3$. Therefore, $s_1 \geq (s^*-1)/9$.
\end{itemize}
This completes the proof.
\end{proof}
Now we are ready to finalize the proof of Theorem \ref{thm:est-adapt}. Note that using the result of Lemma \ref{lem:star-rel}, w.p. $1-\delta$, we have
\begin{equation*}
    \frac{s^*-1}{9} \leq s_1 \leq s^* \,.
\end{equation*}
This basically means that if $s_* = 9s_1+1$, then $s_*$ belongs to $[s^*, 10s^*]$. Hence, we have $\epsilon_k(s_*) \leq \epsilon_k(s^*) \leq G_\delta(s^*)$ and $G_\delta(s_*) \leq G_\delta(10s^*) = \sqrt{10} G_\delta(s^*)$. Now using Theorem \ref{thm:mse_rate}, for $B \geq n/s_*$ w.p. $1-\delta$ we have
\begin{equation*}
        \|\htheta - \theta(x)\|_2 \leq \frac{2}{\lambda}\bp{L_m \epsilon(s_*) + O\bp{\psi_{\max}\sqrt{\frac{p\, s_*}{n} \left(\log\log(n /s_*) + \log(p/\delta)\right)}}}\,.
\end{equation*}
Note that $G_\delta(s_*) = \Delta \sqrt{\frac{2ps_*}{n}\log(2pn/\delta)}$. Therefore,
\begin{equation*}
    \sqrt{\frac{p\, s_*}{n} \left(\log\log(n /s_*) + \log(p/\delta)\right)} \leq G_\delta(s_*) \leq \sqrt{10} G_\delta(s^*)\,.
\end{equation*}
Replacing this in above equation together with a union bound implies that w.p. at least $1-2 \delta$ we have
\begin{equation*}
    \|\htheta - \theta(x)\|_2 = O(G_\delta(s^*))\,,
\end{equation*}
which finishes the first part of the proof. For the second part, note that according to Corollary \ref{cor:exp-kernel-shrink-kNN}, for the $k$-NN kernel $\epsilon(s) \leq C s^{-1/d}$, for a constant $C$. Note that at $s=s^*-1$ we have
\begin{equation*}
    \Delta \sqrt{\frac{2ps}{n} \log(2np/\delta)} = \epsilon_k(s) \leq Cs^{-1/d}\,,
\end{equation*}
for a constant $C$. The above implies that
\begin{equation*}
    s^* \leq 1 + \bp{\frac{C}{\Delta}}^{2d/(d+2)} \bp{\frac{n}{2p\log(2np/\delta)}}^{d/(d+2)} \leq 2 \bp{\frac{C}{\Delta}}^{2d/(d+2)} \bp{\frac{n}{2p\log(2np/\delta)}}^{d/(d+2)}\,.
\end{equation*}
Hence,
\begin{equation*}
    G_\delta(s^*) \leq \sqrt{2} \Delta^{2/(d+2)} C^{d/(d+2)} \bp{\frac{n}{2p\log(2np/\delta)}}^{-1/(d+2)}\,.
\end{equation*}

\begin{remark}\label{rem:H}
Note that although computation of $H(s)$ may look complex as it involves calculation of distance to $k$-nearest neighbor of $x$ on all $\binom{n}{s}$ subsets, there is a closed form representation for $H(s)$ according to its representation based on $L$-statistic. In fact, by sorting samples $(X_1,X_2,\ldots,X_n)$ based on their distance to $x$, i.e, $\|x-X_{(1)}\|_2 \leq \|x-X_{(2)}\|_2 \leq \ldots \leq \|x-X_{(n)}\|_2$, we have
\begin{equation*}
    H(s) = \binom{n}{s}^{-1} \sum_{i=k}^{n-s+k} \binom{i-1}{k-1}\binom{n-i}{s-k} \|x-X_{(i)}\|_2\,.
\end{equation*}
Therefore, after sorting training samples, we can compute values of $H(s)$ very efficient and fast.
\end{remark}

\subsection{Proof of Proposition \ref{thm:normal-adapt}}
Note that according to Lemma \ref{lem:star-rel}, w.p. $1-1/n$, the output of process, $s_1$ satisfies
\begin{equation*}
    \frac{s^*-1}{9} \leq s_1 \leq s^*\,,
\end{equation*}
where $s^*$ is the point for which we have $\epsilon_k(s^*) = G_{1/n}(s^*)$. This basically means that $s_* = 9s_1+1 \geq s^*$. Note that for the $k$-NN kernel we have $\eta_k(s) = \Theta(1/s)$. As $s_\zeta \geq n^{\zeta}$, this also implies that $\epsilon_k(s_\zeta,\eta_k(s_\zeta)) = O((\log(s_\zeta)/s_\zeta)^{1/d}) \rightarrow 0$. Also, according to the inequality $\zeta < \frac{\log(n) - \log(s_*)-\log\log^2(n)}{\log(n)}$ we have $1-\zeta > (\log(s_*)+\log\log^2(n))/\log(n)$ and therefore
\begin{equation*}
    n^{1-\zeta} \geq s_\zeta \log^2(n) \rightarrow \frac{s_\zeta}{n} \leq \frac{1}{\log^2(n)}\,,
\end{equation*}
and hence $n \eta_k(s_\zeta) \rightarrow 0$. Finally, note that $\sigma_{n,j}(x) = \Theta(\sqrt{s/n})$ and according to Theorem \ref{thm:normality} it suffices that
\begin{equation*}
    \max \bp{\epsilon_k(s_\zeta) \bp{\frac{s_\zeta}{n}}^{-1/2}, \epsilon_k(s_\zeta)^{1/4} \bp{\log\log(n/s_\zeta)}^{1/2}, \bp{\frac{s_\zeta}{n}}^{1/8} \bp{\log\log(n/s_\zeta)}^{5/8}} = o(1)\,.
\end{equation*}
Note that for any $\zeta>0, s_\zeta \geq s^*$ and therefore $\epsilon_k(s_\zeta) \leq \epsilon_k(s^*) = G_{1/n}(s^*)$. For the first term,
\begin{align*}
     \epsilon_k(s_\zeta) \bp{\frac{s_\zeta}{n}}^{-1/2} 
     &\leq G_{1/n}(s^*) \bp{\frac{s_\zeta}{n}}^{-1/2} \\
     &= \Delta \sqrt{\frac{2p \,s^*}{n} \log(2n^2/p)} \, \,\bp{\frac{s_\zeta}{n}}^{-1/2} \\
     &=O \bp{\sqrt{\frac{s^*}{s_\zeta}} \log(n)}.
\end{align*}
Now note that $s_\zeta = s_* n^\zeta \geq s^* n^{\zeta}$ and hence $\sqrt{s^*/s_\zeta} \log(n) = O(n^{-\zeta/2} \log(n)) \rightarrow 0$. For the second term, note that again $s_\zeta \geq s^*$ and therefore $\epsilon_k(s_\zeta) \leq \epsilon_k(s^*) = G_{1/n}(s^*) \leq G_{1/n}(s_\zeta)$. Now note that since $s_\zeta/n \leq 1/\log^2(n)$ hence
\begin{equation*}
    \epsilon_k(s_\zeta)^{1/4} \log \log(n/s_\zeta)^{1/2} \leq G_{1/n}(s_\zeta) \log\log(n) = O \bp{ \bp{\frac{\log(n)}{\log^2(n)}}^{1/8} \log\log(n)} \rightarrow 0\,.
\end{equation*}
Finally, for the last term we have $s_\zeta/n \leq 1/\log^2(n)$ and hence
\begin{equation*}
    \bp{\frac{s_\zeta}{n}}^{1/8} \bp{\log\log(n/s_\zeta)}^{5/8} \leq \bp{\frac{1}{\log(n)}}^{1/4} \log\log(n) \rightarrow 0.
\end{equation*}
This basically means w.p. $1-1/n$, $s_\zeta$ belongs to the interval for which the asymptotic normality result in Theorem \ref{thm:normality} holds. As $n \rightarrow \infty$, the conclusion follows.

\section{Stochastic Equicontinuity of $U$-statistics via Bracketing}\label{app:stoch-eq}

We define here some standard terminology on bracketing numbers in empirical process theory. Consider an arbitrary function space $\calF$ of functions from a data space $\calZ$ to $\reals$, equipped with some norm $\|\cdot\|$. A \emph{bracket} $[a, b]\subseteq \calF$, where $a,b:\calZ\rightarrow \reals$ consists of all functions $f\in \calF$, such that $a\leq f \leq b$. An \emph{$\epsilon$-bracket} is a bracket $[a, b]$ such that $\|a - b\|\leq \epsilon$. The \emph{bracketing number} $N_{[]}(\epsilon, \calF, \|\cdot\|)$ is the minimum number of $\epsilon$-brackets needed to cover $\calF$. The functions $[a,b]$ used in the definition of the brackets need not belong to $\calF$ but satisfy the same norm constraints as functions in $\calF$. Finally, for an arbitrary measure $P$ on $\calZ$, let 
\begin{align}
    \|f\|_{P,2} =~& \sqrt{\E_{Z\sim P}[f(Z)^2]} & \|f\|_{P,\infty} = \sup_{z~\in \text{ support}(P)} |f(z)|
\end{align}

\begin{lemma}[Stochastic Equicontinuity for $U$-statistics via Bracketing]\label{lem:stoch_eq_general}
Consider a function space ${\cal F}$ of symmetric functions from some data space $\calZ^s$ to $\reals$ and consider the $U$-statistic of order $s$, with kernel $f$ over $n$ samples:
\begin{equation}
    \Psi_s(f, z_{1:n}) = \binom{n}{s}^{-1} \sum_{1\leq i_1\leq \ldots \leq i_s \leq n} f(z_{i_1},\ldots, z_{i_s})
\end{equation}
Suppose $\sup_{f\in \calF} \|f\|_{P,2} \leq \eta$, $\sup_{f \in \calF} \|f\|_{P,\infty}\leq G$ and let $\kappa=n/s$. Then for $\kappa \geq \frac{G^2}{\log N_{[]}(1/2,\calF, \|\cdot\|_{P,2})}$, w.p. $1-\delta$:
\begin{multline*}
    \sup_{f \in \calF} |\Psi_s(f, Z_{1:n}) - \E[f(Z_{1:s})]| \\
    = O\left( \inf_{\rho>0} \frac{1}{\sqrt{\kappa}}\int_{\rho}^{2\eta}\sqrt{\log(N_{[]}(\epsilon, \calF, \|\cdot\|_{P,2})} + \eta \sqrt{\frac{\log(1/\delta) + \log\log(\eta/\rho)}{\kappa}} + \rho\right)
\end{multline*}
\end{lemma}
\begin{proof}
Let $\kappa=n/s$. Moreover, wlog we will assume that $\calF$ contains the zero function, as we can always augment $\calF$ with the zero function without changing the order of its bracketing number. For $q=1,\ldots,M$, let $\calF_q=\cup_{i=1}^{N_q}\calF_{qi}$ be a partition of $\calF$ into brackets of diameter at most $\epsilon_q = 2\eta/2^{q}$, with $\calF_0$ containing a single partition of all the functions. Moreover, we assume that $\calF_q$ are nested partitions. We can achieve the latter as follows: i) consider a minimal bracketing cover of $\calF$ of diameter $\epsilon_q$, ii) assign each $f\in \calF$ to one of the brackets that it is contained arbitrarily and define the partition $\bar{\calF}_q$ of size $\bar{N}_q= N_{[]}(\epsilon_q, \calF, \|\cdot\|_{P,2})$, by taking $\calF_{qi}$ to be the functions assigned to bracket $i$, iii) let $\calF_q$ be the common refinement of all partitions $\bar{\calF}_0, \ldots, \bar{\calF}_q$. The latter will have size at most $N_q \leq \prod_{q=0}^M \bar{N}_q$. Moreover, assign a representative function $f_{qi}$ to each partition $\calF_{qi}$, with the representative for the single partition at level $q=0$ is the zero function. 

\paragraph{Chaining definitions.} Consider the following random variables, where the dependence on the random input $Z$ is hidden:
\begin{align*}
    \pi_q f =~& f_{qi}, \quad \text{if $f\in \calF_{qi}$}\\
    \Delta_q f =~& \sup_{g, h\in \calF_{qi}} |g - h|, \quad \text{if $f\in \calF_{qi}$}\\
    B_q f =~& \{\Delta_0 f\leq \alpha_0, \ldots, \Delta_{q-1} f\leq \alpha_{q-1}, \Delta_q f>\alpha_q\}\\
    A_q f =~& \{\Delta_0 f\leq \alpha_0, \ldots, \Delta_q f \leq \alpha_q\}\,,
\end{align*}
for some sequence of numbers $\alpha_0, \ldots, \alpha_M$, to be chosen later. By noting that $A_{q-1} f = A_{q} f + B_{q} f$ and continuously expanding terms by adding and subtracting finer approximations to $f$, we can write the telescoping sum:
\begin{align*}
   f - \pi_0 f =~& (f - \pi_0 f)\, B_0 f + (f-\pi_0 f) A_0 f\\
    =~& (f - \pi_0 f)\, B_0 f + (f - \pi_1 f)A_0 f + (\pi_1 f - \pi_0 f) A_0 f\\
    =~& (f - \pi_0 f)\, B_0 f + (f - \pi_1 f)B_1 f + (f - \pi_1 f) A_1f +  (\pi_1 f - \pi_0 f) A_0 f\\
    \ldots\\
    =~& \sum_{q=0}^M (f-\pi_q f) B_q f + \sum_{q=1}^{M} (\pi_{q} f - \pi_{q-1} f) A_{q-1} f + (f - \pi_M f) A_M f\,.
\end{align*}

For simplicity let $\calP_{s,n} f = \Psi(f,Z_{1:n})$, $\calP f = E[f(Z_{1:s})]$ and $\calG_{s,n}$ denote the $U$-process:
\begin{equation}
    \calG_{s,n} f = \calP_{s,n} f - \calP f\,.
\end{equation}
Our goal is to bound $\|\calP_{s,n} f\|_{\calF}=\sup_{f\in \calF} |\calP_{s,n} f|$, with high probability. 
Observe that since $\calF_0$ contains only the zero function, then $\calG_{s,n} f_0 = 0$. Moreover, the operator $\calG_{s,n}$ is linear. Thus:
\begin{align*}
    \calG_{s,n} f = \calG_{s,n} (f - \pi_0 f) = \sum_{q=0}^M \calG_{s,n}(f-\pi_q f) B_q f + \sum_{q=1}^{M} \calG_{s,n}(\pi_{q} f - \pi_{q-1} f) A_{q-1} f + \calG_{s,n} (f - \pi_M f) A_M f\,.
\end{align*}
Moreover, by triangle inequality:
\begin{align*}
    \|\calG_{s,n} f\|_{\calF} \leq \sum_{q=0}^M \|\calG_{s,n}(f-\pi_q f) B_q f\|_{\calF} + \sum_{q=1}^{M} \|\calG_{s,n} (\pi_{q} f - \pi_{q-1} f) A_{q-1} f\|_{\calF} + \|\calG_{s,n} (f - \pi_M f) A_M f\|_{\calF}\,.
\end{align*}
We will bound each term in each summand separately. 

\paragraph{Edge cases.} The final term we will simply bound it by $2\alpha_M$, since $|(f - \pi_M f) A_M f|\leq \alpha_M$, almost surely. Moreover, the summand in the first term for $q=0$, we bound as follows. Observe that $B_0 f = 1\{\sup_{f} |f| > \alpha_0\}$. But we know that $\sup_{f} |f| \leq G$, hence: $B_0 f \leq 1\{ G > \alpha_0 \}$.
\begin{align*}
    \calG_{s,n}(f-\pi_0 f) B_0 f = \calG_{s,n} f B_0 f \leq |\calP_{s,n} f B_0 f| + |\calP f B_0 f| \leq 2 G\, 1\{G > \alpha_0\}\,.
\end{align*}
Hence, if we assume that $\alpha_0$ is large enough such that $\alpha_0 > G$, then the latter term is zero. By the setting of $\alpha_0$ that we will describe at the end, the latter would be satisfied if $\kappa \geq \frac{G^2}{\log N_{[]}(1/2,\calF, \|\cdot\|_{P,2})}$.

\paragraph{$B_q$ terms.} For the terms in the first summand we have by triangle inequality:
\begin{align*}
    |\calG_{s,n}(f-\pi_q f) B_q f| \leq~&  \calP_{s,n} |f-\pi_q f| B_q f + \calP |f-\pi_q f| B_q f\\
    \leq~& \calP_{s,n} \Delta_q f B_q f + \calP \Delta_q f B_q f\\
    \leq~& \calG_{s,n} \Delta_q f B_q f + 2\calP \Delta_q f B_q f \,.
\end{align*}
Moreover, observe that:
\begin{align*}
    \calP \Delta_q f B_q f \leq \calP \Delta_q f 1\{\Delta_q f > \alpha_q\} \leq \frac{1}{\alpha_q} \calP (\Delta_q f)^2 1\{\Delta_q f > \alpha_q\} \leq \frac{1}{\alpha_q} \calP (\Delta_q f)^2 = \frac{1}{\alpha_q} \|\Delta_q f\|_{P, 2}^2 \leq \frac{\epsilon_q^2}{\alpha_q}\,,
\end{align*}
where we used the fact that the partitions in $\calF_q$, have diameter at most $\epsilon_q$, with respect to the $\|\cdot\|_{P,2}$ norm.
Now observe that because the partitions $\calF_q$ are nested, $\Delta_q f \leq \Delta_{q-1} f$. Therefore, $\Delta_q f B_q f\leq \Delta_{q-1}f B_q f \leq \alpha_{q-1}$, almost surely. Moreover, $\|\Delta_q f B_q f\|_{P,2}\leq \|\Delta_q f\|_{P,2}\leq \epsilon_q$.
By Bernstein's inequality for $U$ statistics (see e.g. \cite{peel2010empirical}) for any fixed $f$, w.p. $1-\delta$:
\begin{align*}
|\calG_{s,n} \Delta_q f B_q f| \leq~& \epsilon_q \sqrt{\frac{2\log(2/\delta)}{\kappa}} + \alpha_{q-1}\frac{2\log(2/\delta)}{3\kappa}\,.
\end{align*}
Taking a union bound over the $N_q$ members of the partition, and combining with the bound on $\calP \Delta_q f B_q f$, we have w.p. $1-\delta$:
\begin{equation}
    \|\calG_{s,n}(f-\pi_q f) B_q f\|_{\calF} \leq \epsilon_q \sqrt{\frac{2\log(2N_q /\delta)}{\kappa}} + \alpha_{q-1}\frac{2\log(2N_q/\delta)}{3\kappa} + \frac{2 \epsilon_q^2}{\alpha_q}\,.
\end{equation}

\paragraph{$A_q$ terms.} For the terms in the second summand, we have that since the partitions are nested, $|(\pi_{q} f - \pi_{q-1} f) A_{q-1} f|\leq \Delta_{q-1} f A_{q-1} f\leq \alpha_{q-1}$. Moreover, $\|(\pi_{q} f - \pi_{q-1} f) A_{q-1} f\|_{P,2} \leq \|\Delta_{q-1} f\|_{P,2} \leq \epsilon_{q-1}\leq 2\epsilon_q$. Thus, by similar application of Bernstein's inequality for $U$-statistics, we have for a fixed $f$, w.p. $1-\delta$:
\begin{align*}
    |\calG_{s,n}(\pi_{q} f - \pi_{q-1} f) A_{q-1} f| \leq \epsilon_{q} \sqrt{\frac{8\log(2/\delta)}{\kappa}} + \alpha_{q-1}\frac{2\log(2/\delta)}{3\kappa}\,.
\end{align*}
As $f$ ranges there are at most $N_{q-1}\, N_{q}\leq N_{q}^2$ different functions $(\pi_{q} f - \pi_{q-1} f) A_{q-1} f$. Thus taking a union bound, we have that w.p. $1-\delta$:
\begin{align*}
    \|\calG_{s,n}(\pi_{q} f - \pi_{q-1} f) A_{q-1} f\|_{\calF}\leq \epsilon_q \sqrt{\frac{16\log(2N_{q}/\delta)}{\kappa}} + \alpha_{q-1}\frac{4\log(2N_{q}/\delta)}{3\kappa}\,.
\end{align*}

Taking also a union bound over the $2M$ summands and combining all the above inequalities, we have that w.p. $1-\delta$:
\begin{align*}
    \|\calG_{s,n} f\|_{\calF} \leq \sum_{q=1}^M \epsilon_q \sqrt{\frac{32\log(2N_{q} M /\delta)}{\kappa}} +  \alpha_{q-1}\frac{6\log(2N_{q} M /\delta)}{3\kappa} + \frac{2 \epsilon_q^2}{\alpha_q}\,.
\end{align*}
Choosing $\alpha_q = \epsilon_q \sqrt{\kappa} / \sqrt{\log(2N_{q+1} M /\delta)}$ for $q<M$ and $\alpha_M=\epsilon_M$, we have for some constant $C$:
\begin{align*}
    \|\calG_{s,n} f\|_{\calF} \leq~& C \sum_{q=1}^M \epsilon_q \sqrt{\frac{\log(2N_{q} M /\delta)}{\kappa}} + 3\epsilon_M\\
    \leq~& C \sum_{q=1}^M \epsilon_q \sqrt{\frac{\log(N_{q})}{\kappa}} + C\sum_{q=1}^M \epsilon_q \sqrt{\frac{\log(2M/\delta)}{\kappa}} + 3\epsilon_M\\
    \leq~& C \sum_{q=1}^M \epsilon_q \sqrt{\frac{\log(N_{q})}{\kappa}} + 2C\eta \sqrt{\frac{\log(2M/\delta)}{\kappa}} + 3\epsilon_M\,.
\end{align*}
Moreover, since $\log(N_q) \leq \sum_{t=0}^q \log(N_{[]}(\epsilon_q, \calF, \|\cdot\|_{P,2})$, we have:
\begin{align*}
    \sum_{q=1}^M \epsilon_q \sqrt{\log(N_{q})} \leq~& \sum_{q=1}^M \epsilon_q \sum_{t=0}^q \sqrt{\log(N_{[]}(\epsilon_t, \calF, \|\cdot\|_{P,2})} =  \sum_{t=0}^M \sqrt{\log(N_{[]}(\epsilon_t, \calF, \|\cdot\|_{P,2})} \sum_{q=t}^M \epsilon_q\\
    \leq~& 2\sum_{t=0}^M \epsilon_t \sqrt{\log(N_{[]}(\epsilon_t, \calF, \|\cdot\|_{P,2})}\\
    \leq~& 4\sum_{t=0}^M (\epsilon_t-\epsilon_{t+1}) \sqrt{\log(N_{[]}(\epsilon_t, \calF, \|\cdot\|_{P,2})}\\
    \leq~& 4 \int_{\epsilon_M}^{\epsilon_0}\sqrt{\log(N_{[]}(\epsilon, \calF, \|\cdot\|_{P,2})}\,.
\end{align*}
Combining all the above yields the result.
\end{proof}

\begin{corollary}\label{cor:bracketing}
Consider a function space ${\cal F}$ of symmetric functions. Suppose that $\sup_{f\in \calF} \|f\|_{P,2} \leq \eta$ and $\log(N_{[]}(\epsilon, \calF, \|\cdot\|_{P,2})=O(1/\epsilon)$. Then for $\kappa \geq O(G^2)$, w.p. $1-\delta$:
\begin{equation}
    \sup_{f \in \calF} |\Psi_s(f, Z_{1:n}) - \E[f(Z)]| = O\left( \sqrt{\frac{\eta}{\kappa}} + \eta \sqrt{\frac{\log(1/\delta) + \log\log(\kappa/\eta)}{\kappa}}\right)\,.
\end{equation}
\end{corollary}
\begin{proof}
Applying Lemma~\ref{lem:stoch_eq_general}, we get for every $\rho>0$, the desired quantity is upper bounded by:
\begin{align*}
    O\left(\frac{1}{\sqrt{\kappa}}\int_{\rho}^{\eta}\frac{1}{\sqrt{\epsilon}} + \eta \sqrt{\frac{\log(1/\delta) + \log\log(\eta/\rho)}{\kappa}} + \rho\right)\\
    ~~~~~= O\left(\frac{\sqrt{\eta} - \sqrt{\rho}}{\sqrt{\kappa}} + \eta \sqrt{\frac{\log(1/\delta) + \log\log(\eta/\rho)}{\kappa}} + \rho\right)\,.
\end{align*}
Choosing $\rho=\sqrt{\eta}/\sqrt{\kappa}$, yields the desired bound.
\end{proof}

\section{Proof of Lemma~\ref{lem:normality_of_U}}\label{app:hajek}

We will argue asymptotic normality of the $U$-statistic defined as:
\begin{equation*}
    \Psi_{0,\beta}(x; \theta(x)) = \binom{n}{s}^{-1} \sum_{b\subset [n]: |b|=s} \E_{\omega_b}\bb{\sum_{i\in S_b}\alpha_{S_b,\omega_b}(X_i) \psi_\beta(Z_i; \theta(x))}
\end{equation*}
under the assumption that for any subset of indices $S_b$ of size $s$: $\E\bb{\E[\alpha_{S_b,\omega_b}(X_1) | X_1]^2}=\eta(s)$ and that the kernel satisfies shrinkage in probability with rate $\epsilon(s, \delta)$ such that $\epsilon(s, \eta(s)^2)\rightarrow 0$ and $n \eta(s)\rightarrow \infty$. 
For simplicity of notation we let:
\begin{equation}
    Y_i = \psi_\beta(Z_i; \theta(x))
\end{equation}
and we then denote:
\begin{equation}
    \Phi(Z_1, \ldots, Z_s) = \E_{\omega}\bb{\sum_{i=1}^s K_{\omega}(x, X_i, \{Z_j\}_{j=1}^s)\, Y_i}\,.
\end{equation}
Observe that we can then re-write our $U$-statistic as:
\begin{equation*}
    \Psi_{0,\beta}(x; \theta(x)) = \binom{n}{s}^{-1}\sum_{1\leq i_1\leq \ldots \leq i_s\leq n} \Phi(Z_{i_1}, \ldots, Z_{i_s})\,.
\end{equation*}
Moreover, observe that by the definition of $Y_i$, $\E[Y_i \mid X_i]=0$ and also
\begin{align*}
    |Y_i| \leq \|\beta\|_2 \|M_0^{-1} (\psi(Z_i; \theta(x))-m(X_i;\theta(x))\|_2 2\leq \frac{R}{\lambda} \|\psi(Z_i; \theta(x))\|_2 \leq 2\frac{R\sqrt{p}}{\lambda} \psi_{\max} \triangleq y_{\max}\,.
\end{align*}
Invoking Lemma~\ref{lem:hajek}, it suffices to show that: $\Var\bb{\Phi_1(Z_1)}=\Omega(\eta(s))$, where $\Phi_1(z_1) = \E[\Phi(z_1, Z_2, \ldots, Z_s)]$. The following lemma shows that under our conditions on the kernel, the latter property holds.

\begin{lemma}\label{lem:ustat-norm}
Suppose that the kernel $K$ is symmetric (Assumption \ref{ass:sym}), has been built in an honest manner (Assumption \ref{ass:honest}) and satisfies:
\begin{align*}
    \E\bb{\E\bb{K(x, X_1, \{Z_j\}_{j=1}^s) \mid X_1}^2} &= \eta(s) \leq 1 & \text{and} & &
    \epsilon(s,\eta(s)^2) &\rightarrow 0 \,.
\end{align*}
Then, the following holds
\begin{equation*}
    \Var \bb{\Phi_1(Z_1)} \geq \Var(Y \mid X=x)\, \eta(s) + o(\eta(s)) = \Omega \bp{\eta(s)}\,.
\end{equation*}
\end{lemma}
\begin{proof}
    Note we can write
    \begin{equation*}
        \Phi_1(Z_1) =  \underbrace{\E\bb{\Phi(Z_1, \ldots, Z_s) \mid X_1}}_{A} + \underbrace{\E\bb{\Phi(Z_1, \ldots, Z_s) \mid X_1,Y_1} - \E\bb{\Phi(Z_1, \ldots, Z_s)\mid X_1}}_B.
    \end{equation*}
    Here, $B$ is zero mean conditional on $X_1$ and also $A$ and $B$ are uncorrelated, i.e., $\E[AB]=\E[A]\E[B]=0$. Therefore:
    \begin{equation*}
        \Var \bb{\Phi_1(Z_1)} \geq \Var\bb{B} = \Var \bb{\sum_{i=1}^s \bp{\E\bb{K(x, X_i, \{Z_j\}_{j=1}^s) Y_i \mid X_1,Y_1} - \E[K(x, X_i, \{Z_j\}_{j=1}^s) Y_i \mid X_1]}}\,.
    \end{equation*}
    For simplicity of notation let $W_i = K(x, X_1, \{Z_j\}_{j=1}^s)$ denote the random variable which corresponds to the weight of sample $i$. Note that thanks to the honesty of kernel defined in Assumption \ref{ass:honest}, $W_i$ is independent of $Y_1$ conditional on $X_1$, for $i \geq 2$. Hence all the corresponding terms in the summation are zero. Therefore, the expression inside the variance above simplifies to
    \begin{equation*}
        \E[W_1 Y_1 \mid X_1,Y_1]-\E[W_1 Y_1 \mid X_1]\,.
    \end{equation*}
    Moreover, by honesty $W_1$ is independent of $Y_1$ conditional on $X_1$. Thus, the above further simplifies to:
    \begin{equation*}
        \E[W_1\mid X_1] \, \bp{Y_1 - \E[Y_1 \mid X_1]}\,.
    \end{equation*}
     Using $\Var(G)=\E[G^2]-\E[G]^2$, this can be further rewritten as
    \begin{equation*}
        \Var\bb{\Phi_1(Z_1)} \geq \E \bb{\E[W_1 \mid X_1]^2 \bp{Y_1 - \E[Y_1\mid X_1]}^2} - \E \bb{\E[W_1 \mid X_1] (Y_1-\E[Y_1 \mid X_1])}^2\,.
    \end{equation*}
    Note that $Y_1 - \E[Y_1 \mid X_1]$ is uniformly upper bounded by some $\psi_{\max}$. Furthermore, by the symmetry of the kernel we have $\E \bb{\E[W_1 \mid X_1]} = \E[W_1] = 1/s$.\footnote{Since $\E[W_i]$ are all equal to the same value $\kappa$ and $\sum_i\E[W_i] = 1$, we get $\kappa=1/s$.} Thus the second term in the latter is of order $1/s^2$. Hence:
    \begin{equation*}
        \Var \bb{\Phi(Z_1)} \geq \E \bb{\E[\alpha_b(X_1) \mid X_1]^2 \bp{Y_1 - \E[Y_1\mid X_1]}^2} + o(1/s)\,.
    \end{equation*}
    Focusing at the first term and letting $\sigma^2(x) = \Var(Y | X=x)$, we have:
    \begin{align*}
        \E\bb{\E[W_1 \mid X_1]^2 \bp{Y_1 - \E[Y_1\mid X_1]}^2} =~& \E \bb{\E[W_1 \mid X_1]^2 \sigma^2(X_1)} \\
        =~& \E \bb{\E[W_1 \mid X_1]^2} \sigma^2(x) + \E \bb{\E[W_1 \mid X_1]^2 \bp{\sigma^2(X_1)-\sigma^2(x)}}\,.
    \end{align*}
    The goal is to prove that the second term is $o(1/s)$. For ease of notation let $V_1=\E\bb{W_1 \mid X_1}$. Then we can bound the second term as:
    \begin{align*}
        \left|\E \bb{V_1^2 \bp{\sigma^2(X_1)-\sigma^2(x)}}\right|
        \leq~& L_{mm} \epsilon(s,\delta)\, \E \bb{V_1^2\, \,1\bc{\|x-X_1\|_2 \leq \epsilon(s,\delta)}} \\
        &~~~~~~~~~~~~~~~~~+ 2 y_{\max}^2 \E \bb{V_1^2\,1 \bc{\|x-X_1\|_2 > \epsilon(s,\delta)}}\\
        \leq~& L_{mm} \epsilon(s,\delta)\, \E \bb{V_1^2} +  2 y_{\max}^2 \E \bb{V_1^2\,1 \bc{\|x-X_1\|_2 > \epsilon(s,\delta)}}\\
        \leq~& L_{mm} \epsilon(s,\delta) \eta(s) +  2 y_{\max}^2 \E \bb{V_1\,1 \bc{\|x-X_1\|_2 > \epsilon(s,\delta)}}\\
        \leq~& L_{mm} \epsilon(s,\delta) \eta(s) +  2 y_{\max}^2 \E \bb{W_1\,1 \bc{\|x-X_1\|_2 > \epsilon(s,\delta)}}\,,
    \end{align*}
    where we used the fact that $V_1 \leq 1$, the assumption that $\sigma^2(\cdot)$ is $L_{mm}$-Lipschitz, the tower rule and the definition of $\eta(s)$.
    Furthermore,
    \begin{align*}
        \E \bb{W_1 \,1 \bc{\|x-X_1\|_2 > \epsilon(s,\delta)}} 
        &\leq \Pr \bb{\|x-X_1\|_2 \geq \epsilon(s,\delta){~\text{and}~} W_1 > 0} \\
        &\leq \Pr \bb{\sup_i \bc{\|x-X_i\|_2 : W_i > 0} \geq \epsilon(s,\delta)}\,,
    \end{align*}
    which by definition is at most $\delta$. By putting $\delta = \eta(s)^2$ we obtain
    \begin{align*}
        \left|\E \bb{\E[W_1 \mid X_1]^2 \bp{\sigma^2(X_1)-\sigma^2(x)}}\right| \leq  L_{mm} \epsilon(s,\eta(s)^2) \eta(s) +  2 y_{\max}^2 \eta(s)^2 = o(\eta(s))\,,
    \end{align*}
    where we invoked our assumption that $\epsilon(s,\eta(s)^2) \rightarrow 0$. Thus we have obtained that:
    \begin{equation*}
       \Var\bb{\Phi_1(Z_1)} \geq \E \bb{\E[W_1 \mid X_1]^2} \sigma^2(x) + o(\eta(s)) \,,
    \end{equation*}
    which is exactly the form of the lower bound claimed in the statement of the lemma. This concludes the proof.
\end{proof}

\subsection{H\'{a}jek Projection Lemma for Infinite Order $U$-statistics}

The following is a small adaptation of Theorem 2 of \cite{fan2018dnn}, which we present here for completeness.
\begin{lemma}[\cite{fan2018dnn}]\label{lem:hajek}
Consider a $U$-statistic defined via a symmetric kernel $\Phi$:
\begin{equation}
    U(Z_1,\ldots, Z_n) = \binom{n}{s}^{-1}\sum_{1\leq i_1\leq \ldots \leq i_s\leq n} \Phi(Z_{i_1}, \ldots, Z_{i_s})\,,
\end{equation}
where $Z_i$ are i.i.d. random vectors and $s$ can be a function of $n$. Let $\Phi_1(z_1) = \E[\Phi(z_1, Z_2, \ldots, Z_s)]$ and $\eta_1(s)=\Var_{z_1}\bb{\Phi_1(z_1)}$. Suppose that $\Var{\Phi}$ is bounded, $n\, \eta_1(s) \rightarrow \infty$. Then:
\begin{equation}
    \frac{U(Z_1,\ldots, Z_n) - \E\bb{U}}{\sigma_n} \rightarrow_d \normal(0,1)\,,
\end{equation}
where $\sigma_n^2 = \frac{s^2}{n}\eta_1(s)$.
\end{lemma}
\begin{proof}
The proof follows identical steps as the one in \cite{fan2018dnn}. We argue about the asymptotic normality of a $U$-statistic:
\begin{equation}
    U(Z_1,\ldots, Z_n) = \binom{n}{s}^{-1}\sum_{1\leq i_1\leq \ldots \leq i_s\leq n} \Phi(Z_{i_1}, \ldots, Z_{i_s})\,.
\end{equation}
Consider the following projection functions:
\begin{align*}
    \Phi_1(z_1) =~& \E[\Phi(z_1, Z_2, \ldots, Z_s)]\,, & \tilde{\Phi}_1(z_1) =~& \Phi_1(z_1) - \E\bb{\Phi}\,, \\
    \Phi_2(z_1, z_2) =~& \E[\Phi(z_1, z_2, Z_3, \ldots, Z_s) \,, & \tilde{\Phi}_2(z_1, z_2) =~& \Phi_2(z_1, z_2) - \E\bb{\Phi} \,, \\
    \vdots\\
    \Phi_s(z_1, z_2, \ldots, z_s) =~& \E[\Phi(z_1, z_2, Z_3, \ldots, Z_s) \,, &
    \tilde{\Phi}_s(z_1, z_2, \ldots, z_s) =~& \Phi_s(z_1, z_2, \ldots, z_s) - \E\bb{\Phi}\,,
\end{align*}
where $\E\bb{\Phi}=\E\bb{\Phi(Z_1,\ldots, Z_s)}$. Then we define the canonical terms of Hoeffding's $U$-statistic decomposition as:
\begin{align*}
    g_1(z_1) =~& \tilde{\Phi}_1(z_1) \,, \\
    g_2(z_1, z_2) =~& \tilde{\Phi}_2(z_1, z_2) - g_1(z_1) - g_2(z_2)\,, \\
    g_3(z_1, z_2, z_3) =~& \tilde{\Phi}_2(z_1, z_2, Z_3) - \sum_{i=1}^3 g_1(z_i) - \sum_{1\leq i<j\leq 3} g_2(z_i, z_j) \,, \\
    \vdots\\
    g_s(z_1, z_2, \ldots, z_s) =~& \tilde{\Phi}_s(z_1, z_2, \ldots, z_s) - \sum_{i=1}^s g_1(z_i) - \sum_{1\leq i < j \leq s} g_2(z_i, z_j) - \ldots\\
    &~~~~~~~~~~~~~~~~~~~~...- \sum_{1\leq i_1 < i_2 < \ldots < i_{s-1} \leq s} g_{s-1}(z_{i_1}, z_{i_2}, \ldots, z_{i_{s-1}}) \,.
\end{align*}
Subsequently the kernel of the $U$-statistic can be re-written as a function of the canonical terms:
\begin{equation}
   \tilde{\Phi}(z_1, \ldots, z_s) = \Phi(z_1, \ldots, z_s) - \E\bb{\Phi} = \sum_{i=1}^s g_1(z_i) + \sum_{1\leq i < j \leq s} g_2(z_i, z_j) + \ldots + g_s(z_1, \ldots, z_s)\,.
\end{equation}
Moreover, observe that all the canonical terms in the latter expression are un-correlated. Hence, we have:
\begin{equation}
    \Var \bb{\Phi(Z_1, \ldots, Z_n)} = \binom{s}{1} \E\bb{g_1^2} + \binom{s}{2} \E\bb{g_2^2} + \ldots + \binom{s}{s} \E\bb{g_s^2}\,.
\end{equation}
We can now re-write the $U$ statistic also as a function of canonical terms:
\begin{align*}
   U(Z_1,\ldots, Z_n) - \E\bb{U} =~& \binom{n}{s}^{-1} \sum_{1\leq i_1 < i_2 < \ldots < i_s \leq n} \tilde{\Phi}(Z_{i_1}, \ldots, Z_{i_s})\\
   =~& \binom{n}{s}^{-1} \bigg(\binom{n-1}{s-1}\sum_{i=1}^n g_1(Z_i) + \binom{n-2}{s-2} \sum_{1\leq i < j \leq n} g_2(Z_i, Z_j) + \ldots\\
    & ~~~~~~~~~~~~~~~~~~ + \binom{n-s}{s-s} \sum_{1\leq i_1 < i_2 < \ldots < i_s \leq n} g_s(Z_{i_1}, \ldots, Z_{i_s})\bigg)\,.
\end{align*}
Now we define the H\'{a}jek projection to be the leading term in the latter decomposition:
\begin{equation}
    \hat{U}(Z_1, \ldots, Z_n) = \binom{n}{s}^{-1} \binom{n-1}{s-1}\sum_{i=1}^n g_1(Z_i)\,.
\end{equation}
The variance of the Hajek projection is:
\begin{equation}
    \sigma_n^2 = \Var\bb{\hat{U}(Z_1, \ldots, Z_n)} = \frac{s^2}{n} \Var\bb{\Phi_1(z_1)} = \frac{s^2}{n}\eta_1(s)\,.
\end{equation}
The H\'{a}jek projection is the sum of independent and identically distributed terms and hence by the Lindeberg-Levy Central Limit Theorem (see e.g., \cite {billingsley2008probability,borovkov2013probability}):
\begin{equation}
    \frac{\hat{U}(Z_1,\ldots, Z_n)}{\sigma_n} \rightarrow_d \normal(0,1)\,.
\end{equation}

We now argue that the remainder term: $\frac{U - \E\bb{U} - \hat{U}}{\sigma_n}$ vanishes to zero in probability. The latter then implies that $\frac{U - \E\bb{U}}{\sigma_n} \rightarrow_d \normal(0,1)$ as desired. We will show the sufficient condition of convergence in mean square: $\frac{\E\bb{\bp{U - \E\bb{U} - \hat{U}}^2}}{\sigma_n^2}\rightarrow 0$. From an inequality due to \cite{wager2017estimation}:
\begin{align*}
    \E\bb{\bp{U - \E\bb{U} - \hat{U}}^2} =~& \binom{n}{s}^{-2} \bc{\binom{n-2}{s-2}^2 \binom{n}{2} \E[g_2^2] + \ldots + \binom{n-s}{s-s}^2 \binom{n}{s} \E[g_s^2]}\\
    =~& \sum_{r=2}^s \bc{\binom{n}{s}^{-2} \binom{n-r}{s-r}^2 \binom{n}{r} \E[g_r^2]}\\
    =~& \sum_{r=2}^s \bc{\frac{s! (n-r)!}{n! (s-r)!} \binom{s}{r}\E[g_r^2]}\\
    \leq~& \frac{s(s-1)}{n(n-1)} \sum_{r=2}^s \binom{s}{r} \E[g_r^2]\\
    \leq~& \frac{s^2}{n^2} \Var\bb{\Phi(Z_1, \ldots, Z_s)}\,.
\end{align*}
Since $\Var\bb{\Phi(Z_1, \ldots, Z_n)}$ is bounded by a constant $V^*$ and $n\,\eta_1(s) \rightarrow \infty$, by our assumption, we have:
\begin{equation}
    \frac{\E\bb{\bp{U - \E\bb{U} - \hat{U}}^2}}{\sigma_n^2}\leq \frac{\frac{s^2}{n^2} V^*}{\frac{s^2}{n}\eta_1} = \frac{V^*}{n\, \eta_1(s)} \rightarrow 0 \,.
\end{equation}
\end{proof}
\end{appendix}
\end{document}